\documentclass{article}
\usepackage{amsmath,amsfonts,bm}

\def\eqref#1{equation~\ref{#1}}

\def\1{\bm{1}}

\def\vh{{\bm{h}}}

\def\vs{{\bm{s}}}

\def\vx{{\bm{x}}}

\def\vz{{\bm{z}}}

\def\mA{{\bm{A}}}

\def\mC{{\bm{C}}}
\def\mD{{\bm{D}}}

\def\mI{{\bm{I}}}

\def\mL{{\bm{L}}}
\def\mM{{\bm{M}}}

\def\mO{{\bm{O}}}

\def\mS{{\bm{S}}}

\def\mW{{\bm{W}}}

\DeclareMathAlphabet{\mathsfit}{\encodingdefault}{\sfdefault}{m}{sl}
\SetMathAlphabet{\mathsfit}{bold}{\encodingdefault}{\sfdefault}{bx}{n}

\def\sR{{\mathbb{R}}}

\usepackage{amsmath,amsfonts,bm}
\usepackage{caption}
\usepackage{subcaption}
\usepackage[dvipsnames]{xcolor}
\definecolor{orange}{rgb}{0.93725,0.52549,0.2117647}
\definecolor{blue}{rgb}{0.23137,0.4588,0.68627}

\usepackage{dirtytalk}
	
\usepackage{color, colortbl}
\usepackage{microtype}
\usepackage{stfloats}
\usepackage{graphicx}
\usepackage{booktabs} 
\usepackage{multirow} 
\usepackage{siunitx}
\usepackage{array}
\newcommand{\thickhline}{%
    \noalign {\ifnum 0=`}\fi \hrule height 1pt
    \futurelet \reserved@a \@xhline
}
\newcolumntype{"}{@{\hskip\tabcolsep\vrule width 1pt\hskip\tabcolsep}}

\usepackage{upquote,textcomp,amsmath}
\newcommand{\upquote}{\text{\textquotesingle}}

\usepackage[pagebackref=false,breaklinks=true,letterpaper=true,colorlinks,bookmarks=false]{hyperref}
\hypersetup{colorlinks,breaklinks,
            urlcolor=[rgb]{0.93725,0.52549,0.2117647},
            citecolor=[rgb]{0.23137,0.4588,0.68627},
            linkcolor=[rgb]{0.93725,0.52549,0.2117647}}

\usepackage{amsthm}
\newtheorem{theorem}{Theorem}

\newtheorem{proposition}{Proposition}

\newtheorem{remark_2}{Remark}

\newcommand*{\tran}{^{\mkern-1.5mu\mathsf{T}}}
\renewcommand{\eqref}[1]{Eq.~\ref{#1}}

\newcommand{\ourmethodname}{dendPLRNN} 
\newcommand{\klmetric}{$D_{\textrm{stsp}}$} 
\newcommand{\beginsupplement}{%
        \setcounter{table}{0}
        \renewcommand{\thetable}{S\arabic{table}}%
        \setcounter{figure}{0}
        \renewcommand{\thefigure}{S\arabic{figure}}%
     }
\usepackage[accepted]{icml2022}

\usepackage{xcolor}

\usepackage{authblk}

\begin{document}

\twocolumn[
\icmltitle{Tractable Dendritic RNNs for Reconstructing Nonlinear Dynamical Systems}


\icmlsetsymbol{equal}{*}

\begin{icmlauthorlist}
\icmlauthor{Manuel Brenner}{equal,zi,hd}
\icmlauthor{Florian Hess}{equal,zi,hd}
\icmlauthor{Jonas M. Mikhaeil}{zi,hd}
\icmlauthor{Leonard Bereska}{zi,amst}
\icmlauthor{Zahra Monfared}{zi}
\icmlauthor{Po-Chen Kuo}{tai}
\icmlauthor{Daniel Durstewitz}{zi,hd}
\end{icmlauthorlist}

\icmlaffiliation{zi}{Dept. of Theoretical Neuroscience, Central Institute of Mental Health, Mannheim, Germany}
\icmlaffiliation{hd}{Faculty of Physics and Astronomy, Heidelberg University, Germany}
\icmlaffiliation{amst}{University of Amsterdam, Netherlands}
\icmlaffiliation{tai}{National Taiwan University, Taiwan}

\icmlcorrespondingauthor{Manuel Brenner}{manuel.brenner@zi-mannheim.de}
\icmlcorrespondingauthor{Daniel Durstewitz}{daniel.durstewitz@zi-mannheim.de}

\icmlkeywords{recurrent neural networks, dynamical systems, chaos, attractors, variational inference, dendritic computation, state space models}





\vskip 0.3in
]
\printAffiliationsAndNotice{\icmlEqualContribution}






\begin{abstract}
In many scientific disciplines, we are interested in inferring the nonlinear dynamical system underlying a set of observed time series, a challenging task in the face of chaotic behavior and noise. Previous deep learning approaches toward this goal often suffered from a lack of interpretability and tractability. In particular, the high-dimensional latent spaces often required for a faithful embedding, even when the underlying dynamics lives on a lower-dimensional manifold, can hamper theoretical analysis. Motivated by the emerging principles of dendritic computation, we augment a dynamically interpretable and mathematically tractable piecewise-linear (PL) recurrent neural network (RNN) by a linear spline basis expansion. We show that this approach retains all the theoretically appealing properties of the simple PLRNN, yet boosts its capacity for approximating arbitrary nonlinear dynamical systems in comparatively low dimensions. We employ 
two frameworks for training the system, one combining back-propagation-through-time (BPTT) with teacher forcing, and another based on fast and scalable variational inference. We show that the dendritically expanded PLRNN achieves better reconstructions with fewer parameters and dimensions on various dynamical systems benchmarks and compares favorably to other methods, while retaining a tractable and interpretable structure.
\end{abstract}

\section{Introduction}\label{sec:introduction}
For many complex systems in physics, biology, or the social sciences, we do not know or have only rudimentary knowledge about the dynamical system (DS) that may underlie those quantities that we can empirically observe or measure. Data-driven approaches aimed at automatically inferring the generating DS from time-series observations could therefore strongly support the scientific process, and various such methods have been proposed in recent years \citep{raissi_multistep_2018, zhu_neural_2021, yin_augmenting_2021, norcliffe_neural_2021, mohajerin_multi-step_2018, karl_deep_2017, chen_neural_2018, strauss_augmenting_2020}. However, due to the often high-dimensional, complex, chaotic, and inherently noisy nature of real-world DS, like the brain, weather-, or ecosystems, this remains a formidable challenge. Moreover, although the true DS may evolve on a lower-dimensional manifold in its state space, the system used for approximation usually needs to be of higher dimensionality to achieve a proper embedding \citep{takens_detecting_1981, sauer_embedology_1991, kantz_nonlinear_2004}. This is especially true when the approximating system is of a different functional form than the one that would most naturally describe the data generation process (but is unknown), for instance, when we attempt to approximate a system of exponential or trigonometric functions by polynomials.

In this work we sought to improve the capacity and expressiveness of a specific class of recurrent neural networks (RNNs), achieving agreeable solutions with fewer dimensions and parameters while retaining a set of desirable theoretical properties. Specifically, we build on piecewise-linear RNNs (PLRNNs) based on ReLU activation functions, for which fixed points, periodic orbits, and other dynamical properties can be derived analytically \citep{schmidt_identifying_2021, koppe_identifying_2019}, and for which dynamically equivalent continuous-time (ordinary differential equation, ODE) systems can be constructed \citep{monfared_transformation_2020}. Inspired by principles of dendritic computation in biological neurons (Fig. \ref{fig:base_illustration}), each PLRNN unit was endowed with a set of nonlinear pre-processing subunits (\say{dendritic branches}), such that it effectively takes on the role of an equivalent much larger network. Mathematically, this comes down, in our case, to enhancing each latent unit with a linear spline basis expansion as popular in statistics \citep{hastie_elements_2009}. Through this trick, we achieve a powerful RNN which provides reconstructions of underlying nonlinear DS in lower-dimensional latent spaces than were needed by conventional PLRNNs or other approaches. Model training may be performed by classical Back-Propagation-Through-Time (BPTT; \citet{Rumelhart_1986_learning}) augmented by teacher forcing (TF; \citet{williams_learning_1989,pearlmutter_1990}), or through the scalable framework of sequential variational auto-encoders (SVAE) \citep{archer_black_2015, girin_dynamical_2020, krishnan_structured_2017}. 
Importantly, we prove that these modifications preserve the mathematical and dynamical accessibility of the resulting system, e.g., such that fixed points, cycles, and their stability, can still be computed analytically.
Besides its effectiveness in capturing complex dynamical systems in fewer dimensions within a tractable framework, our approach highlights more generally how principles of dendritic signal processing may be harvested in the design of RNNs. Strongly nonlinear local computations are known for decades to occur within dendritic trees of biological neurons \citep{mel_information_1994, poirazi_pyramidal_2003}, but have hardly been exploited so far for machine learning models.
\section{Related Work}\label{sec:related_work}
One class of DS reconstruction models attempts to discover governing equations from the vector field estimated from data through differencing the time series. Sparse Identification of Nonlinear Dynamics (SINDy), for instance, does so by sparsely regressing on a rich library of basis functions using the least absolute shrinkage and selection operator (LASSO) \citep{brunton_discovering_2016, rudy_data-driven_2017, de_silva_pysindy_2020}. Other methods approximate the vector field using additive ODE models \citep{chen_network_2017}, sparse autoencoders \citep{heim_rodent_2019}, shallow `multi-layer' perceptrons 
reformulated as RNNs \citep{trischler_synthesis_2016}, 
or deep neural networks \citep{chen_neural_2018}. Some works aimed at directly learning the system's underlying Hamiltonian \citep{chen_symplectic_2020, greydanus_hamiltonian_2019}. 
Generally, numerical derivatives obtained from time series tend to be more noise-prone than the time series observations themselves \citep{baydin_automatic_2018, chen_network_2017, raissi_deep_2018}.
This can be a problem particularly if only comparatively short trajectories were empirically observed or when the underlying systems are very high-dimensional, as in these cases the system's vector field may be (severely) under-sampled.
Methods directly based on numerical derivatives also need to be augmented by other techniques, like delay embeddings \citep{kantz_nonlinear_2004, bakarji2022discovering}, if not all the system's dimensions were observed.

Various RNN architectures such as Long-Short-Term-Memory networks (LSTMs) \citep{zheng_state_2017}, 
Reservoir Computing (RC) \citep{pathak_model-free_2018}, 
or PLRNNs \citep{koppe_identifying_2019, schmidt_identifying_2021} have been employed to infer DS directly from the observed time series without going through numerical derivatives. More generally, a wide array of RNN architectures with specific functional or parametric forms, e.g. based on coupled oscillators \citep{rusch_coupled_2021}, has been proposed in recent years \citep{kerg_non-normal_2019, chang_antisymmetricrnn_2019, erichson_lipschitz_2021, kag_rnns_2020, rusch2022long}, mainly in order to tackle the exploding/ vanishing gradient problem in RNN training \citep{bengio1994learning, hochreiter1997long}. 
In the present context it is important to note, however, that most of these are not suitable for DS reconstruction since their functional form or specific parameterization strictly delimits the range of DS phenomena they can learn or generate. For instance, chaotic dynamics are not possible in any of these latter systems by mathematical design \citep{monfared_2022},
with the only exception of 
Long-Expressive-Memory (LEM) networks \citep{rusch2022long}. 
More recently, transformers \citep{shalova_deep_2020, shalova_tensorized_2020} were used as black box approaches for DS prediction. 

Except for PLRNNs, however, all of the approaches reviewed above, even those specifically designed for DS reconstruction and prediction, rest on relatively complex model formulations that are not easy to tackle and analyze from a DS perspective \citep{fraccaro_sequential_2016}. The ability to gain deeper insights into the specific DS properties and mechanisms of the recovered system is, however, often crucial for its applicability to science and engineering problems. Transformers, unlike RNNs, do not even constitute DS themselves (as they explicitly forgo any temporal recursions), and therefore are not directly amenable to DS theory tools. Moreover, most of these models, RC in particular, need very high-dimensional latent spaces, which further adds to their black-box nature.

Better interpretability and tractability is achieved by using PLRNNs \citep{koppe_identifying_2019, schmidt_identifying_2021} or by (locally) linearizing nonlinear systems through ideas from Koopman operator theory \citep{azencot_forecasting_2020, brunton_chaos_2017, yeung_learning_2017}.  
In such systems, certain DS properties can be analytically accessed \citep{schmidt_identifying_2021, monfared_existence_2020}. On the downside, usually one needs to move to very high dimensions to represent the DS in question properly. Here we aim to overcome this limitation by augmenting PLRNNs with linear basis expansions without altering their analytical accessibility. 


Finally, probabilistic (generative) latent variable models such as state space models have been applied to the problem of posterior inference of latent state paths $\bm{z}_t\sim p(\bm{z}_t|\bm{x}_{1:T})$ of DS given time series observations $\{\bm{x}_{1:T}\}$ \citep{
pandarinath_inferring_2018, 
ghahramani_learning_1998, durstewitz_state_2017, krishnan_structured_2017}. The advantage here is that they also account for uncertainty in the model formulation or latent process itself and yield the full distribution over latent space variables \citep{karl_deep_2017}. For DS reconstruction, however, we need to move beyond posterior inference: We require that samples drawn from the model's prior distribution $p(\bm{z})$ after training exhibit the same (invariant) temporal and geometric structure as those produced by the unknown DS, a property that is not automatically guaranteed in this class of algorithms. 

Here we show that PLRNNs augmented with a linear spline expansion can be most efficiently trained by BPTT using a specific form of TF (Appx. \ref{supp-TF-BPTT}). We also embed expanded PLRNNs into a fully probabilistic, variational approach that scales well with system size by employing stochastic gradient variational Bayes (SGVB; \citep{kingma_auto-encoding_2014, rezende_stochastic_2014}), thereby combining the advantages of the two classes of models reviewed above, but with mild detriments in DS reconstruction performance compared to BPTT.
\section{Model Formulation and Theoretical Considerations}\label{sec:method}

\subsection{Piecewise Linear Recurrent Neural Network (PLRNN)}

Our approach builds on PLRNNs~\citep{durstewitz_state_2017, koppe_identifying_2019} because of their mathematical tractability (see Sec. \ref{sec:method:theory}). PLRNNs are defined by the $M$-dimensional latent process equation
\begin{align}\label{eq:plrnn_lat}
	\bm{z}_t = \bm{A} \bm{z}_{t-1} + \bm{W}  \phi(\bm{z}_{t-1}) + \bm{h} + \bm{C s}_t, 
	\end{align}
which describes the temporal evolution of $M$-dimensional latent state vector $\bm{z}_t=(z_{1t} \dots z_{Mt})^T$. The self-connections of the units are represented by diagonal matrix $\bm{A} \in \mathbb{R}^{M \times M}$, whereas the connections between units are collected in off-diagonal matrix  $\bm{W} \in \mathbb{R}^{M \times M}$, with the nonlinear activation function $\phi$ given by the rectified linear unit (ReLU) applied element-wise:
\begin{equation}\label{eq:relu}
	\phi(\bm{z}_{t-1})=\max(0, \bm{z}_{t-1}).
\end{equation}
The diagonal terms in $\bm{A}$ can be interpreted as the system's ``passive" (in the absence of inputs) time constants such that different latent states may capture different time scales of the underlying DS (as illustrated in Fig. \ref{fig:bn_time_scales}; see also  \citet{schmidt_identifying_2021}). The PLRNN also has a bias term $\bm{h} \in \mathbb{R}^{M}$ and accommodates potential external inputs $\bm{s}_t\in \mathbb{R}^{K}$ weighted by $\bm{C} \in \mathbb{R}^{M \times K}$. In a fully probabilistic framework, furthermore a Gaussian noise term $\bm{\epsilon}_t\sim \mathcal{N}(\bm{0},\bm{\Sigma})$ with diagonal covariance $\bm{\Sigma}$ is added to Eq. \ref{eq:plrnn_lat}. 
The PLRNN can be interpreted as a discrete-time neural rate model \citep{durstewitz_state_2017}, where the entries of $\bm{A}$ stand for the individual neurons' time constants, $\bm{W}$ for the synaptic connection strengths between neurons, and $\phi(z)$ for a (ReLU-shaped) voltage-to-spike-rate transfer function.
The latent RNN Eq. \ref{eq:plrnn_lat} is linked to the $N$-dimensional observed time series $(\bm{x}_t)_{t=1\ldots T}$, $\bm{x}_t\in \mathbb{R}^{N}$, drawn from an underlying noisy DS, by an observation function (decoder model) which, in the simplest case, may take the linear Gaussian form
\begin{equation}\label{eq:plrnn_obs}
	\bm{x}_t = \bm{B} \bm{z}_t + \bm{\eta}_t,
\end{equation}
where $\bm{B} \in \mathbb{R}^{N \times M}$ represents a factor loading matrix and $\bm{\eta}_t \sim\mathcal{N}(\bm{0},\bm{\Gamma})$ is Gaussian observation noise with diagonal covariance $\bm{\Gamma} \in \mathbb{R}^{N \times N}$ (only explicitly estimated as a free parameter in the variational approach).

\subsection{Dendritic Computation and Spline Basis Expansion}
Dendrites have long been known to play an active and important part in neural computation \citep{mel_information_1994, mel_why_1999, koch_biophysics_2004
}. Active, fast voltage-gated ion channels endow dendrites with strongly nonlinear behavior, giving rise for instance to dendritic $\mathrm{Ca}^{2+}$ spikes that boost synaptic inputs \citep{schiller_nmda_2000, 
hausser_diversity_2000
}. It has been suggested previously that different dendritic branches may constitute rather independent computational sub-units whose outputs are combined at the soma, as in a 2-layer neural network \citep{poirazi_pyramidal_2003, mel_synaptic_1993, mel_information_1994
}, an idea that received strong empirical support especially in recent years \citep{poirazi_illuminating_2020
}. Here we mimic this functional setup by modeling dendritic processing through a linear combination of ReLU-type threshold-nonlinearities (Fig. \ref{fig:base_illustration}), replacing Eq. \ref{eq:relu} by
\begin {equation}\label{eq:basis_expansion}
	\phi(\bm{z}_{t-1})=\sum_{b=1}^B \alpha_b \max (0, \bm{z}_{t-1} - \bm{h}_b),
\end{equation}
with \say{dendritic input/output} slopes $\alpha_b \in \mathbb{R}$ and \say{activation} thresholds $\bm{h}_b \in \mathbb{R}^M$. As in real dendrites, where both ion channels and morphological structure are subject to learning \citep{poirazi_illuminating_2020, stemmler_how_1999}, we treat these as trainable parameters. 
To emphasize the connection to dendritic computation we call the system Eqs. \ref{eq:plrnn_lat}, \ref{eq:plrnn_obs}, \ref{eq:basis_expansion}, the \textit{\ourmethodname}. 

We note that Eq. \ref{eq:basis_expansion} inserted into model Eq. \ref{eq:plrnn_lat} takes the form of a linear spline basis expansion as popular in statistics and machine learning \citep{hastie_elements_2009} for approximating arbitrary functions \citep{wahba_spline_1990,storace_piecewise-linear_2004
} in regression settings and other model-based approaches. 
In our context of DS reconstruction, however, there are particular challenges associated with such an approach, as we would like to preserve certain mathematical properties of the expanded model for its DS tractability. This is indeed one major contribution of the present work and addressed in the next section.
\begin{figure}[!htp]
    \centering
    \includegraphics[width=0.99\linewidth]{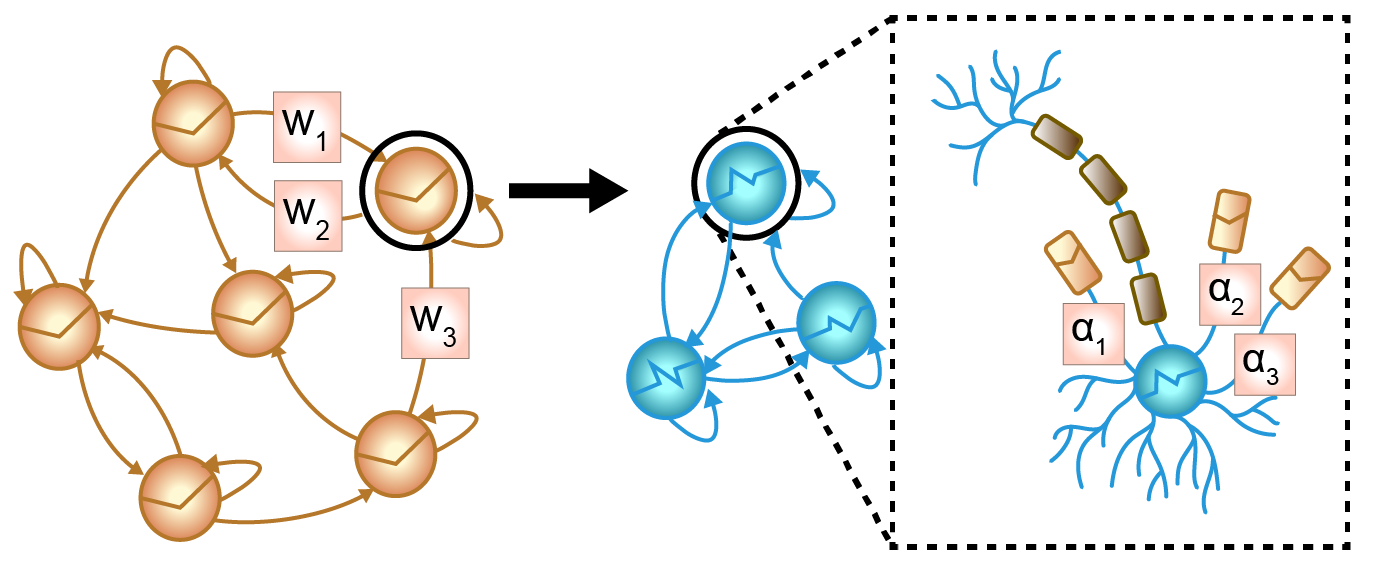}
    \caption{Inspired by principles of dendritic computation, our \ourmethodname\ extends each unit into a set of nonlinear branches connected to a soma, yielding single unit transfer functions with increased approximation capabilities. Figure created with the artistic support of Darshana Kalita.}
    %
    
    \label{fig:base_illustration}
\end{figure}

\subsection{Mathematical Tractability and Dynamical Systems Interpretation}\label{sec:method:theory}
Sharp threshold-nonlinearities (like a ReLU) are a reasonable choice from a neurobiological perspective, as dendrites naturally give rise to this threshold-type behavior \citep{mel_why_1999, koch_biophysics_2004}. Another important consideration in choosing this particular form, however, was that it preserves all the theoretically appealing properties of a PLRNN, as we will formally establish below: For PLRNNs fixed points and cycles can be explicitly computed \citep{schmidt_identifying_2021, koppe_identifying_2019}, and they can be translated into dynamically equivalent continuous-time systems \citep{monfared_transformation_2020}, properties which profoundly ease the analysis of trained systems from a DS perspective.\footnote{For instance, for a PLRNN trained on the Lorenz-63 system (see sect. \ref{sec:experiments}), we exactly located all fixed points in less than $1$ s and cycles up to $40^{th}$ order within $20$ s on a single $1.8$GHz CPU.} This is crucial for application in the sciences, where we are specifically interested in understanding the underlying system's dynamics. For PLRNNs, precise connections between the long-term behavior of the system and that of its gradients have also been established \citep{schmidt_identifying_2021, monfared_2022}. Finally, PLRNNs belong to the class of continuous piecewise-linear (PWL) maps, for which many important types of bifurcations have been well characterized \citep{
feigin_increasingly_1995, 
hogan_dynamics_2007, patra_multiple_2018} (see \citep{monfared_existence_2020} for an overview). Bifurcations are essential to understand how geometrical and topological properties of the system's state space depend on its parameters or could be controlled, and hence are also important to characterize or improve the training process itself \citep{doya_bifurcations_1992, pascanu_difficulty_2013, saxe_exact_2014}, or to understand properties of trained systems \citep{ maheswaranathan_reverse_2019, maheswaranathan_universality_2019}.

Our first proposition, therefore, assures that by the particular form of basis expansion introduced in \eqref{eq:basis_expansion}, the system will remain within the class of continuous PWL maps:

\begin{proposition}\label{pro-1}
    The model defined through \eqref{eq:plrnn_lat} and  \eqref{eq:basis_expansion} constitutes a continuous PWL map.
\end{proposition}

The proof essentially straightforwardly follows from the model's definition as a linear spline basis expansion in each unit, but is formally provided in Appx. \ref{p-pro-1}.

While Proposition \ref{pro-1} is all we need to ensure we can harvest all previously established results on PLRNNs in particular, and on continuous PWL maps more generally, it is revealing to note that any \ourmethodname\ (Eqs. \ref{eq:plrnn_lat}, \ref{eq:basis_expansion}) can be rewritten as a conventional PLRNN, as stated in the following theorem:

\begin{theorem}\label{pro-2}
    Any $M$-dimensional \ourmethodname\ as defined in Eqs. \ref{eq:plrnn_lat}, \ref{eq:basis_expansion}, can always be rewritten as a $M \times B$-dimensional \say{conventional} PLRNN of the form
    \begin{align}\label{eq-MxK-st}
     \hat{\vz}_t \, = \, & \tilde{\mA} \hat{\vz}_{t-1} + \tilde{\mW}  \, \max (0, \hat{\vz}_{t-1})  \,
   +\hat{\vh}_0
        %
        %
   \, + \, \tilde{\mC \vs_t} \, + \, \tilde{\bm{\epsilon}}_t.
   \end{align}
\end{theorem}
\begin{proof}
Straightforward by construction, see Appx. \ref{p-pro-2}.
\end{proof}

This theorem highlights why the \ourmethodname\ will allow to reduce the dimensionality of the reconstructed system, 
as it suggests we may often be able to reformulate a high-dimensional PLRNN in terms of an equally powerful lower-dimensional \ourmethodname.
In Appx. \ref{s-zm-fp-sec} we also spell out the exact computation of fixed points and $k$-cycles for the \ourmethodname.

Finally, the unboundedness of the PLRNN's latent states due to the ReLU function can cause divergence problems in training. The \ourmethodname, on the other hand, offers a simple and natural way to contain the latent states without violating the basic model description above, as established in the following theorem:
\begin{theorem}\label{pro-J-1}
For each basis $\{\alpha_b,\bm{h}_b\}$ in Eq. \ref{eq:basis_expansion} of a \ourmethodname\ let us add another basis $\{\alpha_b^*,\bm{h}_b^*\}$ with $\alpha_b^*=-\alpha_b$ and $\bm{h}_b^*=\mathbf{0}$. 
Then, for $\sigma_{\max}(\mA) < 1$, any orbit of this \say{clipped} \textit{\ourmethodname} (Eq. \ref{eq-clipped}) will remain bounded. \end{theorem}
\begin{proof}
See Appx. \ref{pro-J-1-proof}.
\end{proof}

Appx. \ref{supp-theo-analysis} collects further theoretical results, assuring, for instance, that the manifold attractor regularization employed here (see next section) does not interfere with the results above (Proposition \ref{pro-Z-3}).

\subsection{Training the dendPLRNN}
We apply two different training strategies to infer the parameters $\theta = \{\bm{A}, \bm{W}, \bm{h}, \bm{C}, \bm{\Sigma}, \bm{B}, \bm{\Gamma}, \{\alpha_b, \bm{h_b}\}\}$ of the \ourmethodname\ (Eq.~\ref{eq:plrnn_lat},~\ref{eq:plrnn_obs},~\ref{eq:basis_expansion}) from observed data: First, we employ \say{classical} BPTT with a variant of TF 
\citep{williams_learning_1989,pearlmutter_1990}. TF here means that the first $N$ latent states $z_{k,l \tau+1}, k \leq N$, were replaced by observations $x_{k,l \tau+1}$ at times $l \tau+1, l \in \mathbb{N}_0$, where $\tau \geq 1$ is the forcing interval (for details, see Appx. \ref{supp-TF-BPTT}).
Second, we use a fast and scalable variational inference (VI) algorithm which maximizes the Evidence Lower Bound (ELBO) $\mathcal{L}(\theta, \phi; \bm{x}):= \mathbb{E}_{q_\phi}[\log(p_\theta(\bm{x}|\bm{z})] -\mathrm{KL}[q_\phi(\bm{z}|\bm{x})||p_\theta(\bm{z})]$ using the reparameterization trick \citep{kingma_auto-encoding_2014}, and convolutional neural networks (CNNs) for parameterizing the encoder model $q_\phi(\bm{z}|\bm{x})$ (see Appx. \ref{sec:supp:hypers} for details).
Furthermore, as proposed in \citet{schmidt_identifying_2021}, to efficiently capture DS at multiple time scales, for VI we add a regularization term to the ELBO that encourages the mapping of slow time constants and long-range dependencies (so-called \textit{manifold attractor regularization}, see Eq. \ref{eq:supp:MAR}, with regularization factor $\lambda$). 
\section{Experiments}\label{sec:experiments}

\subsection{Performance Measures}\label{sec:metrics}

In DS reconstruction, we 
aim to capture \textit{invariant} properties of the underlying DS like its geometrical and temporal structure. To evaluate the quality of reconstructions w.r.t. \textit{geometrical properties} we employed a Kullback-Leibler divergence (\klmetric) that quantifies the agreement in attractor geometries (more details in Appx. \ref{sec:supp:metrics}), as first suggested in \citet{koppe_identifying_2019} (see also \citet{schmidt_identifying_2021}). Specifically, this measure evaluates the overlap between the observed data distribution $p(\bm{x}^{\textrm{obs}})$ and the distribution $p(\bm{x}^{\textrm{gen}}|\bm{z}^{\textrm{gen}})$ generated from model simulations (i.e., with $\bm{z}^{\textrm{gen}} \sim p_{\theta}(\bm{z})$ after model training\footnote{For deterministic latent models this comes down to just forward-iterating Eq. \ref{eq:plrnn_lat} from various random initial conditions.}) across state \textit{space} (not time!). Since this measure as originally defined in \citet{koppe_identifying_2019} is expensive to compute, for the high-dimensional benchmark DS we used another approximation, details of which are given in Appx. \ref{sec:supp:metrics}. \klmetric\ is evaluated on a test set of 100 trajectories with randomly sampled initial conditions and 1000 time steps each. 
To assess the agreement in \textit{temporal structure}, power spectra were first computed through the Fast Fourier Transform \citep{cooley_algorithm_1965} on all dimensions (i.e., time series) and slightly smoothed with Gaussian kernels to remove noise. For each dimension, the power spectral correlation (PSC) between ground truth and model-generated time series was then computed and averaged across dimensions (see Appx. \ref{sec:supp:metrics}).  
Finally, we also computed a 20-step-ahead prediction error along test set trajectories to assess short-term behavior (see Appx. \ref{sec:supp:metrics}). We note, however, that prediction errors can be misleading in case of chaotic systems because of exponential trajectory divergence, as illustrated in \citet{koppe_identifying_2019} (i.e., may be large even if the true system has been accurately captured, and vice versa). They therefore need to be interpreted with caution and are less relevant 
in the context of DS reconstruction than the statistics introduced above. 

\subsection{DS Benchmarks Used for Evaluation}\label{sec:exp:metrics}

We evaluated our approach and the specific role of the basis expansion on 
six different types of challenging DS benchmarks.
 
First, the famous 3d chaotic Lorenz attractor (Lorenz-63) originally proposed by \citet{lorenz_deterministic_1963} (formally defined in Appx. \ref{sec:supp:lorenz}) has become a popular benchmark for DS reconstruction algorithms. Fig. \ref{fig:low_dim_examples}\textcolor{black}{a} (l.h.s.) illustrates true (blue) and reconstructed (orange) time series from this system, while the r.h.s. illustrates the chaotic attractor's geometry in its 3d state space for both the ground truth (blue) and reconstructed (orange) systems.
It is important to note that both the time and state space graphs are not merely ahead predictions from the \ourmethodname\, but are produced by \textit{simulating} the trained \ourmethodname\ from some initial condition. This illustrates that the \ourmethodname\ has captured the temporal and geometrical structure of the original Lorenz-63 system in its own governing equations.
Moreover, computing analytically (see Appx. \ref{s-zm-fp-sec}) the fixed points of the reconstructed system, we see that their positions in state space agree well with those of the true system. 

Second, a 3d biophysical model of a bursting neuron (see Eq. \ref{eq:bursting_neuron} in Appx. \ref{sec:supp:burstingneuron}; \citet{durstewitz_implications_2009}) highlights another aspect of DS reconstruction: Besides an equation for membrane voltage ($V$), the model consists of one very fast ($n$) and one slow ($h$) variable that control the gating of the model's ionic conductances. This produces fast spikes that ride on top of a much slower oscillation, making this system challenging to reconstruct.  
One such successful \ourmethodname\ reconstruction is illustrated in Fig. \ref{fig:low_dim_examples}\textcolor{black}{b} (orange) together with time graphs and state space representations of the true system (blue).
 \begin{figure}
\centering
\includegraphics[width=.95\linewidth]{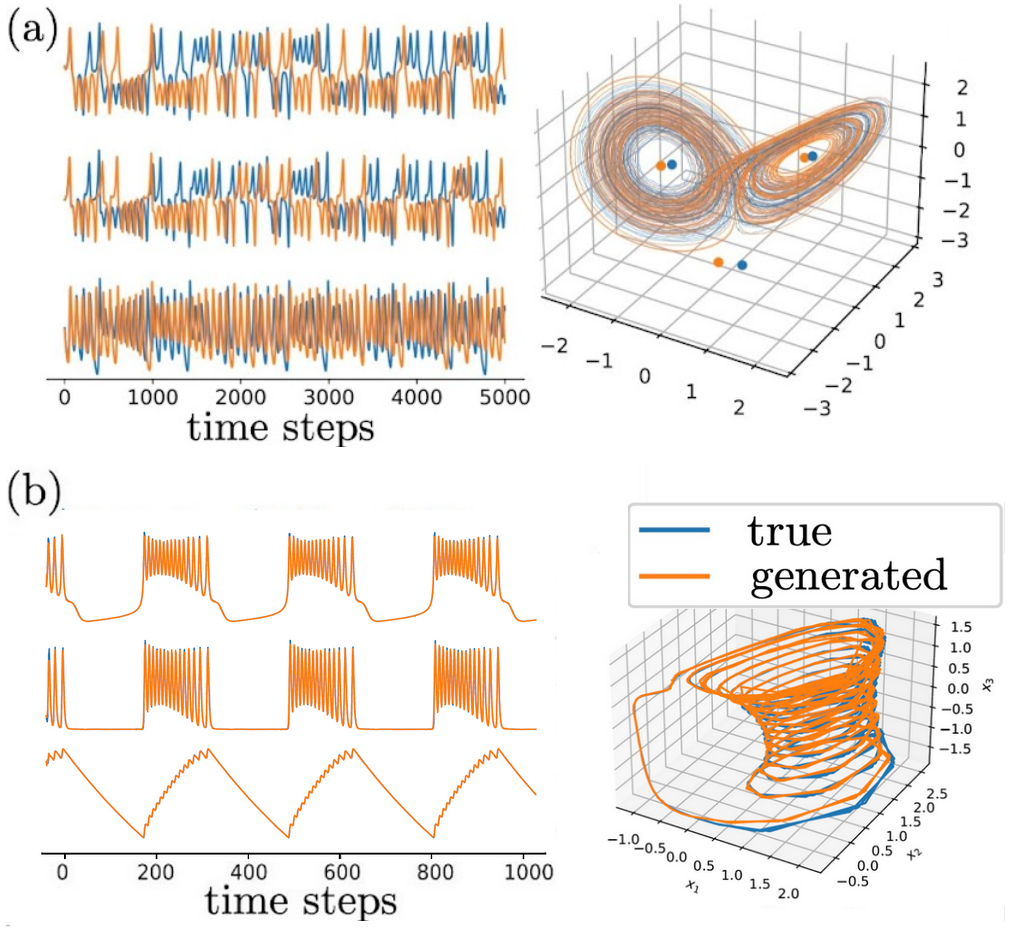}  
\caption{Examples of low-dimensional model reconstructions: (a) Time series (left) and state space trajectories (right) for the original Lorenz-63 chaotic attractor and simulations produced by a \ourmethodname\ trained with VI ($B=20$, $M=15$, $\lambda=1.0$, $M_{\textrm{reg}}/M=0.5$). Dots indicate true and reconstructed fixed points. (b) Same for the bursting neuron model, produced by a \ourmethodname\ trained with TF ($B=47$, $M=26$, $\tau=5$). Note that the bursting is a complex limit cycle but \textit{non-chaotic}.}
\label{fig:low_dim_examples}
\end{figure}

Third, the Lorenz-96 weather model is an example of a higher-dimensional, spatially organized chaotic system with local neighborhood interactions that can be extended to arbitrary dimensionality (Eq. \ref{eq:lorenz-96} in Appx. \ref{sec:supp:lorenz96}). It has also been used more widely for benchmarking DS reconstruction algorithms. For our experiments we employed a 10-dimensional spatial layout. Fig. \ref{fig:high_dim_examples}\textcolor{black}{a} illustrates time graphs for selected dimensions (top), the full evolving spatio-temporal pattern (center), and examples of power spectra (bottom) for both the ground truth system (blue) and an example reconstruction (orange). The spatio-temporal characteristics of the true and the \ourmethodname-generated time series tightly agree. 

Fourth, as another high-dimensional example we used a neural population model with structured connectivity tuned to produce coherent chaos \citep{landau_coherent_2018}, from which we produced 50d observations (see Appx. \ref{sec:supp:neuralpopulation} for details). Fig. \ref{fig:high_dim_examples}\textcolor{black}{b} provides example time series (top), full spatio-temporal patterns (center), and overlaid power spectra (bottom) for time series drawn from the true system (blue) and those simulated by a trained \ourmethodname\ (orange). Again there is a tight agreement, and again we emphasize that - like in all the other examples - these are not mere model ahead-predictions but fully simulated 
from some random initial condition.

\begin{figure}[ht]
\centering
\includegraphics[width=.99\linewidth]{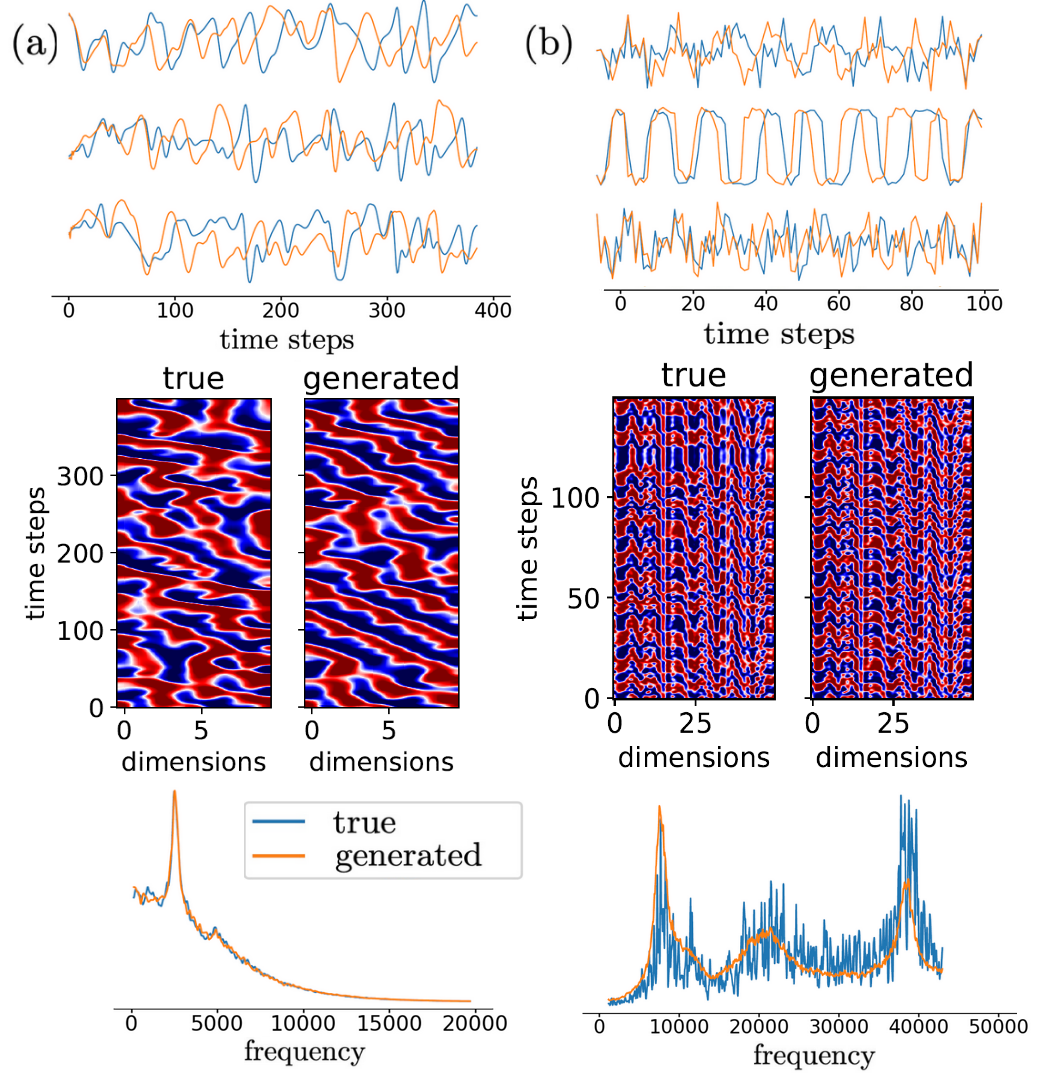}  
\caption{Examples of high-dimensional model reconstructions: (a) Time series (top), spatio-temporal evolution (middle), and power spectra (bottom) for the true 10d Lorenz-96 system and for \ourmethodname\ simulations ($B=50$, $M=30$, $\tau=10$). 
(b) Same for a 50d neural population model producing coherent chaos ($B=5$, $M=12$, $\lambda=1.0$, $M_{\textrm{reg}}/M=0.2$).}
\label{fig:high_dim_examples}
\end{figure}
Finally, we studied two real-world datasets consisting of electroencephalogram (EEG) recordings from human subjects, described in more detail with results (Fig. \ref{fig:EEG_trajectory}) in Appx. \ref{sec:supp:EEG}, and an electrocardiogram (ECG), recorded from a human subject with a chest sensor (Fig. \ref{fig:ECG_trajectory}), described in more detail in Appx. \ref{sec:supp:ECG}.

\subsection{Basis Expansion Allows for Reduced Dimensionality}  
Fig. \ref{fig:effect_basis_expansion} shows the reconstruction performance of the \ourmethodname\ on the Lorenz-63 DS when trained with a range of different numbers of bases $B$ and latent states $M$. As conjectured in Sec. \ref{sec:method} and confirmed by these results, the latent space dimensionality $M$ can indeed be reduced profoundly, at no loss in geometrical reconstruction quality (as assessed by \klmetric), by increasing the expansion order $B$.
\begin{figure}[htp!]
\centering
\begin{subfigure}{.49\textwidth}
  \centering
  \includegraphics[width=0.81\linewidth]{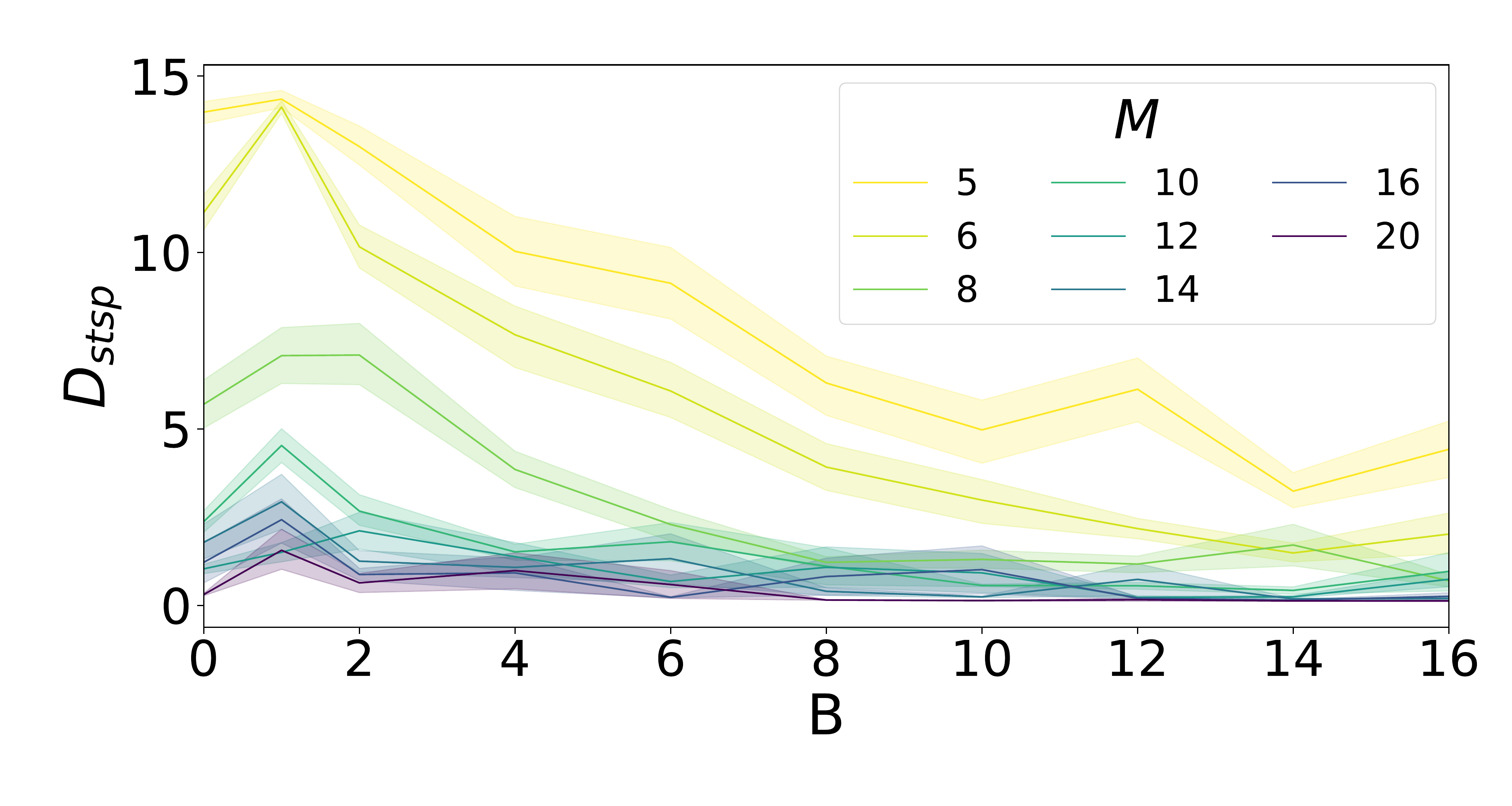}
\end{subfigure}%
\\
\begin{subfigure}{.49\textwidth}
  \centering
  \includegraphics[width=.845\linewidth]{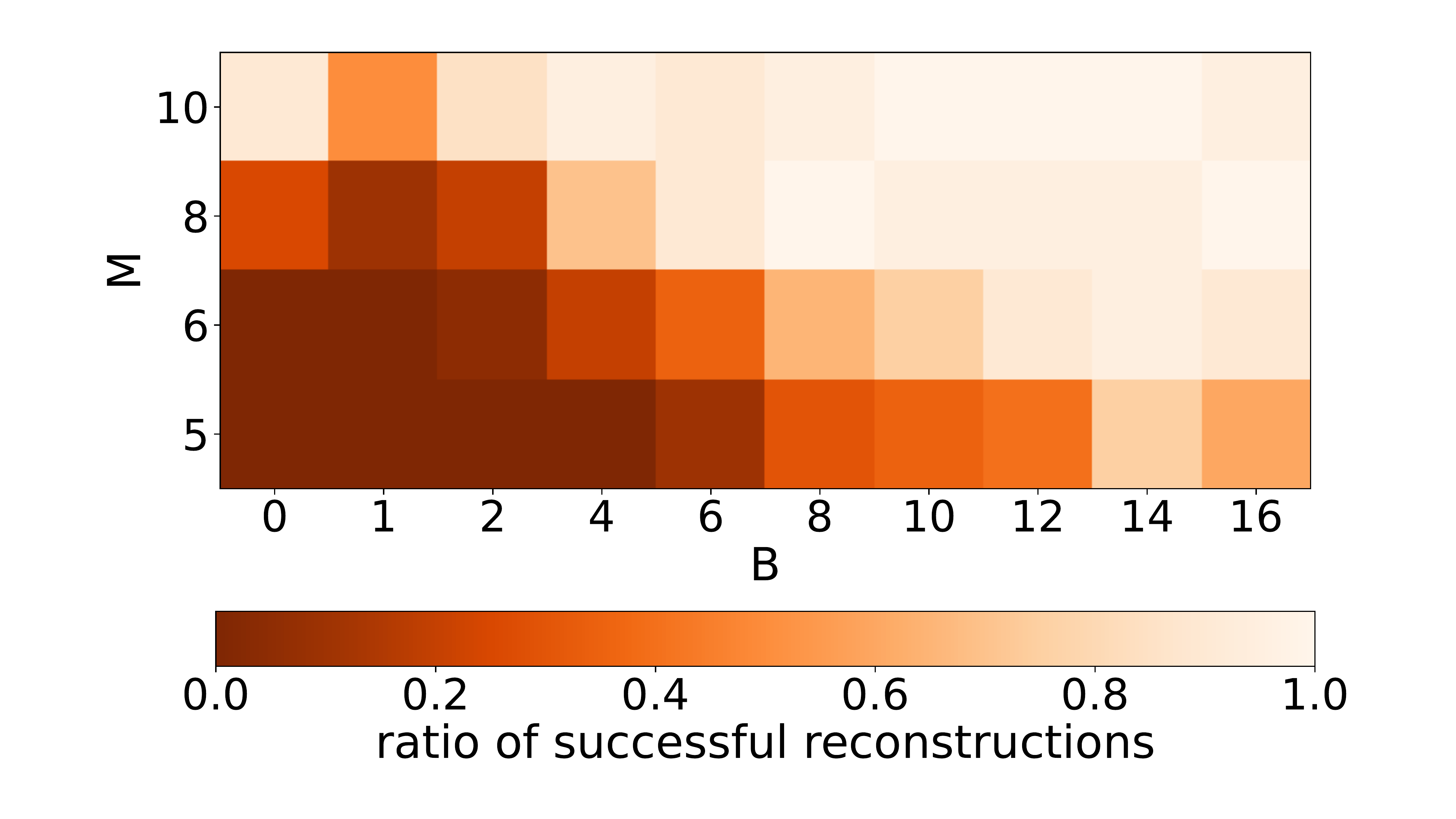}
\end{subfigure}
\caption{Effect of basis expansion on latent space dimensionality. Agreement in attractor geometries (top) and proportion of successful reconstructions (bottom) for the Lorenz-63 system as a function of the number of bases ($B$) and latent states ($M$). $B=0$ in the top graph denotes the standard PLRNN (no basis expansion). Each data point corresponds to 20 independent runs from different initial parameter configurations. In the bottom graph, runs with $D_{stsp} < 4$ were defined as successful as by inspection (but similar results are obtained with other choices for the $D_{stsp}$ threshold).
}
\label{fig:effect_basis_expansion}
\end{figure}
\subsection{Model Comparisons}
We compared our model to the PLRNN without dendritic expansion and four other algorithms purpose-tailored for DS reconstruction: First, SINDy \citep{brunton_discovering_2016} aims to reconstruct the governing equations by approximating numerical derivatives (obtained by differencing the time series, and applying a variance regularization to reduce noise) through a large library of (usually polynomial) basis functions. Sparse (LASSO) regression is used to pick out the right terms from the library (we used the PySINDy implementation \citep{de_silva_pysindy_2020} with multinomials up to sixth order). Second, \citet{vlachas_data-driven_2018} used a hybrid of LSTMs, trained using truncated BPTT, 
and mean-field stochastic models based on Ornstein-Uhlenbeck processes (LSTM-MSM) to approximate the true system's vector field estimated from observed time series. Third, \citet{pathak_model-free_2018} built on reservoir computing (RC) for their approach with reservoir parameters chosen to satisfy the \say{echo state property} \citep{jaeger_harnessing_2004}. For higher-dimensional systems, a spatially arranged set of reservoirs with local neighborhood relations is employed. Fourth, Neural-ODEs \citep{chen_neural_2018} were trained based on an implementation in the \texttt{torchdiffeq} package using the \texttt{odeintadjoint} method for the backward pass. 
For all these systems, we performed grid searches for optimal hyperparameters (see Appx. \ref{sec:supp:hypers}). For our own system, the \ourmethodname, we also performed a grid search for optimal hyper-parameters $\lambda$,
$\tau$, 
$M$, and $B$ (see Appx. \ref{sec:supp:hypers} and Table \ref{tab:reg} for details). 
For all five methods, to the degree possible we tried to ensure roughly the same number of trainable parameters 
(see Table \ref{tab:benchmarks}).
\begin{table*}[!b]

\definecolor{Gray}{gray}{0.9}
\caption{Comparison of \ourmethodname\ (Ours) trained by BPTT+TF, RC 
\citep{pathak_model-free_2018}, LSTM-MSM 
\citep{vlachas_data-driven_2018}, SINDy 
\citep{brunton_discovering_2016} and Neural ODE \citep{chen_neural_2018} on 4 DS benchmarks and two experimental datasets (top) and 3 challenging data situations (bottom). Values are mean $\pm$ SEM.} 
\centering
\scalebox{0.80}{
\begin{tabular} {l l r@{ \,$\pm$\, }l r@{ \,$\pm$\, }l r@{ \,$\pm$\, }ll l l l}
        \toprule
        Dataset	&	Method	&	\multicolumn{2}{c}{PSC}	&	\multicolumn{2}{c}{$D_{\textrm{stsp}}$}  & \multicolumn{2}{c}{20-step PE}	&  Dyn.var. & \#parameters \\
        \midrule
        \multirow{5}{4em}{Lorenz-63}
        
        & {\ourmethodname} TF	    &    $0.997$
 	&   $0.002$	&	$\num{0.13}$    &   $\num{0.18}$  & ${9.2\mbox{e$-$}5}$  & ${2.8\mbox{e$-$}5}$  &  $22$	    & $1032$  \\
        &	 {RC}	                &	$\num{0.991}$ 	    &   $\num{0.001}$		&	${0.24}$       &   ${0.05}$	& ${1.2\mbox{e$-$}2}$  & ${0.1\mbox{e$-$}3}$ &   $345$       & $1053$  \\
        &	 {LSTM-MSM}	                &	$0.985$ 	        &   ${0.004}$		    &	$\num{0.85}$    &   $\num{0.07}$ & ${1.2\mbox{e$-$}2}$  & ${0.1\mbox{e$-$}3}$ & $29$    & $1035$  \\
        &	 {SINDy}	            &	$\mathbf{0.998}$    &   $\mathbf{0.0003}$	&	$\mathbf{0.04}$ &   $\mathbf{0.01}$ & $\mathbf{6.8\mbox{e$-$}5}$  & $\mathbf{0.2\mbox{e$-$}5}$ & $3$         & $252$ \\
        &	 {Neural ODE}	            &	${0.992}$    &   ${0.001}$	&	${0.149}$ &   ${0.014}$ & ${1.1\mbox{e$-$}3}$  & ${4.1\mbox{e$-$}5}$ & $3$         & $1011$
        
        \\
        \midrule
        \multirow{5}{4em}{Bursting Neuron}
        & {\ourmethodname} TF	&	$\mathbf{0.76}$ & $\mathbf{0.04}$	        &	$\mathbf{0.61}$ &$\mathbf{0.09}$      & ${6.1\mbox{e$-$}2}$  & ${2.2\mbox{e$-$}2}$           & $26$          & $2040$    \\
        
        &	 {RC}	        &	$\num{0.51}$ &$\num{0.01}$	        &	$\num{5.1}$ &$\num{0.6}$           & ${8.6\mbox{e$-$}2}$  & ${0.1\mbox{e$-$}2}$            & $711$         & $2133$	\\
        &	 {LSTM-MSM}	        &	$0.54$ &$0.02$	&	${2.83}$ &${0.36}$	  & $\mathbf{3.9\mbox{e$-$}2}$  & $\mathbf{0.1\mbox{e$-$}2}$         & $45$	    & $2166$    \\
        &	 {SINDy}	           &	$0.25$ &$0.01$	&	${6.36}$ &${0.02}$	  & ${5.4\mbox{e$-$}1}$  & ${0.1\mbox{e$-$}2}$                  & $3$           & $252$       \\
        &	 {Neural ODE}	            &	${0.65}$    &   ${0.017}$	&	${3.85}$ &   ${0.1}$ & ${2.1\mbox{e$-$}1}$  & ${0.5\mbox{e$-$}2}$ & $3$        & $2073$ \\

            \midrule
        \multirow{5}{4em}{Lorenz-96}
        & {\ourmethodname} TF	&		$\mathbf{0.998}$ 	& $\mathbf{0.0001}$	&	$\mathbf{0.04}$ &$\mathbf{0.01}$   & $\mathbf{4.1\mbox{e$-$}2}$  & $\mathbf{0.8\mbox{e$-$}2}$   &  $50$       & $4480$	\\
    
        &	 {RC}	 &	$\num{0.986}$ &$\num{0.008}$	&	$\num{0.25}$ &$\num{0.17}$    & ${7.1\mbox{e$-$}1}$  & ${0.1\mbox{e$-$}2}$       &  $440$	    & $4400$	\\        
        &	 {LSTM-MSM}	& 	$\num{0.993}$ &$\num{0.002}$	&	$\num{0.23}$ &$\num{0.03}$      & ${8.2\mbox{e$-$}1}$  & ${0.3\mbox{e$-$}2}$   &  $62$	& $4384$	\\
        &	 {SINDy}	&	${0.997}$ &${0.001}$	&	$0.06$ & 	$0.003$     & ${6.3\mbox{e$-$}2}$  & ${0.1\mbox{e$-$}3}$       &  $10$        & $27410$	    \\
        &	 {Neural ODE}	            &	${0.985}$    &   ${0.001}$	&	${0.21}$ &   ${0.02}$ & ${4.4\mbox{e$-$}2}$  & ${4.5\mbox{e$-$}3}$ & $10$         & $4130$ \\
        \midrule    
        
        
        \multirow{5}{4em}{Neural Population Model}
        & {\ourmethodname} TF	&	$\mathbf{0.52}$ &$\mathbf{0.01}$	&	${0.37}$ &${0.05}$    &	${1.43}$ &${0.01}$                                          & $75$        & $9990$	\\ 
        
        &	 {RC}	&	$\num{0.34}$ &$\num{0.03}$	&	$\num{2.8}$ &$\num{0.4}$                    &	$1.64$ &$0.07$                          & $200$      & $10000$	\\
        &	 {LSTM-MSM}	&	$0.51$ &$0.02$	&	$\mathbf{0.29}$ &$\mathbf{0.04}$	                 &	$1.56$ &$0.01$               & $56$	& $10298$   \\
        &	 {SINDy}	&	\multicolumn{2}{c}{diverging} &	\multicolumn{2}{c}{diverging}	&	\multicolumn{2}{c}{diverging}		                            & $50$      & $66300$    \\
        &	 {Neural ODE}	            &	${0.47}$    &   ${0.03}$	&	${9.56}$ &   ${0.86}$ & $\mathbf{0.58}$  & $\mathbf{0.006}$ & $50$         & $10200$ \\
        \midrule
        
        \multirow{5}{4em}{EEG}
        & {\ourmethodname} TF	&	$\mathbf{0.923}$ & $\mathbf{0.012}$	        &	$\mathbf{1.96}$ &$\mathbf{0.18}$      & $\mathbf{0.202}$  & $\mathbf{0.007}$           & $128$          & $27058$    \\
        
        &	 {RC}	        &	$\num{0.782}$ &$\num{0.002}$	        &	$\num{8.8}$ &$\num{0.8}$           & ${0.78}$  & ${0.02}$            & $448$         & $28672$	\\
        &	 {LSTM-MSM}	        &	$0.827$ &$0.002$	&	${8.3}$ &${0.3}$	  & ${0.708}$  & ${0.003}$        & $168$	    & $27728$    \\
        &	 {SINDy}	           &	\multicolumn{2}{c}{diverging}	    &	\multicolumn{2}{c}{diverging}	&	\multicolumn{2}{c}{diverging}	            & $64$           & $133120$       \\
        &	 {Neural ODE}	            &	${0.82}$    &   ${0.002}$	&	${21.72}$ &   ${0.71}$ & ${0.31}$  & ${0.005}$ & $64$         & $30559$ \\
        \midrule
        
        \multirow{5}{4em}{ECG}
        & {\ourmethodname} TF	&	$\mathbf{0.929}$ & $\mathbf{0.014}$	        &	$\mathbf{0.4}$ &$\mathbf{0.6}$      & ${0.23}$  & ${0.03}$           & $30$          & $2641$    \\
        
        &	 {RC}	        &	$\num{0.880}$ &$\num{0.013}$	        &	$\num{1.78}$ &$\num{0.44}$           & ${0.571}$  & ${0.013}$            & $378$         & $2646$	\\
        &	 {LSTM-MSM}	        &	$0.926$ &$0.007$	&	${0.59}$ &${0.08}$	  & $\mathbf{7.0\mbox{e$-$}2}$  & $\mathbf{0.6\mbox{e$-$}2}$        & $51$	    & $2801$    \\
        &	 {SINDy}	           &	\multicolumn{2}{c}{diverging}	    &	\multicolumn{2}{c}{diverging}	&	\multicolumn{2}{c}{diverging}	            & $7$           & $4424$       \\
        &	 {Neural ODE}	            &	${0.90}$    &   ${0.011}$	&	${1.18}$ &   ${0.02}$ & ${0.61}$  & ${0.01}$ & $7$         & $2599
$ \\
        \midrule
        \midrule
        \multirow{5}{4em}{Low amount of data}
        &	 {\ourmethodname} TF	&	$0.97$ & $0.04$ 	&	$6.9$ & $5.3$	   & ${1.5\mbox{e$-$}2}$  & ${0.9\mbox{e$-$}2}$            & $22$   & $1032$    \\
        &	 {RC}	     &	$0.68$ &$0.05$	        &	$\num{5.74}$ &$\num{0.11}$                  & ${4.1\mbox{e$+$}5}$  & ${1.2\mbox{e$+$}5}$                            & $345$     & $1053$    \\
        &	 {LSTM-MSM}	        &	$\num{0.960}$ &$\num{0.006}$	&	$\num{6.06}$ &$\num{0.37}$	& ${2.1\mbox{e$-$}1}$  & ${0.3\mbox{e$-$}2}$  	                & $29$    & $1035$    \\
        &	 {SINDy}	    &	$\mathbf{0.998}$ &$\mathbf{0.0003}$	&	$\mathbf{0.04}$ &$\mathbf{0.01}$ & $\mathbf{6.8\mbox{e$-$}5}$  & $\mathbf{0.2\mbox{e$-$}5}$  & $3$       & $252$	    \\
        &	 {Neural ODE}	            &	${0.967}$    &   ${0.008}$	&	${4.66}$ &   ${0.31}$ & ${1.6\mbox{e$-$}3}$  & ${1.8\mbox{e$-$}4}$ & $3$         & $1011$
        
        \\
        
        \midrule
        \multirow{5}{4em}{Partially observed}
        &	 {\ourmethodname} TF & $\mathbf{0.993}$ & $\mathbf{0.003}$	&	$\mathbf{0.54}$ & $\mathbf{0.16}$	  & $\mathbf{5.3\mbox{e$-$}3}$  & $\mathbf{0.2\mbox{e$-$}3}$                             & $22$	    & $1032$  \\
        
        &	 {RC}	& $0.981$ &$0.001$	&	$2.92$ &$0.08$	  & ${7.6\mbox{e$-$}3}$  & ${0.1\mbox{e$-$}3}$             & $345$     & $1053$  \\
        &	 {LSTM-MSM}	& $\num{0.934}$ &$\num{0.005}$	&	$\num{6.06}$ &$\num{0.37}$		     & ${2.3\mbox{e$-$}2}$  & ${0.3\mbox{e$-$}2}$                      & $29$    & $1035$  \\
        &	 {SINDy}	& $\num{0.974}$ &$\num{0.001}$	&	$\num{2.52}$ &$\num{0.01}$	       & ${7.4\mbox{e$-$}3}$  & ${0.1\mbox{e$-$}3}$                    & $3$         & $252$	    \\
        &	 {Neural ODE}	            &	${0.945}$    &   ${0.004}$	&	${3.34}$ &   ${0.12}$ & ${8.3\mbox{e$-$}3}$  & ${9\mbox{e$-$}5}$ & $3$         & $1011$
        
        \\
        \midrule
        \multirow{5}{4em}{High noise}
        &	 {\ourmethodname} TF	& $\mathbf{0.995}$	 &  $\mathbf{0.002}$               &	$\mathbf{0.4}$ & $\mathbf{0.13}$	             & $\mathbf{4.6\mbox{e$-$}3}$  & $\mathbf{0.4\mbox{e$-$}3}$                          & $22$	    & $1032$  \\
        
        &	 {RC}	& $0.988$ &$0.001$	&	$\num{2.33}$ &$\num{0.21}$	              & ${3.1\mbox{e$-$}2}$  & ${0.2\mbox{e$-$}2}$             & $345$       & $1053$  \\
        &	 {LSTM-MSM} &	 $\num{0.967}$ &$\num{0.006}$	        &	$1.19$ &$0.27$            & ${3.3\mbox{e$-$}2}$  & ${0.2\mbox{e$-$}2}$                 & $29$    & $1035$  \\
        &	 {SINDy} & $\num{0.984}$ &$\num{0.005}$	&	$0.42$ &$0.01$           & ${7.0\mbox{e$-$}3}$  & ${0.1\mbox{e$-$}4}$            & $3$         & $252$	    \\
        &	 {Neural ODE}	            &	${0.982}$    &   ${0.055}$	&	${0.79}$ &   ${0.06}$ & ${5.5\mbox{e$-$}3}$  & ${1.7\mbox{e$-$}4}$ & $3$         & $1011$
        
        \\
        \bottomrule
\end{tabular}
}
\label{tab:benchmarks}
\end{table*}

Results for all five models on all six DS benchmarks employed here are summarized in the upper part of Table \ref{tab:benchmarks}, using the 
temporal and geometrical reconstruction measures introduced in Sec. \ref{sec:exp:metrics} (as well as a 20-step-ahead prediction error for comparison\footnote{As pointed out in sect. \ref{sec:metrics}, for some of the chaotic systems we indeed observed, however, that lower prediction errors do not always go in hand with better DS reconstructions.}). To produce this table, 100,000 time steps for both training and testing were simulated from each ground truth system, 
all dimensions were standardized to have zero mean and unit variance, and process noise and observation noise (with about 1\% of the data variance) were added while simulating the (now stochastic) differential equations, and after drawing the observations, respectively (see Appx. \ref{sec:supp:dynsys} for further methodological details). To produce statistics, each method was run from a total of 20 randomly chosen initial conditions for the parameters. We also tested all five methods on real-world EEG and ECG data and on challenging data situations produced using the Lorenz-63 system (Fig. \ref{fig:low_dim_examples}\textcolor{black}{a}), with either short time series of just 1000 time steps for training, only partial observations (just state variable $x$ in Eq.\ref{eq:lorenz} in Appx. \ref{sec:supp:lorenz}), or high process and high observation noise (drawing from a Gaussian with $d\epsilon \sim\mathcal{N}(0,0.1 dt \times \mathbf{I})$ for the process noise as described in Appx. \ref{sec:supp:dynsys}, and using 10\% of the observation variance, respectively). 
SINDy cannot naturally handle missing observations, as it has no latent variables but formulates the model directly in terms of the observations. Therefore, for the partially observed system, we used a delay embedding \citep{takens_detecting_1981, sauer_embedology_1991} to create a 3d dataset, adding two time-lagged versions of $x$ 
as coordinates.\footnote{Note that SINDy by design always has as many dynamical variables as observed (or embedded) dimensions. This could be an advantage if the observed system really is that low-dimensional. If, however, the number of observations is much larger than those needed to describe the generating DS (as often suspected in neuroscience) or -- vice versa -- not all system states have been observed, 
SINDy may need to be augmented by other techniques (see \citet{champion_data-driven_2019,bakarji2022discovering}).
}

A general observation is that indeed all five models are quite powerful for reconstructing the underlying DS. However, in most comparisons the \ourmethodname\ had an edge over the other methods, or came out second after SINDy, especially when trained by BPTT+TF (see Appx. \ref{tab:supp:vae} for results obtained with VI).
SINDy tends to outperform the \ourmethodname\ on the Lorenz-63 DS, 
but it performs poorly on the bursting neuron example and 
fails on the neural population model, as well as on the EEG and ECG data. It also becomes comparatively slow to train on high-dimensional systems (as the number of bases needed scales exponentially with the number of dimensions). This can be explained by the fact that SINDy already has the correct functional form for the Lorenz-63 (and also Lorenz-96) DS: Both of these have a strictly polynomial form (see Eq. \ref{eq:lorenz} and Eq. \ref{eq:lorenz-96} in Appx. \ref{sec:supp:dynsys}), and SINDy (in our tests) works with a set of polynomial library functions to begin with. Hence, SINDy only needs to pick out the right terms from its expansion to succeed, giving it a clear advantage on these model systems by design. On the other hand, as indicated in Table \ref{tab:benchmarks}, it 
largely fails on systems which have a different (in this case non-polynomial) functional form, or when the true form, as in the EEG and ECG empirical examples, is simply not known. Unlike the other methods, SINDy therefore appears less suitable as a general framework for DS reconstruction if an appropriate library of basis functions cannot be specified a priori, a potential shortcoming already  
discussed by the original authors \cite{brunton_discovering_2016}. 

While our conclusion is that essentially all of the four tested models LSTM-MSM, RC, Neural ODE and \ourmethodname, are suitable for reconstruction of \textit{arbitrary} unknown DS, even in very challenging data situations (Table \ref{tab:benchmarks}, bottom), LSTM-MSM, RC and Neural-ODE performed worse on average 
and have other disadvantages compared to our method: First, they are quite complex in their architectures and hence not easily interpretable, i.e. much harder to track and analyze mathematically.\footnote{This is especially true for RC. Moreover, the fact that only the weights of the linear output layer are trainable while the recurrent connections within the reservoirs are static, may raise the question of what precisely is learnt in terms of dynamics if the reservoirs themselves cannot adapt to the DS at hand.} In contrast, as summarized in Sec. \ref{sec:method:theory}, the \ourmethodname\ is a continuous PWL map and as such comes with a huge bulk of already existing theoretical results \citep{schmidt_identifying_2021, monfared_existence_2020, monfared_transformation_2020}, as well as with mathematical tractability. This aspect is illustrated more explicitly in Fig. \ref{fig:wc_flow} which shows true and reconstructed vector fields and fixed point locations for the Wilson-Cowan model of neural population dynamics (see Appx. \ref{supp-wc}) in the bistable regime. Fixed points were computed analytically for the \ourmethodname\ (cf. also Fig. \ref{fig:low_dim_examples}\textcolor{black}{a} and Appx. \ref{s-zm-fp-sec}), and match those of the ground truth system both in position and stability as determined from the \ourmethodname's Jacobian. 
On top, the \ourmethodname\ mostly achieves reconstructions of the DS in (much) lower dimensions than RC or LSTM-MSM 
(see Table \ref{tab:benchmarks}), further adding to its better interpretability. 
By embedding the \ourmethodname\ within a SVAE \citep{archer_black_2015} framework (see Appx. \ref{supp-approxposterior}), one could also obtain uncertainty estimates on the state trajectories and perform posterior inference, 
features that the other models lack.
\begin{figure}[!htb]
    \centering
	\includegraphics[width=0.85\linewidth]{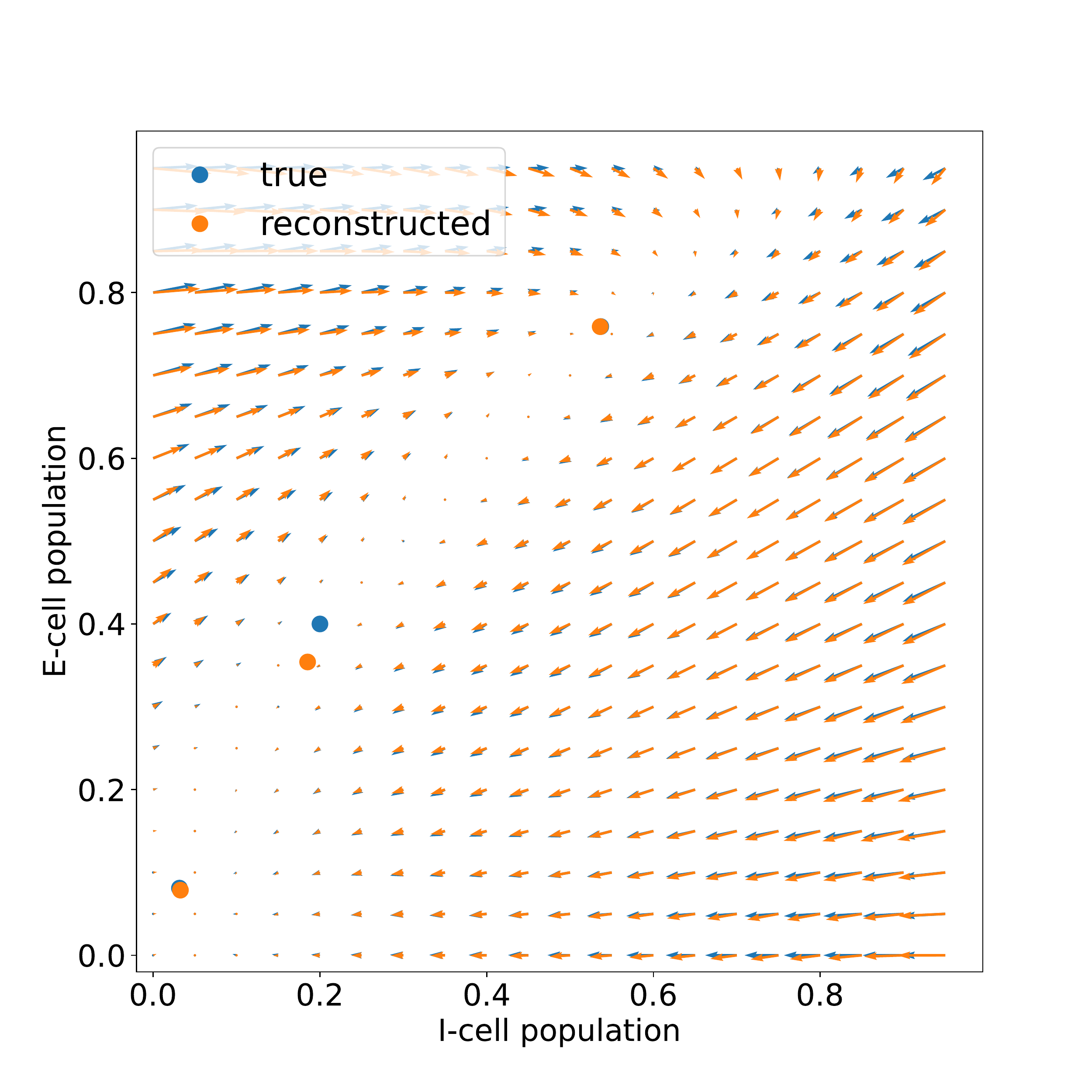}
	
	\caption{Comparison of ground truth vector field for the 2D Wilson-Cowan equations to a reconstruction obtained by the \ourmethodname\ with $M=5, B=20, \tau=15, M_{\textrm{reg}}/M=0.5, \lambda=5 \cdot 10^{-3}$. Also shown are locations of the true system's fixed points and those computed analytically for the trained \ourmethodname.}
\label{fig:wc_flow}
\end{figure}
 \section{Conclusions}\label{sec:conclusion}
In this work we augmented PLRNNs \citep{durstewitz_state_2017, koppe_identifying_2019} by a linear spline basis expansion inspired by principles of dendritic computation. We show mathematically that by doing so we remain within the theoretical framework of continuous PWL maps and hence can harvest a huge bulk of existing DS theory (Sec. \ref{sec:method:theory}), while at the same time achieving better performance with less parameters and in lower dimensions. 
This is a key advantage from a scientific perspective where mechanistic insight and understanding of the system under study is sought. Indeed, we are not aware of any other current DS reconstruction approach that combines these features, a simple, mathematically tractable design with competitive performance, yet providing comparatively low-dimensional representations of the dynamics. 
Another contribution of the present work is that it assembles a set of DS benchmarks, experimental settings, and reconstruction measures which may be helpful more generally for assessing DS reconstruction algorithms, including an extension of a geometrical reconstruction measure (\klmetric) for use in high dimensions.
Using a stochastic latent model and VI/SVAE for training  (see Appx. \ref{supp-approxposterior}) in addition yields posterior distributions and uncertainty estimates across latent state trajectories. 
Somewhat surprisingly, however, the BPTT+TF approach to model training clearly outperformed the more sophisticated VI approach (see Table \ref{tab:supp:vae}). This could be rooted in suboptimal encoder models or, as we suspect, in suboptimal sampling from the approximate posterior: BPTT+TF (and, similarly, LSTM-MSM) allow trajectories to evolve freely for several time steps during training and hence assess longer bits of trajectory for optimization. In contrast, in VI single time points are sampled separately from the approximate posterior and overall temporal consistency is ensured only through the Kullback-Leibler term in the ELBO. 
Other more expressive yet still fast to compute encoder models that better map a trajectory's temporal evolution, e.g., based on normalizing flows \citep{rezende_variational_2015}, may boost performance. Specific annealing and curriculum training protocols (as used in \citet{koppe_identifying_2019}) are other amendments to consider. 

Another potentially interesting direction would be to augment the model with multiple trainable and adaptive time scales, as in LEM networks \cite{rusch2022long}. While this is similar in spirit to the manifold attractor regularization scheme proposed for PLRNNs by \citet{schmidt_identifying_2021} (see also Fig. \ref{fig:bn_time_scales}), LEMs allow for a more flexible and state/input-dependent way of adjusting the system's time constants. 
In preliminary studies we observed LEM networks to 
come close to our model for DS reconstruction tasks. Borrowing LEM's basic functional principles while retaining the \ourmethodname's tractable form thus appears to be another promising though challenging avenue.
\section*{Software and Data}
All code created in here is available at \url{https://github.com/DurstewitzLab/dendPLRNN}.

 \section*{Acknowledgements}
 This work was funded by the German Research Foundation (DFG) within Germany’s Excellence Strategy – EXC-2181 – 390900948 (’Structures’), by DFG grant Du354/10-1 to DD, and the European Union Horizon-2020 consortium SC1-DTH-13-2020 ('IMMERSE').




\bibliography{files/basis_expansion.bib}

\begin{thebibliography}{93}
\providecommand{\natexlab}[1]{#1}
\providecommand{\url}[1]{\texttt{#1}}
\expandafter\ifx\csname urlstyle\endcsname\relax
  \providecommand{\doi}[1]{doi: #1}\else
  \providecommand{\doi}{doi: \begingroup \urlstyle{rm}\Url}\fi

\bibitem[Archer et~al.(2015)Archer, Park, Buesing, Cunningham, and
  Paninski]{archer_black_2015}
Archer, E., Park, I.~M., Buesing, L., Cunningham, J., and Paninski, L.
\newblock Black box variational inference for state space models.
\newblock \emph{arXiv preprint arXiv:1511.07367}, 2015.
\newblock URL \url{http://arxiv.org/abs/1511.07367}.

\bibitem[Azencot et~al.(2020)Azencot, Erichson, Lin, and
  Mahoney]{azencot_forecasting_2020}
Azencot, O., Erichson, N.~B., Lin, V., and Mahoney, M.~W.
\newblock Forecasting {Sequential} {Data} using {Consistent} {Koopman}
  {Autoencoders}.
\newblock In \emph{Proceedings of the 37th {International} {Conference} on
  {Machine} {Learning}}, 2020.
\newblock URL \url{http://arxiv.org/abs/2003.02236}.

\bibitem[Ba et~al.(2016)Ba, Kiros, and Hinton]{ba_layer_2016}
Ba, J.~L., Kiros, J.~R., and Hinton, G.~E.
\newblock Layer {Normalization}.
\newblock \emph{arXiv:1607.06450 [cs, stat]}, July 2016.
\newblock URL \url{http://arxiv.org/abs/1607.06450}.

\bibitem[Bakarji et~al.(2022)Bakarji, Champion, Kutz, and
  Brunton]{bakarji2022discovering}
Bakarji, J., Champion, K., Kutz, J.~N., and Brunton, S.~L.
\newblock Discovering governing equations from partial measurements with deep
  delay autoencoders.
\newblock \emph{arXiv preprint arXiv:2201.05136}, 2022.

\bibitem[Baydin et~al.(2018)Baydin, Pearlmutter, Radul, and
  Siskind]{baydin_automatic_2018}
Baydin, A.~G., Pearlmutter, B.~A., Radul, A.~A., and Siskind, J.~M.
\newblock Automatic {Differentiation} in {Machine} {Learning}: a {Survey}.
\newblock \emph{Journal of Machine Learning Research}, 18\penalty0
  (153):\penalty0 1--43, 2018.
\newblock ISSN 1533-7928.
\newblock URL \url{http://jmlr.org/papers/v18/17-468.html}.

\bibitem[Bayer et~al.(2021)Bayer, Soelch, Mirchev, Kayalibay, and van~der
  Smagt]{bayer_mind_2021}
Bayer, J., Soelch, M., Mirchev, A., Kayalibay, B., and van~der Smagt, P.
\newblock Mind the gap when conditioning amortised inference in sequential
  latent-variable models.
\newblock \emph{International Conference on Learning Representations}, 2021.
\newblock URL \url{http://arxiv.org/abs/2101.07046}.

\bibitem[Bengio et~al.(1994)Bengio, Simard, and Frasconi]{bengio1994learning}
Bengio, Y., Simard, P., and Frasconi, P.
\newblock Learning long-term dependencies with gradient descent is difficult.
\newblock \emph{IEEE transactions on neural networks}, 5\penalty0 (2):\penalty0
  157--166, 1994.

\bibitem[Blei et~al.(2017)Blei, Kucukelbir, and
  McAuliffe]{Blei_variational_2017}
Blei, D.~M., Kucukelbir, A., and McAuliffe, J.~D.
\newblock Variational inference: A review for statisticians.
\newblock \emph{Journal of the American Statistical Association}, 112\penalty0
  (518):\penalty0 859–877, Apr 2017.
\newblock ISSN 1537-274X.
\newblock \doi{10.1080/01621459.2017.1285773}.
\newblock URL \url{http://dx.doi.org/10.1080/01621459.2017.1285773}.

\bibitem[Brunton et~al.(2016)Brunton, Proctor, and
  Kutz]{brunton_discovering_2016}
Brunton, S.~L., Proctor, J.~L., and Kutz, J.~N.
\newblock Discovering governing equations from data by sparse identification of
  nonlinear dynamical systems.
\newblock \emph{Proc Natl Acad Sci U S A}, 113\penalty0 (15):\penalty0
  3932--3937, 2016.
\newblock ISSN 0027-8424.
\newblock \doi{10.1073/pnas.1517384113}.
\newblock URL \url{https://www.ncbi.nlm.nih.gov/pmc/articles/PMC4839439/}.

\bibitem[Brunton et~al.(2017)Brunton, Brunton, Proctor, Kaiser, and
  Kutz]{brunton_chaos_2017}
Brunton, S.~L., Brunton, B.~W., Proctor, J.~L., Kaiser, E., and Kutz, J.~N.
\newblock Chaos as an intermittently forced linear system.
\newblock \emph{Nat Commun}, 8\penalty0 (1):\penalty0 19, 2017.
\newblock ISSN 2041-1723.
\newblock \doi{10.1038/s41467-017-00030-8}.
\newblock URL \url{http://www.nature.com/articles/s41467-017-00030-8}.

\bibitem[Champion et~al.(2019)Champion, Lusch, Kutz, and
  Brunton]{champion_data-driven_2019}
Champion, K., Lusch, B., Kutz, J.~N., and Brunton, S.~L.
\newblock Data-driven discovery of coordinates and governing equations.
\newblock \emph{Proc Natl Acad Sci USA}, 116\penalty0 (45):\penalty0
  22445--22451, 2019.
\newblock ISSN 0027-8424, 1091-6490.
\newblock \doi{10.1073/pnas.1906995116}.
\newblock URL \url{http://www.pnas.org/lookup/doi/10.1073/pnas.1906995116}.

\bibitem[Chang et~al.(2019)Chang, Chen, Haber, and
  Chi]{chang_antisymmetricrnn_2019}
Chang, B., Chen, M., Haber, E., and Chi, E.~H.
\newblock {AntisymmetricRNN}: A dynamical system view on recurrent neural
  networks.
\newblock \emph{International Conference on Learning Representations}, 2019.
\newblock URL \url{http://arxiv.org/abs/1902.09689}.

\bibitem[Chen et~al.(2018)Chen, Rubanova, Bettencourt, and
  Duvenaud]{chen_neural_2018}
Chen, R. T.~Q., Rubanova, Y., Bettencourt, J., and Duvenaud, D.
\newblock Neural {Ordinary} {Differential} {Equations}.
\newblock In \emph{Advances in {Neural} {Information} {Processing} {Systems}
  31}, 2018.
\newblock URL \url{http://arxiv.org/abs/1806.07366}.

\bibitem[Chen et~al.(2017)Chen, Shojaie, and Witten]{chen_network_2017}
Chen, S., Shojaie, A., and Witten, D.~M.
\newblock Network {Reconstruction} {From} {High}-{Dimensional} {Ordinary}
  {Differential} {Equations}.
\newblock \emph{Journal of the American Statistical Association}, 112\penalty0
  (520):\penalty0 1697--1707, 2017.
\newblock ISSN 0162-1459.
\newblock \doi{10.1080/01621459.2016.1229197}.
\newblock URL \url{https://doi.org/10.1080/01621459.2016.1229197}.

\bibitem[Chen et~al.(2020)Chen, Zhang, Arjovsky, and
  Bottou]{chen_symplectic_2020}
Chen, Z., Zhang, J., Arjovsky, M., and Bottou, L.
\newblock Symplectic {Recurrent} {Neural} {Networks}.
\newblock In \emph{Proceedings of the 8th {International} {Conference} on
  {Learning} {Representations}}, 2020.
\newblock URL \url{http://arxiv.org/abs/1909.13334}.

\bibitem[Cooley \& Tukey(1965)Cooley and Tukey]{cooley_algorithm_1965}
Cooley, J.~W. and Tukey, J.~W.
\newblock An algorithm for the machine calculation of complex fourier series.
\newblock \emph{Mathematics of Computation}, 19\penalty0 (90):\penalty0
  297--301, 1965.
\newblock ISSN 0025-5718.
\newblock \doi{10.2307/2003354}.
\newblock URL \url{https://www.jstor.org/stable/2003354}.
\newblock Publisher: American Mathematical Society.

\bibitem[Cui et~al.(2016)Cui, Chen, and Chen]{chen_multiscale_2016}
Cui, Z., Chen, W., and Chen, Y.
\newblock Multi-scale convolutional neural networks for time series
  classification.
\newblock \emph{Computing Research Repository}, abs/1603.06995, 2016.
\newblock URL \url{http://arxiv.org/abs/1603.06995}.

\bibitem[de~Silva et~al.(2020)de~Silva, Champion, Quade, Loiseau, Kutz, and
  Brunton]{de_silva_pysindy_2020}
de~Silva, B.~M., Champion, K., Quade, M., Loiseau, J.-C., Kutz, J.~N., and
  Brunton, S.~L.
\newblock {PySINDy}: {A} {Python} package for the {Sparse} {Identification} of
  {Nonlinear} {Dynamics} from {Data}.
\newblock \emph{arXiv preprint arXiv:2004.08424}, 2020.
\newblock URL \url{http://arxiv.org/abs/2004.08424}.

\bibitem[Doya(1992)]{doya_bifurcations_1992}
Doya, K.
\newblock Bifurcations in the learning of recurrent neural networks.
\newblock In \emph{Proceedings of the 1992 {IEEE} {International} {Symposium}
  on {Circuits} and {Systems}}, 1992.
\newblock ISBN 978-0-7803-0593-9.
\newblock \doi{10.1109/ISCAS.1992.230622}.
\newblock URL \url{http://ieeexplore.ieee.org/document/230622/}.

\bibitem[Durstewitz(2009)]{durstewitz_implications_2009}
Durstewitz, D.
\newblock Implications of synaptic biophysics for recurrent network dynamics
  and active memory.
\newblock \emph{Neural Networks}, 22\penalty0 (8):\penalty0 1189--1200, 2009.
\newblock ISSN 08936080.
\newblock \doi{10.1016/j.neunet.2009.07.016}.
\newblock URL
  \url{https://linkinghub.elsevier.com/retrieve/pii/S0893608009001622}.

\bibitem[Durstewitz(2017)]{durstewitz_state_2017}
Durstewitz, D.
\newblock A state space approach for piecewise-linear recurrent neural networks
  for identifying computational dynamics from neural measurements.
\newblock \emph{PLoS Comput. Biol.}, 13\penalty0 (6):\penalty0 e1005542, 2017.
\newblock ISSN 1553-7358.
\newblock \doi{10.1371/journal.pcbi.1005542}.

\bibitem[Erichson et~al.(2021)Erichson, Azencot, Queiruga, Hodgkinson, and
  Mahoney]{erichson_lipschitz_2021}
Erichson, N.~B., Azencot, O., Queiruga, A., Hodgkinson, L., and Mahoney, M.~W.
\newblock Lipschitz recurrent neural networks.
\newblock \emph{International Conference on Learning Representations}, 2021.
\newblock URL \url{http://arxiv.org/abs/2006.12070}.

\bibitem[Feigin(1995)]{feigin_increasingly_1995}
Feigin, M.~I.
\newblock The increasingly complex structure of the bifurcation tree of a
  piecewise-smooth system.
\newblock \emph{Journal of Applied Mathematics and Mechanics}, 59\penalty0
  (6):\penalty0 853--863, 1995.
\newblock ISSN 0021-8928.
\newblock \doi{10.1016/0021-8928(95)00118-2}.
\newblock URL
  \url{https://www.sciencedirect.com/science/article/pii/0021892895001182}.

\bibitem[Fraccaro et~al.(2016)Fraccaro, Sønderby, Paquet, and
  Winther]{fraccaro_sequential_2016}
Fraccaro, M., Sønderby, S.~K., Paquet, U., and Winther, O.
\newblock Sequential neural models with stochastic layers.
\newblock \emph{{arXiv}:1605.07571 [cs, stat]}, 2016.
\newblock URL \url{http://arxiv.org/abs/1605.07571}.

\bibitem[Ghahramani \& Roweis(1998)Ghahramani and
  Roweis]{ghahramani_learning_1998}
Ghahramani, Z. and Roweis, S.~T.
\newblock Learning nonlinear dynamical systems using an {EM} algorithm.
\newblock In \emph{Advances in {Neural} {Information} {Processing} {Systems}
  11}, 1998.

\bibitem[Girin et~al.(2020)Girin, Leglaive, Bie, Diard, Hueber, and
  Alameda-Pineda]{girin_dynamical_2020}
Girin, L., Leglaive, S., Bie, X., Diard, J., Hueber, T., and Alameda-Pineda, X.
\newblock Dynamical {Variational} {Autoencoders}: {A} {Comprehensive} {Review}.
\newblock \emph{arXiv preprint arXiv:2008.12595}, 2020.
\newblock URL \url{http://arxiv.org/abs/2008.12595}.

\bibitem[Greydanus et~al.(2019)Greydanus, Dzamba, and
  Yosinski]{greydanus_hamiltonian_2019}
Greydanus, S., Dzamba, M., and Yosinski, J.
\newblock Hamiltonian {Neural} {Networks}.
\newblock In \emph{Advances in {Neural} {Information} {Processing} {Systems}
  32}, 2019.
\newblock URL \url{http://arxiv.org/abs/1906.01563}.

\bibitem[Hastie et~al.(2009)Hastie, Tibshirani, and
  Friedman]{hastie_elements_2009}
Hastie, T., Tibshirani, R., and Friedman, J.
\newblock \emph{The elements of statistical learning: data mining, inference,
  and prediction}.
\newblock Springer Science \& Business Media, 2009.

\bibitem[Heim et~al.(2019)Heim, Šmídl, and Pevný]{heim_rodent_2019}
Heim, N., Šmídl, V., and Pevný, T.
\newblock Rodent: {Relevance} determination in differential equations.
\newblock \emph{arXiv preprint arXiv:1912.00656}, 2019.
\newblock URL \url{http://arxiv.org/abs/1912.00656}.

\bibitem[Hershey \& Olsen(2007)Hershey and Olsen]{Hershey2007ApproximatingTK}
Hershey, J.~R. and Olsen, P.~A.
\newblock Approximating the kullback leibler divergence between gaussian
  mixture models.
\newblock \emph{2007 IEEE International Conference on Acoustics, Speech and
  Signal Processing - ICASSP '07}, 4:\penalty0 IV--317--IV--320, 2007.

\bibitem[Hochreiter \& Schmidhuber(1997)Hochreiter and
  Schmidhuber]{hochreiter1997long}
Hochreiter, S. and Schmidhuber, J.
\newblock Long short-term memory.
\newblock \emph{Neural computation}, 9\penalty0 (8):\penalty0 1735--1780, 1997.

\bibitem[Hogan et~al.(2007)Hogan, Higham, and Griffin]{hogan_dynamics_2007}
Hogan, S.~J., Higham, L., and Griffin, T. C.~L.
\newblock Dynamics of a piecewise linear map with a gap.
\newblock \emph{Proceedings of the Royal Society A: Mathematical, Physical and
  Engineering Sciences}, 463\penalty0 (2077):\penalty0 49--65, 2007.
\newblock \doi{10.1098/rspa.2006.1735}.
\newblock URL
  \url{https://royalsocietypublishing.org/doi/abs/10.1098/rspa.2006.1735}.

\bibitem[Häusser et~al.(2000)Häusser, Spruston, and
  Stuart]{hausser_diversity_2000}
Häusser, M., Spruston, N., and Stuart, G.~J.
\newblock Diversity and {Dynamics} of {Dendritic} {Signaling}.
\newblock \emph{Science}, 290\penalty0 (5492):\penalty0 739--744, 2000.
\newblock ISSN 0036-8075, 1095-9203.
\newblock \doi{10.1126/science.290.5492.739}.
\newblock URL \url{https://science.sciencemag.org/content/290/5492/739}.

\bibitem[Jaeger \& Haas(2004)Jaeger and Haas]{jaeger_harnessing_2004}
Jaeger, H. and Haas, H.
\newblock Harnessing {Nonlinearity}: {Predicting} {Chaotic} {Systems} and
  {Saving} {Energy} in {Wireless} {Communication}.
\newblock \emph{Science}, 304\penalty0 (5667):\penalty0 78--80, 2004.
\newblock ISSN 0036-8075, 1095-9203.
\newblock \doi{10.1126/science.1091277}.
\newblock URL \url{https://science.sciencemag.org/content/304/5667/78}.

\bibitem[Kag et~al.(2020)Kag, Zhang, and Saligrama]{kag_rnns_2020}
Kag, A., Zhang, Z., and Saligrama, V.
\newblock Rnns incrementally evolving on an equilibrium manifold: A panacea for
  vanishing and exploding gradients?
\newblock \emph{International Conference on Learning Representations}, pp.\
  ~24, 2020.

\bibitem[Kantz \& Schreiber(2004)Kantz and Schreiber]{kantz_nonlinear_2004}
Kantz, H. and Schreiber, T.
\newblock \emph{Nonlinear time series analysis}, volume~7.
\newblock Cambridge university press, 2004.

\bibitem[Karl et~al.(2017)Karl, Soelch, Bayer, and van~der
  Smagt]{karl_deep_2017}
Karl, M., Soelch, M., Bayer, J., and van~der Smagt, P.
\newblock Deep {Variational} {Bayes} {Filters}: {Unsupervised} {Learning} of
  {State} {Space} {Models} from {Raw} {Data}.
\newblock In \emph{Proceedings of the 5th {International} {Conference} on
  {Learning} {Representations}}, 2017.
\newblock URL \url{http://arxiv.org/abs/1605.06432}.

\bibitem[Kerg et~al.(2019)Kerg, Goyette, Touzel, Gidel, Vorontsov, Bengio, and
  Lajoie]{kerg_non-normal_2019}
Kerg, G., Goyette, K., Touzel, M.~P., Gidel, G., Vorontsov, E., Bengio, Y., and
  Lajoie, G.
\newblock Non-normal recurrent neural network ({nnRNN}): learning long time
  dependencies while improving expressivity with transient dynamics.
\newblock \emph{Proceedings of the 33rd International Conference on Neural
  Information Processing Systems}, pp.\ ~11, 2019.

\bibitem[Kingma \& Ba(2015)Kingma and Ba]{kingma_adam_2015}
Kingma, D.~P. and Ba, J.
\newblock Adam: {A} {Method} for {Stochastic} {Optimization}.
\newblock In \emph{Proceedings of the 3rd {International} {Conference} on
  {Learning} {Representations}}, 2015.
\newblock URL \url{http://arxiv.org/abs/1412.6980}.

\bibitem[Kingma \& Welling(2014)Kingma and Welling]{kingma_auto-encoding_2014}
Kingma, D.~P. and Welling, M.
\newblock Auto-{Encoding} {Variational} {Bayes}.
\newblock In \emph{Proceedings of the 2nd {International} {Conference} on
  {Learning} {Representations}}, 2014.
\newblock URL \url{http://arxiv.org/abs/1312.6114}.

\bibitem[Koch(2004)]{koch_biophysics_2004}
Koch, C.
\newblock \emph{Biophysics of computation: information processing in single
  neurons}.
\newblock Oxford university press, 2004.

\bibitem[Koppe et~al.(2019)Koppe, Toutounji, Kirsch, Lis, and
  Durstewitz]{koppe_identifying_2019}
Koppe, G., Toutounji, H., Kirsch, P., Lis, S., and Durstewitz, D.
\newblock Identifying nonlinear dynamical systems via generative recurrent
  neural networks with applications to {fMRI}.
\newblock \emph{PLOS Computational Biology}, 15\penalty0 (8):\penalty0
  e1007263, 2019.
\newblock ISSN 1553-7358.
\newblock \doi{10.1371/journal.pcbi.1007263}.
\newblock URL
  \url{https://journals.plos.org/ploscompbiol/article?id=10.1371/journal.pcbi.1007263}.

\bibitem[Krishnan et~al.(2017)Krishnan, Shalit, and
  Sontag]{krishnan_structured_2017}
Krishnan, R.~G., Shalit, U., and Sontag, D.
\newblock Structured {Inference} {Networks} for {Nonlinear} {State} {Space}
  {Models}.
\newblock In \emph{31st {AAAI} {Conference} on {Artificial} {Intelligence}},
  2017.
\newblock URL \url{http://arxiv.org/abs/1609.09869}.

\bibitem[Landau \& Sompolinsky(2018)Landau and
  Sompolinsky]{landau_coherent_2018}
Landau, I.~D. and Sompolinsky, H.
\newblock Coherent chaos in a recurrent neural network with structured
  connectivity.
\newblock \emph{PLoS Comput Biol}, 14\penalty0 (12):\penalty0 e1006309, 2018.
\newblock ISSN 1553-7358.
\newblock \doi{10.1371/journal.pcbi.1006309}.
\newblock URL \url{https://dx.plos.org/10.1371/journal.pcbi.1006309}.

\bibitem[Lorenz(1963)]{lorenz_deterministic_1963}
Lorenz, E.~N.
\newblock Deterministic nonperiodic flow.
\newblock \emph{Journal of atmospheric sciences}, 20\penalty0 (2):\penalty0
  130--141, 1963.

\bibitem[Lorenz(1996)]{lorenz_predictability_1996}
Lorenz, E.~N.
\newblock Predictability: {A} problem partly solved.
\newblock In \emph{Proc. {Seminar} on predictability}, volume~1, 1996.

\bibitem[Maheswaranathan et~al.(2019{\natexlab{a}})Maheswaranathan, Williams,
  Golub, Ganguli, and Sussillo]{maheswaranathan_reverse_2019}
Maheswaranathan, N., Williams, A.~H., Golub, M.~D., Ganguli, S., and Sussillo,
  D.
\newblock Reverse engineering recurrent networks for sentiment classification
  reveals line attractor dynamics.
\newblock In \emph{Advances in neural information processing systems 32},
  2019{\natexlab{a}}.
\newblock URL \url{https://www.ncbi.nlm.nih.gov/pmc/articles/PMC7416638/}.

\bibitem[Maheswaranathan et~al.(2019{\natexlab{b}})Maheswaranathan, Williams,
  Golub, Ganguli, and Sussillo]{maheswaranathan_universality_2019}
Maheswaranathan, N., Williams, A.~H., Golub, M.~D., Ganguli, S., and Sussillo,
  D.
\newblock Universality and individuality in neural dynamics across large
  populations of recurrent networks.
\newblock In \emph{Advances in {Neural} {Information} {Processing} {Systems}
  32}, 2019{\natexlab{b}}.
\newblock URL \url{https://www.ncbi.nlm.nih.gov/pmc/articles/PMC7416639/}.

\bibitem[Mel(1993)]{mel_synaptic_1993}
Mel, B.~W.
\newblock Synaptic integration in an excitable dendritic tree.
\newblock \emph{Journal of Neurophysiology}, 70\penalty0 (3):\penalty0
  1086--1101, 1993.
\newblock ISSN 0022-3077, 1522-1598.
\newblock \doi{10.1152/jn.1993.70.3.1086}.
\newblock URL \url{https://www.physiology.org/doi/10.1152/jn.1993.70.3.1086}.

\bibitem[Mel(1994)]{mel_information_1994}
Mel, B.~W.
\newblock Information {Processing} in {Dendritic} {Trees}.
\newblock \emph{Neural Computation}, 6\penalty0 (6):\penalty0 1031--1085, 1994.
\newblock ISSN 0899-7667.
\newblock \doi{10.1162/neco.1994.6.6.1031}.
\newblock URL \url{https://doi.org/10.1162/neco.1994.6.6.1031}.

\bibitem[Mel(1999)]{mel_why_1999}
Mel, B.~W.
\newblock Why {Have} {Dendrites}? {A} {Computational} {Perspective}.
\newblock In \emph{Dendrites}. Oxford University Press, 1999.
\newblock ISBN 978-0-19-172420-6.
\newblock URL
  \url{https://oxford.universitypressscholarship.com/view/10.1093/acprof:oso/9780198566564.001.0001/acprof-9780198566564-chapter-016}.

\bibitem[Mohajerin \& Waslander(2018)Mohajerin and
  Waslander]{mohajerin_multi-step_2018}
Mohajerin, N. and Waslander, S.~L.
\newblock Multi-step prediction of dynamic systems with recurrent neural
  networks.
\newblock \emph{{arXiv}:1806.00526 [cs]}, 2018.
\newblock URL \url{http://arxiv.org/abs/1806.00526}.

\bibitem[Monfared \& Durstewitz(2020{\natexlab{a}})Monfared and
  Durstewitz]{monfared_existence_2020}
Monfared, Z. and Durstewitz, D.
\newblock Existence of n-cycles and border-collision bifurcations in
  piecewise-linear continuous maps with applications to recurrent neural
  networks.
\newblock \emph{Nonlinear Dyn}, 101\penalty0 (2):\penalty0 1037--1052,
  2020{\natexlab{a}}.
\newblock ISSN 1573-269X.
\newblock \doi{10.1007/s11071-020-05841-x}.
\newblock URL \url{https://doi.org/10.1007/s11071-020-05841-x}.

\bibitem[Monfared \& Durstewitz(2020{\natexlab{b}})Monfared and
  Durstewitz]{monfared_transformation_2020}
Monfared, Z. and Durstewitz, D.
\newblock Transformation of {ReLU}-based recurrent neural networks from
  discrete-time to continuous-time.
\newblock In \emph{Proceedings of the 37th {International} {Conference} on
  {Machine} {Learning}}, 2020{\natexlab{b}}.
\newblock URL \url{http://proceedings.mlr.press/v119/monfared20a.html}.

\bibitem[Monfared et~al.(2021)Monfared, Mikhaeil, and
  Durstewitz]{monfared_2022}
Monfared, Z., Mikhaeil, J.~M., and Durstewitz, D.
\newblock How to train rnns on chaotic data?, 2021.
\newblock URL \url{https://arxiv.org/abs/2110.07238}.

\bibitem[Norcliffe et~al.(2021)Norcliffe, Bodnar, Day, Moss, and
  Liò]{norcliffe_neural_2021}
Norcliffe, A., Bodnar, C., Day, B., Moss, J., and Liò, P.
\newblock Neural {ODE} {Processes}.
\newblock In \emph{Proceedings of the 9th {International} {Conference} on
  {Learning} {Representations}}, 2021.
\newblock URL \url{https://openreview.net/forum?id=27acGyyI1BY}.

\bibitem[Pandarinath et~al.(2018)Pandarinath, O’Shea, Collins, Jozefowicz,
  Stavisky, Kao, Trautmann, Kaufman, Ryu, Hochberg, Henderson, Shenoy, Abbott,
  and Sussillo]{pandarinath_inferring_2018}
Pandarinath, C., O’Shea, D.~J., Collins, J., Jozefowicz, R., Stavisky, S.~D.,
  Kao, J.~C., Trautmann, E.~M., Kaufman, M.~T., Ryu, S.~I., Hochberg, L.~R.,
  Henderson, J.~M., Shenoy, K.~V., Abbott, L.~F., and Sussillo, D.
\newblock Inferring single-trial neural population dynamics using sequential
  auto-encoders.
\newblock \emph{Nature Methods}, 15\penalty0 (10):\penalty0 805--815, 2018.
\newblock ISSN 1548-7105.
\newblock \doi{10.1038/s41592-018-0109-9}.
\newblock URL \url{https://www.nature.com/articles/s41592-018-0109-9}.

\bibitem[Pascanu et~al.(2013)Pascanu, Mikolov, and
  Bengio]{pascanu_difficulty_2013}
Pascanu, R., Mikolov, T., and Bengio, Y.
\newblock On the difficulty of training recurrent neural networks.
\newblock In \emph{Proceedings of the 30th {International} {Conference} on
  {Machine} {Learning}}, 2013.
\newblock URL \url{http://proceedings.mlr.press/v28/pascanu13.html}.

\bibitem[Pathak et~al.(2018)Pathak, Hunt, Girvan, Lu, and
  Ott]{pathak_model-free_2018}
Pathak, J., Hunt, B., Girvan, M., Lu, Z., and Ott, E.
\newblock Model-{Free} {Prediction} of {Large} {Spatiotemporally} {Chaotic}
  {Systems} from {Data}: {A} {Reservoir} {Computing} {Approach}.
\newblock \emph{Phys. Rev. Lett.}, 120\penalty0 (2):\penalty0 024102, 2018.
\newblock ISSN 0031-9007, 1079-7114.
\newblock \doi{10.1103/PhysRevLett.120.024102}.
\newblock URL \url{https://link.aps.org/doi/10.1103/PhysRevLett.120.024102}.

\bibitem[Patra(2018)]{patra_multiple_2018}
Patra, M.
\newblock Multiple {Attractor} {Bifurcation} in {Three}-{Dimensional}
  {Piecewise} {Linear} {Maps}.
\newblock \emph{Int. J. Bifurcation Chaos}, 28\penalty0 (10):\penalty0 1830032,
  2018.
\newblock ISSN 0218-1274.
\newblock \doi{10.1142/S021812741830032X}.
\newblock URL
  \url{https://www.worldscientific.com/doi/abs/10.1142/S021812741830032X}.

\bibitem[Pearlmutter(1990)]{pearlmutter_1990}
Pearlmutter, B.
\newblock Dynamic recurrent neural networks, 1990.
\newblock URL
  \url{https://kilthub.cmu.edu/articles/journal_contribution/Dynamic_recurrent_neural_networks/6605018/1}.

\bibitem[Poirazi \& Papoutsi(2020)Poirazi and
  Papoutsi]{poirazi_illuminating_2020}
Poirazi, P. and Papoutsi, A.
\newblock Illuminating dendritic function with computational models.
\newblock \emph{Nat Rev Neurosci}, 21\penalty0 (6):\penalty0 303--321, 2020.
\newblock ISSN 1471-003X, 1471-0048.
\newblock \doi{10.1038/s41583-020-0301-7}.
\newblock URL \url{http://www.nature.com/articles/s41583-020-0301-7}.

\bibitem[Poirazi et~al.(2003)Poirazi, Brannon, and Mel]{poirazi_pyramidal_2003}
Poirazi, P., Brannon, T., and Mel, B.~W.
\newblock Pyramidal neuron as two-layer neural network.
\newblock \emph{Neuron}, 37\penalty0 (6):\penalty0 989--999, 2003.

\bibitem[Raissi(2018)]{raissi_deep_2018}
Raissi, M.
\newblock Deep {Hidden} {Physics} {Models}: {Deep} {Learning} of {Nonlinear}
  {Partial} {Differential} {Equations}.
\newblock \emph{Journal of Machine Learning Research}, 19\penalty0
  (25):\penalty0 1--24, 2018.
\newblock URL \url{http://jmlr.org/papers/v19/18-046.html}.

\bibitem[Raissi et~al.(2018)Raissi, Perdikaris, and
  Karniadakis]{raissi_multistep_2018}
Raissi, M., Perdikaris, P., and Karniadakis, G.~E.
\newblock Multistep neural networks for data-driven discovery of nonlinear
  dynamical systems.
\newblock \emph{{arXiv}:1801.01236 [nlin, physics:physics, stat]}, 2018.
\newblock URL \url{http://arxiv.org/abs/1801.01236}.

\bibitem[Reiss et~al.(2019)Reiss, Indlekofer, Schmidt, and
  Van~Laerhoven]{Reiss_2019}
Reiss, A., Indlekofer, I., Schmidt, P., and Van~Laerhoven, K.
\newblock Deep ppg: Large-scale heart rate estimation with convolutional neural
  networks.
\newblock \emph{Sensors}, 19\penalty0 (14), 2019.
\newblock ISSN 1424-8220.
\newblock \doi{10.3390/s19143079}.
\newblock URL \url{https://www.mdpi.com/1424-8220/19/14/3079}.

\bibitem[Rezende \& Mohamed(2015)Rezende and Mohamed]{rezende_variational_2015}
Rezende, D. and Mohamed, S.
\newblock Variational {Inference} with {Normalizing} {Flows}.
\newblock In \emph{Proceedings of the 32nd {International} {Conference} on
  {Machine} {Learning}}, 2015.
\newblock URL \url{http://proceedings.mlr.press/v37/rezende15.html}.

\bibitem[Rezende et~al.(2014)Rezende, Mohamed, and
  Wierstra]{rezende_stochastic_2014}
Rezende, D.~J., Mohamed, S., and Wierstra, D.
\newblock Stochastic {Backpropagation} and {Approximate} {Inference} in {Deep}
  {Generative} {Models}.
\newblock In \emph{Proceedings of the 31st {International} {Conference} on
  {Machine} {Learning}}, 2014.
\newblock URL \url{http://arxiv.org/abs/1401.4082}.

\bibitem[Rudy et~al.(2017)Rudy, Brunton, Proctor, and
  Kutz]{rudy_data-driven_2017}
Rudy, S.~H., Brunton, S.~L., Proctor, J.~L., and Kutz, J.~N.
\newblock Data-driven discovery of partial differential equations.
\newblock \emph{Science Advances}, 3\penalty0 (4):\penalty0 e1602614, 2017.
\newblock ISSN 2375-2548.
\newblock \doi{10.1126/sciadv.1602614}.
\newblock URL \url{https://advances.sciencemag.org/content/3/4/e1602614}.

\bibitem[Rumelhart et~al.(1986)Rumelhart, Hinton, and
  Williams]{Rumelhart_1986_learning}
Rumelhart, D.~E., Hinton, G.~E., and Williams, R.~J.
\newblock \emph{Learning Internal Representations by Error Propagation},
  volume~1, pp.\  318--362.
\newblock Bradford Books, Cambridge MA, 1986.

\bibitem[Rusch \& Mishra(2021)Rusch and Mishra]{rusch_coupled_2021}
Rusch, T.~K. and Mishra, S.
\newblock Coupled oscillatory recurrent neural network ({coRNN}): An accurate
  and (gradient) stable architecture for learning long time dependencies.
\newblock In \emph{International Conference on Learning Representations}, 2021.
\newblock URL \url{https://openreview.net/forum?id=F3s69XzWOia}.

\bibitem[Rusch et~al.(2022)Rusch, Mishra, Erichson, and Mahoney]{rusch2022long}
Rusch, T.~K., Mishra, S., Erichson, N.~B., and Mahoney, M.~W.
\newblock Long expressive memory for sequence modeling.
\newblock In \emph{International Conference on Learning Representations}, 2022.
\newblock URL \url{https://openreview.net/forum?id=vwj6aUeocyf}.

\bibitem[Sauer et~al.(1991)Sauer, Yorke, and Casdagli]{sauer_embedology_1991}
Sauer, T., Yorke, J.~A., and Casdagli, M.
\newblock Embedology.
\newblock \emph{Journal of statistical Physics}, 65\penalty0 (3):\penalty0
  579--616, 1991.

\bibitem[Saxe et~al.(2014)Saxe, McClelland, and Ganguli]{saxe_exact_2014}
Saxe, A.~M., McClelland, J.~L., and Ganguli, S.
\newblock Exact solutions to the nonlinear dynamics of learning in deep linear
  neural networks.
\newblock In \emph{Proceedings of the 2nd {International} {Conference} on
  {Learning} {Representations}}, 2014.
\newblock URL \url{http://arxiv.org/abs/1312.6120}.

\bibitem[Schalk et~al.(2000)Schalk, {McFarland}, Hinterberger, Birbaumer, and
  Wolpaw]{schalk_bci2000_2004}
Schalk, G., {McFarland}, D.~J., Hinterberger, T., Birbaumer, N., and Wolpaw,
  J.~R.
\newblock {BCI}2000: a general-purpose brain-computer interface ({BCI}) system.
\newblock \emph{{IEEE} transactions on bio-medical engineering}, 51\penalty0
  (6):\penalty0 1034--1043, 2000.
\newblock ISSN 0018-9294.
\newblock \doi{10.1109/TBME.2004.827072}.

\bibitem[Schiller et~al.(2000)Schiller, Major, Koester, and
  Schiller]{schiller_nmda_2000}
Schiller, J., Major, G., Koester, H.~J., and Schiller, Y.
\newblock {NMDA} spikes in basal dendrites of cortical pyramidal neurons.
\newblock \emph{Nature}, 404\penalty0 (6775):\penalty0 285--289, 2000.
\newblock ISSN 1476-4687.
\newblock \doi{10.1038/35005094}.
\newblock URL \url{https://www.nature.com/articles/35005094}.

\bibitem[Schmidt et~al.(2021)Schmidt, Koppe, Monfared, Beutelspacher, and
  Durstewitz]{schmidt_identifying_2021}
Schmidt, D., Koppe, G., Monfared, Z., Beutelspacher, M., and Durstewitz, D.
\newblock Identifying nonlinear dynamical systems with multiple time scales and
  long-range dependencies.
\newblock In \emph{Proceedings of the 9th {International} {Conference} on
  {Learning} {Representations}}, 2021.
\newblock URL \url{http://arxiv.org/abs/1910.03471}.

\bibitem[Shalova \& Oseledets(2020{\natexlab{a}})Shalova and
  Oseledets]{shalova_deep_2020}
Shalova, A. and Oseledets, I.
\newblock Deep {Representation} {Learning} for {Dynamical} {Systems}
  {Modeling}.
\newblock \emph{arXiv preprint arXiv:2002.05111}, 2020{\natexlab{a}}.

\bibitem[Shalova \& Oseledets(2020{\natexlab{b}})Shalova and
  Oseledets]{shalova_tensorized_2020}
Shalova, A. and Oseledets, I.
\newblock Tensorized {Transformer} for {Dynamical} {Systems} {Modeling}.
\newblock \emph{arXiv preprint arXiv:2006.03445}, 2020{\natexlab{b}}.

\bibitem[Stemmler \& Koch(1999)Stemmler and Koch]{stemmler_how_1999}
Stemmler, M. and Koch, C.
\newblock How voltage-dependent conductances can adapt to maximize the
  information encoded by neuronal firing rate.
\newblock \emph{Nature Neuroscience}, 2\penalty0 (6):\penalty0 521--527, 1999.
\newblock ISSN 1546-1726.
\newblock \doi{10.1038/9173}.
\newblock URL \url{https://www.nature.com/articles/nn0699_521}.

\bibitem[Storace \& De~Feo(2004)Storace and
  De~Feo]{storace_piecewise-linear_2004}
Storace, M. and De~Feo, O.
\newblock Piecewise-linear approximation of nonlinear dynamical systems.
\newblock \emph{IEEE Transactions on Circuits and Systems I: Regular Papers},
  51\penalty0 (4):\penalty0 830--842, April 2004.
\newblock ISSN 1558-0806.
\newblock \doi{10.1109/TCSI.2004.823664}.
\newblock Conference Name: IEEE Transactions on Circuits and Systems I: Regular
  Papers.

\bibitem[Strauss(2020)]{strauss_augmenting_2020}
Strauss, R.
\newblock Augmenting neural differential equations to model unknown dynamical
  systems with incomplete state information.
\newblock \emph{{arXiv}:2008.08226 [physics, q-bio]}, 2020.
\newblock URL \url{http://arxiv.org/abs/2008.08226}.

\bibitem[Takens(1981)]{takens_detecting_1981}
Takens, F.
\newblock Detecting strange attractors in turbulence.
\newblock In \emph{Dynamical {Systems} and {Turbulence}, {Warwick} 1980},
  volume 898, pp.\  366--381. Springer, 1981.
\newblock ISBN 978-3-540-11171-9 978-3-540-38945-3.
\newblock URL \url{http://link.springer.com/10.1007/BFb0091924}.

\bibitem[Talathi \& Vartak(2016)Talathi and Vartak]{talathi_improving_2016}
Talathi, S.~S. and Vartak, A.
\newblock Improving performance of recurrent neural network with relu
  nonlinearity.
\newblock In \emph{Proceedings of the 4th {International} {Conference} on
  {Learning} {Representations}}, 2016.
\newblock URL \url{http://arxiv.org/abs/1511.03771}.

\bibitem[Trischler \& D’Eleuterio(2016)Trischler and
  D’Eleuterio]{trischler_synthesis_2016}
Trischler, A.~P. and D’Eleuterio, G.~M.
\newblock Synthesis of recurrent neural networks for dynamical system
  simulation.
\newblock \emph{Neural Networks}, 80:\penalty0 67--78, 2016.
\newblock ISSN 08936080.
\newblock \doi{10.1016/j.neunet.2016.04.001}.
\newblock URL
  \url{https://linkinghub.elsevier.com/retrieve/pii/S0893608016300314}.

\bibitem[Vlachas et~al.(2018)Vlachas, Byeon, Wan, Sapsis, and
  Koumoutsakos]{vlachas_data-driven_2018}
Vlachas, P.~R., Byeon, W., Wan, Z.~Y., Sapsis, T.~P., and Koumoutsakos, P.
\newblock Data-driven forecasting of high-dimensional chaotic systems with long
  short-term memory networks.
\newblock \emph{Proc. R. Soc. A.}, 474\penalty0 (2213):\penalty0 20170844,
  2018.
\newblock ISSN 1364-5021, 1471-2946.
\newblock \doi{10.1098/rspa.2017.0844}.
\newblock URL
  \url{https://royalsocietypublishing.org/doi/10.1098/rspa.2017.0844}.

\bibitem[Wahba(1990)]{wahba_spline_1990}
Wahba, G.
\newblock \emph{Spline models for observational data}.
\newblock SIAM, 1990.

\bibitem[Williams \& Zipser(1989)Williams and Zipser]{williams_learning_1989}
Williams, R.~J. and Zipser, D.
\newblock A learning algorithm for continually running fully recurrent neural
  networks.
\newblock \emph{Neural Computation}, 1\penalty0 (2):\penalty0 270--280, June
  1989.
\newblock ISSN 0899-7667, 1530-888X.
\newblock \doi{10.1162/neco.1989.1.2.270}.
\newblock URL \url{https://direct.mit.edu/neco/article/1/2/270-280/5490}.

\bibitem[Wilson \& Cowan(1972)Wilson and Cowan]{wilson_excitatory_1972}
Wilson, H.~R. and Cowan, J.~D.
\newblock Excitatory and inhibitory interactions in localized populations of
  model neurons.
\newblock \emph{Biophysical Journal}, 12\penalty0 (1):\penalty0 1--24, 1972.
\newblock ISSN 0006-3495.
\newblock \doi{10.1016/S0006-3495(72)86068-5}.

\bibitem[Yeung et~al.(2017)Yeung, Kundu, and Hodas]{yeung_learning_2017}
Yeung, E., Kundu, S., and Hodas, N.
\newblock Learning {Deep} {Neural} {Network} {Representations} for {Koopman}
  {Operators} of {Nonlinear} {Dynamical} {Systems}.
\newblock \emph{arXiv preprint arXiv:1708.06850}, 2017.
\newblock \doi{10.23919/ACC.2019.8815339}.

\bibitem[Yin et~al.(2021)Yin, Guen, Dona, Bezenac, Ayed, Thome, and
  Gallinari]{yin_augmenting_2021}
Yin, Y., Guen, V.~L., Dona, J., Bezenac, E.~d., Ayed, I., Thome, N., and
  Gallinari, P.
\newblock Augmenting {Physical} {Models} with {Deep} {Networks} for {Complex}
  {Dynamics} {Forecasting}.
\newblock In \emph{Proceedings of the 9th {International} {Conference} on
  {Learning} {Representations}}, 2021.
\newblock URL \url{https://openreview.net/forum?id=kmG8vRXTFv}.

\bibitem[Zheng et~al.(2017)Zheng, Zaheer, Ahmed, Wang, Xing, and
  Smola]{zheng_state_2017}
Zheng, X., Zaheer, M., Ahmed, A., Wang, Y., Xing, E.~P., and Smola, A.~J.
\newblock State {Space} {LSTM} {Models} with {Particle} {MCMC} {Inference}.
\newblock \emph{arXiv preprint arXiv:1711.11179}, 2017.
\newblock URL \url{https://arxiv.org/abs/1711.11179}.

\bibitem[Zhu et~al.(2021)Zhu, Guo, and Lin]{zhu_neural_2021}
Zhu, Q., Guo, Y., and Lin, W.
\newblock Neural {Delay} {Differential} {Equations}.
\newblock In \emph{Proceedings of the 9th {International} {Conference} on
  {Learning} {Representations}}, 2021.
\newblock URL \url{https://openreview.net/forum?id=Q1jmmQz72M2}.

\end{thebibliography}
\bibliographystyle{stylefiles/icml2022}



\section*{}
\newpage

\newpage
\onecolumn
\beginsupplement
\section{Appendix}

\subsection{Further Methodological Details}

    \paragraph{Manifold Attractor Regularization}\label{sec:supp:reg}
	As proposed in \citet{schmidt_identifying_2021}, to encourage the discovery of long-term dependencies and slow time scales in the data, a subset of $M_\mathrm{reg} \leq M$ states was regularized by adding the following term to the ELBO for the VI approach:
	\begin{equation} \label{eq:supp:MAR}
		\mathcal{L}_{\mathrm{reg}}=\lambda \left(
			\sum_{i=1}^{M_{\mathrm{reg}}} (A_{ii}-1)^2+\sum_{i=1}^{M_{\mathrm{reg}}}\sum_{j\neq i}^{M} (W_{ij})^2+\sum_{i=1}^{M_{\mathrm{reg}}} h_{i}^2
		\right).
	\end{equation}
	This regularization pushes the regularized subset of states toward a continuous set of marginally stable fixed points that tends to form an attracting manifold in the full state space, which supports the learning of systems with widely differing time scales, such as the bursting neuron model (cf. Sec. \ref{sec:experiments}). 
	During training with VI, the ELBO was divided by the number of time steps $T$ of a given batch to put it on equal grounds with the regularization term.
	Regularization settings used are summarized in Table \ref{tab:reg} along other hyper-parameter settings.
	
	\begin{figure}[!htb]
    \centering
	\includegraphics[width=0.99\linewidth]{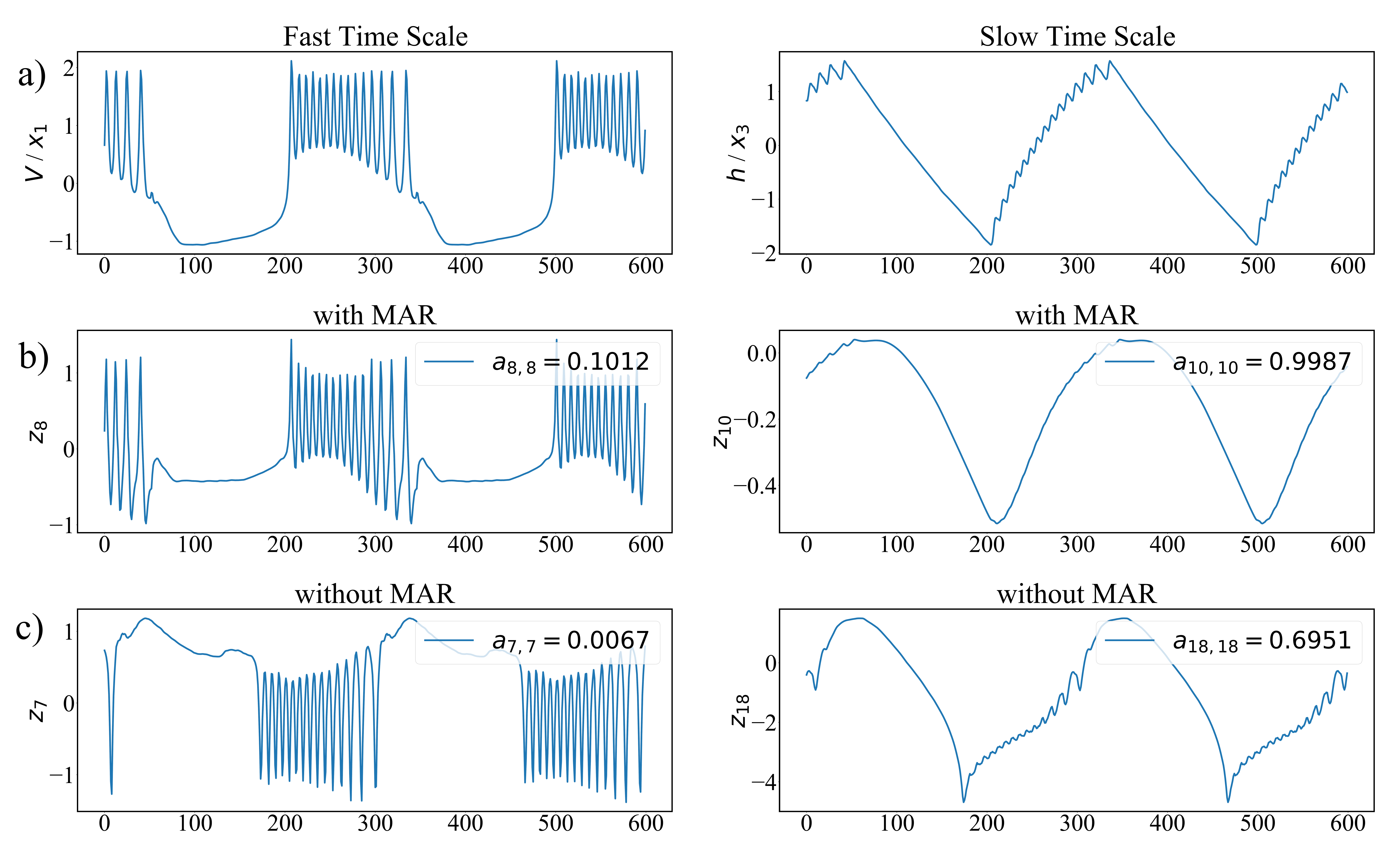}
	\caption{Different latent states capture different time scales of the reconstructed DS. (a) Simulated time series of the first ($V/ x_1$) and third ($h/ x_3$) system variables from a reconstruction (\ourmethodname\ trained with BPTT, $M=20, B=20, \tau=5, M_{\textrm{reg}}/M=0.5, \lambda=0.05$) of the bursting neuron model (Eq. \ref{eq:bursting_neuron}). (b) Time series of the two latent states with the lowest ($a_{8,8} \approx 0.1012$, left) and highest ($a_{10,10} \approx 0.9987$, right) time constants of a \ourmethodname\ trained with manifold attractor regularization (MAR), capturing the fast spiking and slow oscillatory time scales of the bursting neuron, respectively. (c) Same as (b) for a \ourmethodname\ trained without manifold attractor regularization. In this case the separation of time scales in the latent states is often less clear, although in this particular example the lowest ($a_{7,7} \approx 0.0067$, left) and highest ($a_{18,18} \approx 0.6951$, right) time constants still tend to capture the bursting neuron's fastest and slowest time scales, respectively.
	Similar observations were also made for other systems like the ECG and EEG data (not shown).
	}
	\label{fig:bn_time_scales}
    \end{figure}

\paragraph{BPTT-TF}\label{supp-TF-BPTT}
To train a deterministic version of the \ourmethodname, we employ BPTT with a scheduled version of TF \citep{williams_learning_1989,pearlmutter_1990}.
To do so, we choose an \say{identity-mapping} for the observation model $\hat\vx_t = \mathcal{I}\vz_t $,
where $\mathcal{I} \in \mathbb{R}^{N \times M}$ with $\mathcal{I}_{kk}=1$ if $k\leq N$ and zeroes everywhere else. This allows us to regularly replace latent states with observations to ``recalibrate'' the model and break trajectory divergence in case of chaotic dynamics.
Consider a time series $\{\vx_1,\vx_2,\cdots,\vx_T\}$ generated by a DS we want to reconstruct. At times $l\tau+1,\ l \in \mathbb{N}_0$, where $\tau \geq 1$ is the forcing interval, we replace the first $N$ latent states by observations $\hat z_{k,l\tau+1} = x_{k,l\tau+1},\  k\leq N $. The remaining latent states, $\hat z_{k,l\tau+1} = z_{k,l\tau+1}, \ k> N$, remain unaffected by the forcing. This means that we optimize the \ourmethodname\ such that a subspace of the latent space directly maps to the observed time series variables. The forcing interval $\tau$ is a hyperparameter, with optimal settings varying depending on the dataset. The settings we chose are summarized in Table \ref{tab:reg}.
With $\mathcal{F}=\{l\tau+1\}_{l \in \mathbb{N}_0}$, the \ourmethodname\ updates can then be written as
\begin{align}\label{eq:forcing}
  \vz_{t+1} =
  \begin{cases}
       \ourmethodname(\tilde\vz_{t}) & \text{if $t\in \mathcal{F}$} \\
       \ourmethodname(\vz_{t}) & \text{else} 
  \end{cases}.
\end{align}
The loss is calculated prior to the forcing, such that $\mathcal{L}_{t}= \lVert \vx_t - \mathcal{I}\vz_t\rVert_2^2$ for every time step. To improve performance, in some experiments we employed a mean-centered \ourmethodname\ (for details see next paragraph). In the evaluation phase, the trained \ourmethodname\  is simulated freely without any forcing. As the model is deterministic, the initial condition $z_1 = [\vx_{1}, \mL \vx_{1} ]\tran 
$ is estimated from the first data point $\vx_1$ with a matrix $\mL \in \mathbb{R}^{(M-N)\times N}$ which is jointly learned with the other model parameters.

\paragraph{Mean-Centered \ourmethodname}
Layer normalization has recently been developed as a way of significantly improving RNN training \citep{ba_layer_2016}. Here we adapt the idea of layer normalization to the piecewise-linear nature of our \ourmethodname. Instead of fully standardizing the latent states at every time step before applying the activation function, we only mean-center them:
\begin{align}
    \vz_t = \mA \vz_{t-1}+ \mW \phi\big(\mathcal{M}(\vz_{t-1})\big)+\vh_0,
\end{align}
where $\phi(\cdot)$ is given in \eqref{eq:basis_expansion} and $\mathcal{M}(\vz_{t-1}) = \vz_{t-1} - \mathbf{\mu}_{t-1} =\vz_{t-1} -\displaystyle \1 \frac{1}{M} \sum_{j=1}^M z_{j,t-1}$, where $\displaystyle \1 \in \mathbb{R}^M$ is a vector of ones. Note that this mean-centering is linear and can be rewritten as a matrix-multiplication
\begin{align}\nonumber
    \mathcal{M}(\vz_{t-1}) &= \vz_{t-1} - \bold{\mu}_{t-1} \\
    &=\frac{1}{M} \begin{pmatrix}
M-1 & -1 & \cdots & -1\\
-1 & M-1 & \cdots & -1 \\
\vdots &\vdots& \ddots  &\vdots \\
-1 & -1 & \cdots &  M-1
\end{pmatrix} \vz_{t-1} = \mM \vz_{t-1}.
\end{align}
As Remark \ref{remark-FPs-cylces-mcPLRNN} points out, all results about the tractability of the \ourmethodname\ also hold for the mean-centred \ourmethodname.

	\paragraph{State Clipping}\label{supp-state-clipping}

Since the ReLU function used in the \ourmethodname \ is non-saturating, states may diverge to infinity.
As Theorem \ref{pro-J-1} guarantees, there is a simple and natural way to construct a ``clipped'' \ourmethodname
\begin{align}\label{eq-clipped}
\vz_{t} \, =  \, \mA \vz_{t-1} + \mW \sum_{b=1}^B\alpha_b \big[ \max(0,\vz_{t-1}-\vh_b)-\max(0,\vz_{t-1}) \big]+\vh_0.
\end{align}
~\\
  Note that the results of Theorem \ref{pro-J-1} also hold true when the manifold attractor regularization is applied. This is detailed in Proposition \ref{pro-Z-3} further below.
\paragraph{Approximate Posterior for Variational Inference}\label{supp-approxposterior}

To estimate the true unknown posterior $p(\bm{z}|\bm{x})$, we make a Gaussian assumption for the approximate posterior  
$q_{\phi}(\bm{z}|\bm{x})=\mathcal{N}(\boldsymbol{\mu}_\phi(\bm{x}),\boldsymbol{\Sigma}_\phi(\bm{x}))$, where mean and covariance are functions of the observations. Without any simplifying assumptions, the number of parameters in $\boldsymbol{\Sigma}_\phi(\bm{x}) \in \mathbb{R}^{MT\times MT}$ would scale unacceptably with time series length $T$. We therefore made a mean field assumption and factorized $q_{\phi}(\bm{z}|\bm{x})$ across time. Specifically, a time-dependent mean $\boldsymbol{\mu}_{t,\phi}$ and covariance $\boldsymbol{\Sigma}_{t,\phi}$ were parameterized through stacked convolutional networks which take the observations $\{\vx_{t-w}...\vx_{t+w}\}$ as inputs, with $w$ given by the largest kernel size. The mean is given by a 4-layer CNN with decreasing kernel sizes ($41, 31, 21$ and $11$, respectively), with the last layer of the CNN feeding into the parameters of the approximate posterior. For the diagonal covariance, the observations are mapped directly onto the logarithms of the covariance through a single convolutional layer (with a kernel size of $41$) mapping onto the parameters of the approximate posterior.

The classical motivation behind using CNNs rests on the assumption that the data contains translationally invariant patterns, and that this allows the recognition model to embed potentially meaningful temporal context into the latent representation (see e.g. \citet{chen_multiscale_2016}). We note that while the mean-field approximation is computationally highly efficient, it makes potentially strongly simplifying assumptions \citep{Blei_variational_2017, bayer_mind_2021} that may limit the ability of the encoder model to approximate the true posterior.
Somewhat surprisingly, the BPTT+TF approach to model training clearly outperformed the more sophisticated VI approach. This could be rooted in suboptimal encoder models or in suboptimal sampling from the approximate posterior: While BPTT+TF assesses longer bits of trajectory during optimization, in VI single time-point samples are drawn and the temporal consistency is ensured only through the Kullback-Leibler term in the ELBO. Other more expressive yet still fast to compute encoder models, e.g., based on normalizing flows \citep{rezende_variational_2015}, may boost performance. 

\paragraph{Hyperparameter Settings}\label{sec:supp:hypers}

To train the \ourmethodname\ in the VI framework, Adam \citep{kingma_adam_2015}, with a batch size of $1000$ and learning rate of $10^{-3}$ was used as the optimizer. For the training with BPTT, we used the Adam optimizer with an initial learning rate of $10^{-3}$ that was iteratively reduced during training down to $10^{-5}$. For each epoch we randomly sampled sequences of length $T_{seq}=500$ (except for the Lorenz-63 runs, where $T_{seq}=200$ time steps were sufficient) from the total training data pool of each dataset, which are then fed into the reconstruction method 
in batches of size $16$. Parameters $\bm{A}$, $\bm{W}$ and $\bm{h}$ were initialized according to \citet{talathi_improving_2016}, while the $\{\alpha_b\}$ were initialized according to a uniform distribution in $\left[-B^{-0.5}, B^{-0.5}\right]$. Initial thresholds $\{\bm{h}_b\}$ also followed uniform distributions, but with ranges determined by the extent of the data, i.e. such that the whole data domain was covered. 

To find optimal hyper-parameters we performed a grid search within $\lambda \in \{0, 0.01,0.1,1,10\}$ (VI), $\tau \in \{1, 5, 10, 25, 50, 100\}$ (BPTT-TF), $M \in \{5,10,15,20,25,30,35,40,45,50,75,100\}$, and $B \in \{0,1,2,5,10,20,35,50\}$. Hyper-parameters chosen for the benchmarks in Sec. \ref{sec:experiments} are reported in Table \ref{tab:reg} below (note that these may not fully agree with the ranges initially scanned, as given above, since we attempted to adjust them further in order to approximately match the number of parameters among models in Table \ref{tab:benchmarks}).

\begin{table}[!htp]
\centering
    \caption{Hyperparameter settings for \ourmethodname\ VI/TF for all data sets from Sec. \ref{sec:experiments}. 
    }
    \label{tab:reg}
\begin{tabular}{ c c c c c c}

Dataset & M & B &$M_{\textrm{reg}}/M$ & $\lambda$ & $\tau$ \\ 
\hline
Lorenz-63 &  $22$ & $20$ & $1.0/-$ & $1.0/-$ & $-/25$ \\ 
Lorenz-96 &  $42/50$ & $50/30$  & $1.0/-$ &  $1.0/-$  & $-/10$ \\  
Bursting Neuron &  $26$ & $50/47$  & $0.5/-$ & $1.5/-$  & $-/5$ \\
Neural Population &  $12/75$ & $5/40$  & $0.2/-$ &  $1.0/-$ & $-/5$ \\ 
EEG &  $117/128$ & $50/50$  & $0.8/0.1$ &  $1.0/5\cdot10^{-3}$ & $-/10$ \\ 
ECG &  $-/30$ & $-/50$  & $-/-$ &  $-/-$ & $-/10$ \\ 

        \bottomrule

\end{tabular}
\end{table}

\paragraph{LSTM-MSM and Reservoir Computing}
For the Lorenz-63 and Lorenz-96 system, hyperparameters were used according to the default settings for these specific datasets in the codebase provided at \url{https://github.com/pvlachas/RNN-RC-Chaos}. For 
the other datasets not previously explored by \citet{vlachas_data-driven_2018} and \citet{pathak_model-free_2018}, we performed a grid search to find best performing hyperparameter configurations. To this end, the following hyperparameters were scanned for the RC and LSTM-MSM approach, respectively: For RC the scanned hyperparameters and values are the \texttt{dynamics\_length} $\in \{10, 50, 100, 200, 500\}$, \texttt{noise\_level} $\in \{10, 100, 1000\}$, \texttt{regularization} $\in \{0, 10^{-1}, 10^{-2}, 10^{-3}\}$ and \texttt{learning\_rate} $\in \{10^{-2}, 10^{-3}, 10^{-4}\}$. For LSTM-MSM, similar parameters were explored, where the equivalent of \texttt{dynamics\_length} is \texttt{hidden\_state\_propagation\_length} in the respective code. Also, due to the comparatively slow execution times, only fewer parameter combinations were tried.

\paragraph{Neural ODE}
For Neural ODEs we used the implementation in the \texttt{torchdiffeq} package. The number of layers was fixed to make the numbers of 
trainable parameters comparable to those in Table \ref{tab:benchmarks}, while a grid search was performed over activation functions $\{\upquote elu \upquote, \upquote silu \upquote, \upquote tanh\upquote\}$, sequence length $\{5, 10, 25, 50\}$ used per batch, and learning rates $\{1e-3, 1e-2\}$. For each dataset, $20$ models were trained for $1000$ epochs, each with a batch size of $20$. The \texttt{odeintadjoint} method with the default solver \texttt{dopri5} was mostly used 
for training. As the adaptive step size caused numerical instabilities during training on the two experimental datasets, we switched to the solver \texttt{euler} for these.

\paragraph{PySINDy}
For PySINDy we performed a hyperparameter search over threshold values $\{.0001,.0002, .0005, 0.001,0.002,0.005,0.01,0.02,0.05, 0.1\}$ in the \texttt{stlsqoptimizer}. This threshold specifies the minimum value for a coefficient in the weight vector, below which it is dropped out (i.e., set to zero).

\subsection{Performance Measures}\label{sec:supp:metrics}

    \paragraph{Geometrical Measure}
    
    $D_{\textrm{stsp}}$ used for evaluating attractor geometries (Fig. \ref{fig:effect_basis_expansion}) measures the match between the ground truth distribution $p_{\text {true}}(\vx)$ and the generated distribution $p_{\text {gen}}(\vx \mid \vz)$ through the discrete binning approximation \citep{koppe_identifying_2019}
    \begin{align}\label{eq:D_stsp} 
    D_{\mathrm{stsp}}\left(p_{\mathrm {true }}(\vx), p_{\mathrm {gen }}(\vx \mid \vz)\right) \approx \sum_{k=1}^{K} \hat{p}_{\mathrm {true }}^{(k)}(\vx) \log \left(\frac{\hat{p}_{\mathrm{true }}^{(k)}(\vx)}{\hat{p}_{\mathrm {gen }}^{(k)}(\vx \mid \vz)}\right),
    \end{align}
    
    where $K$ is the total number of bins, and $\hat{p}_{\text {true }}^{(k)}(\vx)$ and $\hat{p}_{\text {gen }}^{(k)}(\vx \mid \vz)$ are estimated as relative frequencies through sampling trajectories from the benchmark DS and the trained reconstruction method, respectively. A range of $2 \times$ the data standard deviation on each dimension was partioned into $m$ bins, yielding a total of $K=m^{N}$ bins, where $N$ is the dimension of the ground truth system. 
    Initial transients are removed from sampled trajectories to ensure that the system has reached a limiting set. 
    If the bin size is chosen too large, important geometrical details may be lost, while if it is chosen too small, noise and (low) sampling artifacts with many empty bins may misguide the approximation above. Here we chose a bin number of $m=30$ per dimension as an optimal compromise that distinguished well between successful and poor reconstructions.

    Since the number of bins needed to cover the relevant (populated) region of state space 
    scales exponentially with the number of dimensions, for high-dimensional systems evaluating $D_{\textrm{stsp}}$ as outlined above is not feasible.
    We therefore resorted to an approximation of the densities based on Gaussian Mixture Models (GMMs), similar to a strategy outlined in \citep{koppe_identifying_2019}. 

Specifically, we approximate the true data distribution by a GMM $p_{\mathrm{true}}(\vx) \approx \frac{1}{T}\sum_{t=1}^T \mathcal{N}(\vx_t,\boldsymbol\Sigma)$ with Gaussians centered on the observed data points $\{\vx_t\}$ and covariance $\boldsymbol\Sigma$, which we treat as a hyper-parameter that determines the granularity of the spatial resolution (similar to the bin size in Eq. \ref{eq:D_stsp}). We can generate a likewise distribution by sampling trajectories from the trained models (or one very long trajectory) and placing Gaussians on the sampled data points, $p_{\mathrm{gen}}(\vx | \vz)\approx\frac{1}{L}\sum_{l=1}^L \mathcal{N}(\hat\vx_l \mid \vz_l,\boldsymbol\Sigma)$ (in the case of VI, rather than sampling, one could also use the model's distributional assumptions to build this posterior across the observations).

For the Kullback-Leibler divergence between two GMMs efficient approximations are at hand \citep{Hershey2007ApproximatingTK}. Here we employ a Monte Carlo approximation 
\begin{align}
    \widetilde{D}_{\mathrm{stsp}}\big(p_{\mathrm{}{true}}(\vx), p_{\mathrm{gen}}(\vx | \vz)\big) \approx \frac{1}{n}\sum_{i=1}^n \log \frac{1/T\sum_{t=1}^T\mathcal{N}(\vx^{(i)};\vx_t,\boldsymbol\Sigma)}{1/L\sum_{l=1}^L\mathcal{N}(\vx^{(i)};\hat\vx_l,\boldsymbol\Sigma)},
\end{align}
where $n$ Monte-Carlo samples $\vx^{(i)}$ are drawn from the GMM representing $p_{\mathrm{true}}$. 
In practice, we set the covariance $\boldsymbol\Sigma = \sigma^2\mI$ equal to a scaled identity matrix, with a single hyperparamter $\sigma^2$.

    Scanning the range $\sigma^2 \in \{0.01, 0.02,0.05,0.1,0.2,0.5,1 ,2, 5\}$, we found that values for $\sigma^2= 0.1 - 1.0$ to differentiate best between good and bad reconstructions. 
    We chose $\sigma^2=1.0$ for numerical stability. For this setting, $D_{\textrm{stsp}}$ as derived with the binning method and $\widetilde{D}_{\textrm{stsp}}$ computed through the GMMs also correlated highly on the low-dimensional benchmark systems (see Figure \ref{fig:klx_correlation}). 

    \begin{figure}[!htb]
    \centering
	\includegraphics[width=0.5\linewidth]{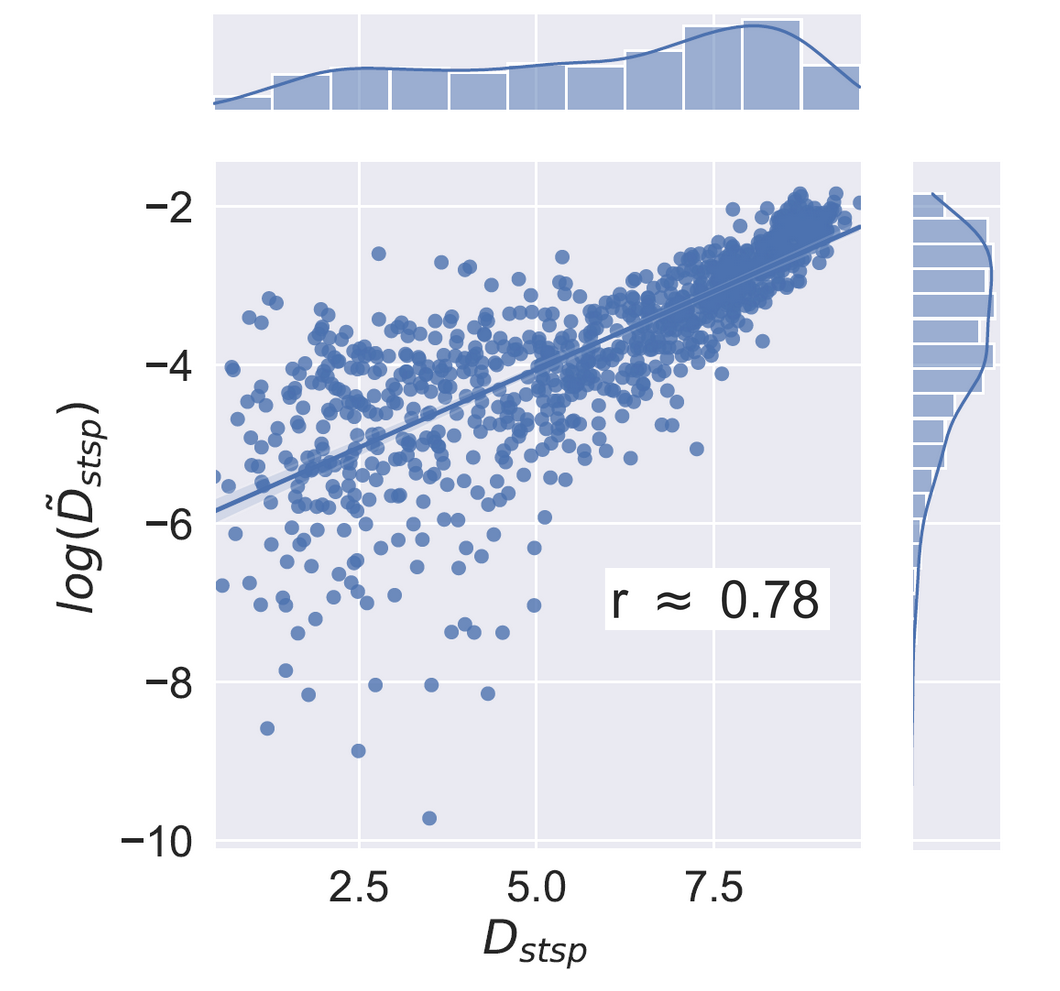}
	\caption{Correlation between the binning approximation ($m=30$)  
	and the logarithm of the GMM approximation ($\sigma^2=1$) to $D_{stsp}$ on the Lorenz-63 system for different noise realizations and variances. Similar as reported for the $KL_\bold{Z}$ approximation in \cite{koppe_identifying_2019} we observed a logarithmic relation between the GMM and binning based measures.
	}
	\label{fig:klx_correlation}
\end{figure}

  \paragraph{Power Spectrum Correlation}
  
  The power spectrum correlations (PSC) were obtained by first sampling time series of 100,000 time steps, standardizing these, and computing dimension-wise Fast Fourier Transforms (using \texttt{scipy.fft}) for both the ground truth systems and model-simulated time series. Individual power spectra were then slightly smoothed with a Gaussian kernel, normalized, and the long, high-frequency tails of the spectra, mainly representing noise, were cut off. Smoothing width $\sigma$ and cutoff values were increased linearily with the length of the time series used to compute the spectrum, and were chosen by visual inspection of the individual spectra. Dimension-wise correlations between smoothed power spectra were then averaged to obtain the reported PSC scores.
  
  \begin{figure}[!htb]
    \centering
	\includegraphics[width=0.99\linewidth]{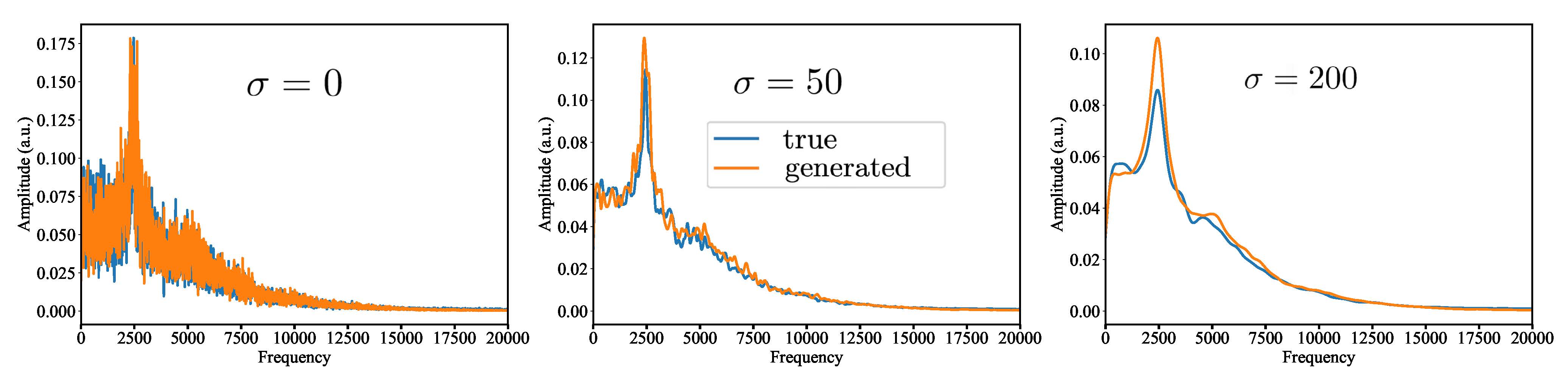}
\caption{Example power spectra for different values of the smoothing factor $\sigma$.}

	\label{fig:smoothing_sigma}
\end{figure}

   \paragraph{Mean Squared Prediction Error} 
   
   A mean squared prediction error (PE) was computed across test sets of length $T=1000$ by initializing the reconstructed model with the test set time series up to some time point $t$, from where it was then iterated forward by $n$ time steps to yield a prediction at time step $t+n$. The $n$-step ahead prediction error (PE) is then defined as the MSE between predicted and true observations:
   \begin{equation}
   PE(n)=\frac{1}{N(T-n)} \sum_{t=1}^{T-n}\sum_{i=1}^{N}(x_{i, t+n}-\hat{x}_{i, t+n})^2    .
   \end{equation}
   Note that for a chaotic system initially close trajectories will exponentially diverge, such that PEs for too large prediction steps $n$ are not meaningful anymore (in a chaotic system with noise, for large $n$ the PE may be high even when estimated from two different runs of the same ground truth model from the same initial condition; see \citep{koppe_identifying_2019}). 
   How quickly this happens depends on the rate of exponential divergence as quantified through the system's maximal Lyapunov exponent \citep{kantz_nonlinear_2004}.

\subsection{Details on Dynamical Systems Benchmarks} \label{sec:supp:dynsys}

\paragraph{Lorenz-63 System}\label{sec:supp:lorenz}
The famous 3d chaotic Lorenz attractor with the butterfly wing shape, originally proposed in \citep{lorenz_deterministic_1963}, is defined by

 \begin{align} \label{eq:lorenz} \nonumber
       \mathrm{d}x &= (\sigma (y - x))\mathrm{d}t+\mathrm{d}\epsilon_1(t), \, \\
        	\mathrm{d}y &= (x (\rho - z) - y) \mathrm{d}t+\mathrm{d}\epsilon_2(t), \, \\\nonumber
       \mathrm{d}z &=  (x y - \beta z) \mathrm{d}t+\mathrm{d}\epsilon_3(t).
    \end{align}

Parameters used for producing ground truth data in the chaotic regime were $\sigma=10, \rho=28$, and $\beta=8/3$. Process noise was injected into the system by drawing from a Gaussian term $d\boldsymbol\epsilon \sim\mathcal{N}(\mathbf{0},0.01^2 dt \times \mathbf{I})$.

\paragraph{Bursting Neuron Model}\label{sec:supp:burstingneuron}

The 3d biophysical bursting neuron model was introduced in \citep{durstewitz_implications_2009} and is defined by one voltage and two ion channel gating variables (one slow and one fast):

\begin{gather}
\begin{aligned} \label{eq:bursting_neuron}
-C_{m} \dot{V} &=g_{L}\left(V-E_{L}\right)+g_{N a} m_{\infty}(V)\left(V-E_{N a}\right) \\
&+g_{K} n\left(V-E_{K}\right)+g_{M} h\left(V-E_{K}\right) \\
&+g_{N M D A} \left[1+.33 e^{-.0625 V}\right]^{-1} \left(V-E_{N M D A}\right) \\
\end{aligned} 
\end{gather}
\begin{gather}
\begin{aligned}
\dot{n}=\frac{n_{\infty}(V)-n}{\tau_{n}} \\
\dot{h}=\frac{h_{\infty}(V)-h}{\tau_{h}} \\
\end{aligned}
\end{gather}

The limiting values of the ionic gates are given by
\begin{equation}
\left\{m_{\infty}, n_{\infty}, h_{\infty}\right\}=\left[1+e^{\left(\left\{V_{h N a}, V_{h K}, V_{h M}\right\}-V\right) /\left\{k_{N a}, k_{K}, k_{M}\right\}}\right]^{-1}.
\end{equation}

We borrowed parameter settings from \citet{schmidt_identifying_2021} to place the system into the burst-firing regime: 
\begin{align*}
    C_{m}=6 \mu \mathrm{F}, g_{L}=8 \mathrm{mS}, E_{L}=-80 \mathrm{mV}, g_{N a}=20 \mathrm{mS}, E_{N a}=60 \mathrm{mV}, V_{h N a}=-20 \mathrm{mV}, \\
    k_{N a}=15, g_{K}=10 \mathrm{mS}, E_{K}=-90 \mathrm{mV}, V_{h K}=-25 \mathrm{mV}, k_{K}=5, \tau_{n}=1 \mathrm{~ms}, g_{M}=25 \mathrm{mS} \\
    V_{h M}=-15 \mathrm{mV}, k_{M}=5, \tau_{h}=200 \mathrm{~ms}, g_{N M D A}=10.2 \mathrm{mS}
\end{align*}

\paragraph{Lorenz-96 System}\label{sec:supp:lorenz96}

The Lorenz-96 is a high-dimensional, spatially extended weather model, also introduced by Edward Lorenz \citep{lorenz_predictability_1996}:
        \begin{equation} \label{eq:lorenz-96}
		\mathrm{d}x_i = ((x_{i+1}-x_{i-2})x_{i-1} - x_i+ F)\mathrm{d}t+\mathrm{d}\epsilon, \ i=1 \dots N,	\end{equation}
    with (constant) forcing term $F$. $F=8$ is a common choice that leads to chaotic behavior. Process noise was added as for the Lorenz-63 system, $d\boldsymbol\epsilon \sim\mathcal{N}(\mathbf{0},0.01^2 dt \times \mathbf{I})$. In our simulations we used $N=10$, but in principle the system allows for arbitrary dimensionality.

\paragraph{Neural Population Model}\label{sec:supp:neuralpopulation}

A larger-scale neural population model was recently introduced in \citet{landau_coherent_2018} to examine the effect of structured connectivity on top of a randomly initialized network matrix. Specifically, an independently Gaussian distributed weight structure was combined with a rank-1 component with coupling strength $J_1$. The dynamics of the single unit currents were defined as
\begin{equation} \label{eq:neuralpop}
\frac{d\mathbf{h}}{dt}=-\mathbf{h}+\mathrm{J} \phi(\mathbf{h})+\frac{J_{1}}{\sqrt{N}} \xi v^{T}\phi(\mathbf{h}),
\end{equation}
where $\phi(\mathbf{h})=\tanh (\mathbf{h}(t))$. We produced a 50-dimensional chaotic network model based on the code provided in \citet{landau_coherent_2018} using $J_1=0.09$ and seeding the random number generator with $35$.\\

The Lorenz-63 and Lorenz-96 systems were simulated using \texttt{scipy.integrate}, while for the bursting neuron and neural population model we used the code provided in \citet{schmidt_identifying_2021} and \citet{landau_coherent_2018}, respectively.

\begin{table}
\caption{Comparison of \ourmethodname\ (Ours) trained by VI or BPTT+TF, and a standard PLRNN \citep{schmidt_identifying_2021}, trained by VI or BPTT+TF on four DS benchmarks (top) and three challenging data situations (bottom). Values are mean $\pm$ SEM.} 
\centering
\scalebox{0.75}{
\begin{tabular}[hbt!]{l l r@{ \,$\pm$\, }l r@{ \,$\pm$\, }l r@{ \,$\pm$\, }ll l l l}
        \toprule
        Dataset	&	Method	&	\multicolumn{2}{c}{PSC}	&	\multicolumn{2}{c}{$D_{\textrm{stsp}}$}  & \multicolumn{2}{c}{20-step PE}	&  Dyn.var. & \#parameters \\
        \midrule
        \multirow{5}{4em}{Lorenz}
        &	 {\ourmethodname} VI	  &     $\mathbf{0.997}$ 	&   $\mathbf{0.001}$	&	$\num{0.80}$    &   $\num{0.25}$ & ${2.1\mbox{e$-$}3}$  & ${0.2\mbox{e$-$}3}$   & $22$	    & $1032$  \\
        &	 {\ourmethodname} TF	    &    $\mathbf{0.997}$
 	&   $\mathbf{0.002}$	&	$\mathbf{0.13}$    &   $\mathbf{0.18}$  & $\mathbf{9.2\mbox{e$-$}5}$  & $\mathbf{2.8\mbox{e$-$}5}$  &  $22$	    & $1032$  \\
 	&	 PLRNN VI	  &     $0.94$ 	&   $0.004$	&	$\num{16.6}$    &   $\num{0.4}$ & ${1.8\mbox{e$-$}1}$  & ${0.1\mbox{e$-$}1}$   & $30$	    & $1020$  \\
        &	 PLRNN TF	    &    $0.994$
 	&   $0.001$	&	$\num{0.4}$    &   $\num{0.09}$  & ${4.3\mbox{e$-$}3}$  & ${0.2\mbox{e$-$}3}$  &  $30$	    & $1011$  \\
        \midrule
        \multirow{5}{4em}{Bursting Neuron}
        &	 {\ourmethodname} VI	&	$0.55$ & $0.03$	        &	$\num{7.5}$ &$\num{0.4}$     & ${6.1\mbox{e$-$}1}$  & ${0.1\mbox{e$-$}1}$            & $26$          & $2052$    \\
        &	 {\ourmethodname} TF	&	$\mathbf{0.76}$ & $\mathbf{0.04}$	        &	$\mathbf{0.61}$ &$\mathbf{0.09}$      & $\mathbf{6.1\mbox{e$-$}2}$  & $\mathbf{2.2\mbox{e$-$}2}$           & $26$          & $2040$    \\
         &	 PLRNN VI	&	$0.54$ & $0.01$	        &	$\num{17.5}$ &$\num{0.5}$     & ${1.17}$  & ${0.14}$            & $42$          & $2021$    \\
        &	 PLRNN TF	&	${0.72}$ & ${0.07}$	        &	$\num{0.63}$ &$\num{0.11}$      & ${6.4\mbox{e$-$}2}$  & ${2.0\mbox{e$-$}2}$           & $43$          & $2021$    \\
            \midrule
         
        \multirow{5}{4em}{Lorenz-96}
        &	 {\ourmethodname} VI	&		${0.987}$ 	& ${0.001}$	&	$\num{0.10}$ &$\num{0.01}$  & ${3.1\mbox{e$-$}1}$  & ${0.9\mbox{e$-$}1}$    &  $42$       & $4384$	\\
        &	 {\ourmethodname} TF	&		$\mathbf{0.998}$ 	& $\mathbf{0.0001}$	&	$\mathbf{0.04}$ &$\mathbf{0.01}$   & $\mathbf{4.1\mbox{e$-$}2}$  & $\mathbf{0.8\mbox{e$-$}2}$   &  $50$       & $4480$	\\
        &	 PLRNN VI	  &     $0.93$ 	&   $0.002$	&	$\num{1.68}$    &   $\num{0.03}$ & ${2.1\mbox{e$-$}3}$  & ${0.2\mbox{e$-$}3}$   & $60$	    & $4260$  \\
        &	 PLRNN TF	    &    $0.996$
 	&   $0.0003$	&	$\num{0.05}$    &   $\num{0.01}$  & ${2.2\mbox{e$-$}1}$  & ${0.2\mbox{e$-$}1}$  &  $64$	    & $4700$  \\    
        \midrule
        \multirow{5}{4em}{Neural Population Model}
        &	 {\ourmethodname} VI	&	$0.45$ &$0.05$	&	$\num{0.56}$ &$\num{0.05}$         &	${0.82}$ &${0.09}$                                     & $12$        & $821$	\\ 
        &	 {\ourmethodname} TF	&	$\mathbf{0.52}$ &$\mathbf{0.01}$	&	$\mathbf{0.37}$ &$\mathbf{0.05}$    &	$1.53$ &$0.03$                                          & $75$        & $9990$	\\ 
        &	 PLRNN VI	&	$0.48$ &$0.01$	&	$\num{11.65}$ &$\num{1.32}$         &	$\mathbf{0.68}$ &$\mathbf{0.09}$                                     & $13$        & $832$	\\ 
        &	 PLRNN TF	&	${0.47}$ &${0.15}$	&	${0.6}$ &${0.3}$    &	$4$ &$10$                                          & $98$        & $12102$	 \\
        \midrule
        \midrule
        \multirow{5}{4em}{Low amount of data}
        &	 {\ourmethodname} VI	&	$0.967$ & $0.007$ 	&	$\mathbf{4.36}$ & $\mathbf{0.10}$	     & ${2.8\mbox{e$-$}2}$  & ${0.2\mbox{e$-$}2}$          & $22$   & $1032$    \\
        &	 {\ourmethodname} TF	&	$\mathbf{0.97}$ & $\mathbf{0.04}$ 	&	$6.9$ & $5.3$	   & $\mathbf{1.5\mbox{e$-$}2}$  & $\mathbf{0.9\mbox{e$-$}2}$            & $22$   & $1032$    \\
        &	 PLRNN VI	&	$0.96$ & $0.01$ 	&	$18.1$ & $0.10$	     & ${1.08}$  & ${0.02}$          & $30$   & $1020$    \\
        &	 PLRNN TF	&	$0.96$ & $0.04$ 	&	$9.0$ & $5.4$	   & ${1.8\mbox{e$-$}2}$  & ${0.5\mbox{e$-$}2}$            & $30$   & $1011$    \\
        \midrule
        \multirow{5}{4em}{Partially observed}
        &	 {\ourmethodname} VI & $\num{0.940}$ & $\num{0.006}$	&	$\num{12.6}$ & $\num{1.0}$	    & ${6.5\mbox{e$-$}2}$  & ${1.4\mbox{e$-$}2}$                           & $22$	    & $1032$  \\
        &	 {\ourmethodname} TF & ${0.993}$ & ${0.003}$	&	$\mathbf{0.54}$ & ${0.16}$	  & ${5.3\mbox{e$-$}3}$  & ${0.2\mbox{e$-$}3}$                             & $22$	    & $1032$  \\
        &	 PLRNN VI & $\num{0.944}$ & $\num{0.002}$	&	$\num{17.2}$ & $\num{0.2}$	    & ${2.7\mbox{e$-$}1}$  & ${0.03\mbox{e$-$}1}$                           & $30$	    & $1020$  \\
        &	 PLRNN TF & $\mathbf{0.994}$ & $\mathbf{0.003}$	&	${0.56}$ & ${0.34}$	  & $\mathbf{5.0\mbox{e$-$}3}$  & $\mathbf{0.2\mbox{e$-$}3}$                             & $30$	    & $1011$   \\
        \midrule
        \multirow{5}{4em}{High noise}
        &	 {\ourmethodname} VI	& $0.973$	 &  $0.006$	                &	$4.9$ & $0.75$	       & ${3.5\mbox{e$-$}2}$  & ${0.1\mbox{e$-$}2}$                                & $22$	    & $1032$  \\
        &	 {\ourmethodname} TF	& $\mathbf{0.995}$	 &  $\mathbf{0.002}$               &	$\mathbf{0.4}$ & $\mathbf{0.13}$	             & ${4.6\mbox{e$-$}3}$  & ${0.4\mbox{e$-$}3}$                          & $22$	    & $1032$  \\
        &	 PLRNN VI	& $0.94$	 &  $0.004$	                &	$18.2$ & $0.04$	       & ${6.4\mbox{e$-$}1}$  & ${0.1\mbox{e$-$}1}$                                & $30$	    & $1020$  \\
        &	 PLRNN TF	& ${0.994}$	 &  ${0.002}$               &	${0.5}$ & ${0.08}$	             & $\mathbf{4.3\mbox{e$-$}3}$  & $\mathbf{0.2\mbox{e$-$}3}$                          & $30$	    & $1011$  \\
        \bottomrule
\end{tabular}
}
\label{tab:supp:vae}
\end{table}

\paragraph{Wilson Cowan Model}\label{supp-wc}

The Wilson-Cowan model is a classical model of neural population dynamics that describes the interactions between a pool of excitatory (E) cells and one of inhibitory (I) cells \citep{wilson_excitatory_1972}, defined by
\begin{align}
&\tau_{i} \frac{d r_{i}}{d t}=-r_{i}+\phi\left(w_{e i} \cdot r_{e}-w_{i i} \cdot r_{i}-z_{i}\right) \\
&\tau_{e} \frac{d r_{e}}{d t}=-r_{e}+\phi\left(w_{e e} \cdot r_{e}-w_{e i} \cdot r_{i}-z_{e}\right),
\end{align}
where $w_{e i},w_{e e},w_{i e},w_{i i}$ are coupling strengths, $z_i$ and $z_e$ denote constant input currents, and $\tau_{i}$ and $\tau_{e}$ are time constants. We chose parameters that placed the model into a bistable regime: $w_{e i}=9.,w_{e e}=9.,w_{i e}=5.,w_{i i}=5.,z_{e}=3,z_{i}=4$. The vector field and fixed points for this configuration are shown in Figure \ref{fig:wc_flow}. 

For simulating the model, we used the implementation provided at \url{https://github.com/OpenSourceBrain/WilsonCowan}. For training the \ourmethodname, we sampled 400 initial conditions evenly distributed across the unit square $[0,1]^2$, and then simulated trajectories of $T=300$ from each of these initial states. Vector fields for the \ourmethodname\ were approximated as finite $1$-step difference vectors $F(x_n)-x_n$ at each grid point $x_n$, where $F$ denotes the map induced by the \ourmethodname\ in observation space.
Approximated and ground truth vector fields are shown in Figure \ref{fig:wc_flow}.

\paragraph{EEG Dataset} \label{sec:supp:EEG}

Electroencephalogram (EEG) data were taken from a study by \citep{schalk_bci2000_2004} available at \url{https://physionet.org/content/eegmmidb/1.0.0/}. These are 64-channel EEG data obtained from human subjects during different motor and imagery tasks. We trained the \ourmethodname\ using BPTT+TF on the "eyes open" baseline time series from subject $0$, which had a total of 9760 time steps. The signal was standardized and smoothed with a Hann function, using \texttt{numpy.hanning} and a window length of $15$. Results for ground-truth and freely generated EEG signals from several brain regions are shown in Figure \ref{fig:EEG_trajectory}.

\begin{figure}[!htb]
    \centering
	\includegraphics[width=0.99\linewidth]{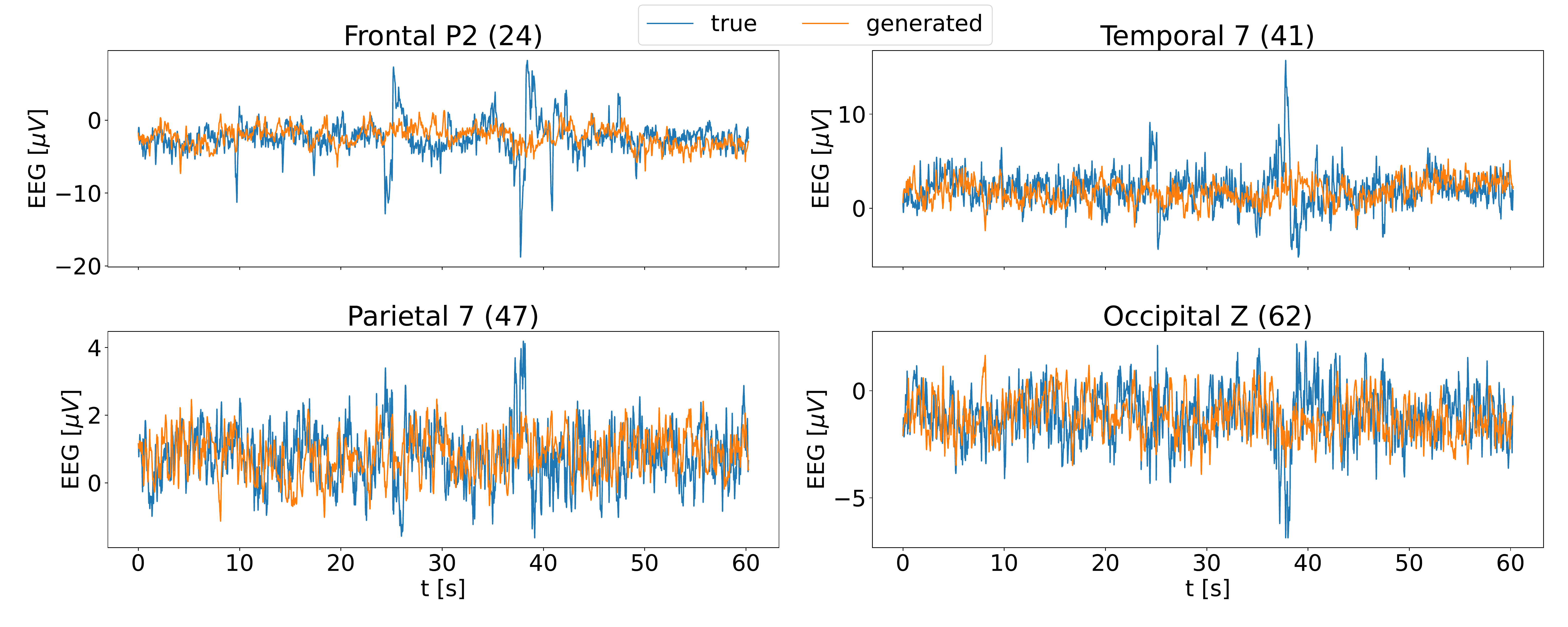}
	\caption{EEG recordings from frontal, occipital, parietal and temporal lobe vs. freely generated trajectories, sampled from the \ourmethodname, trained with BPTT
	($M=128, B=50, \tau=10, M_{\textrm{reg}}/M=0.1, \lambda=5 \cdot 10^{-3}$).}
    \label{fig:EEG_trajectory}
\end{figure}

\paragraph{ECG Dataset} \label{sec:supp:ECG}

Electrocardiogram (ECG) time series were taken from the PPG-DaLiA dataset \citep{Reiss_2019}. ECG signals were captured using a chest-worn device (RespiBAN Professional) with a sampling rate of 700 Hz. For the benchmark model comparisons (Table \ref{tab:benchmarks}), we used the first (one-dimensional) time series (index 0, ``sitting'') from subject $2$, which consists of 419973 time steps. We preprocessed the data by slightly smoothing the series using a Gaussian kernel ($\sigma=5$ time bins), standardization, and performing a temporal delay embedding with dimension $m=7$ and lag $\tau_{lag}=61$. Optimal embedding parameters were determined using the \texttt{DynamicalSystems.jl} Julia package. For the experiments, we constructed a training and test set, each of length $T=100,000$ cut out from the available series.
\begin{figure}[!htb]
    \centering
	\includegraphics[width=0.80\linewidth]{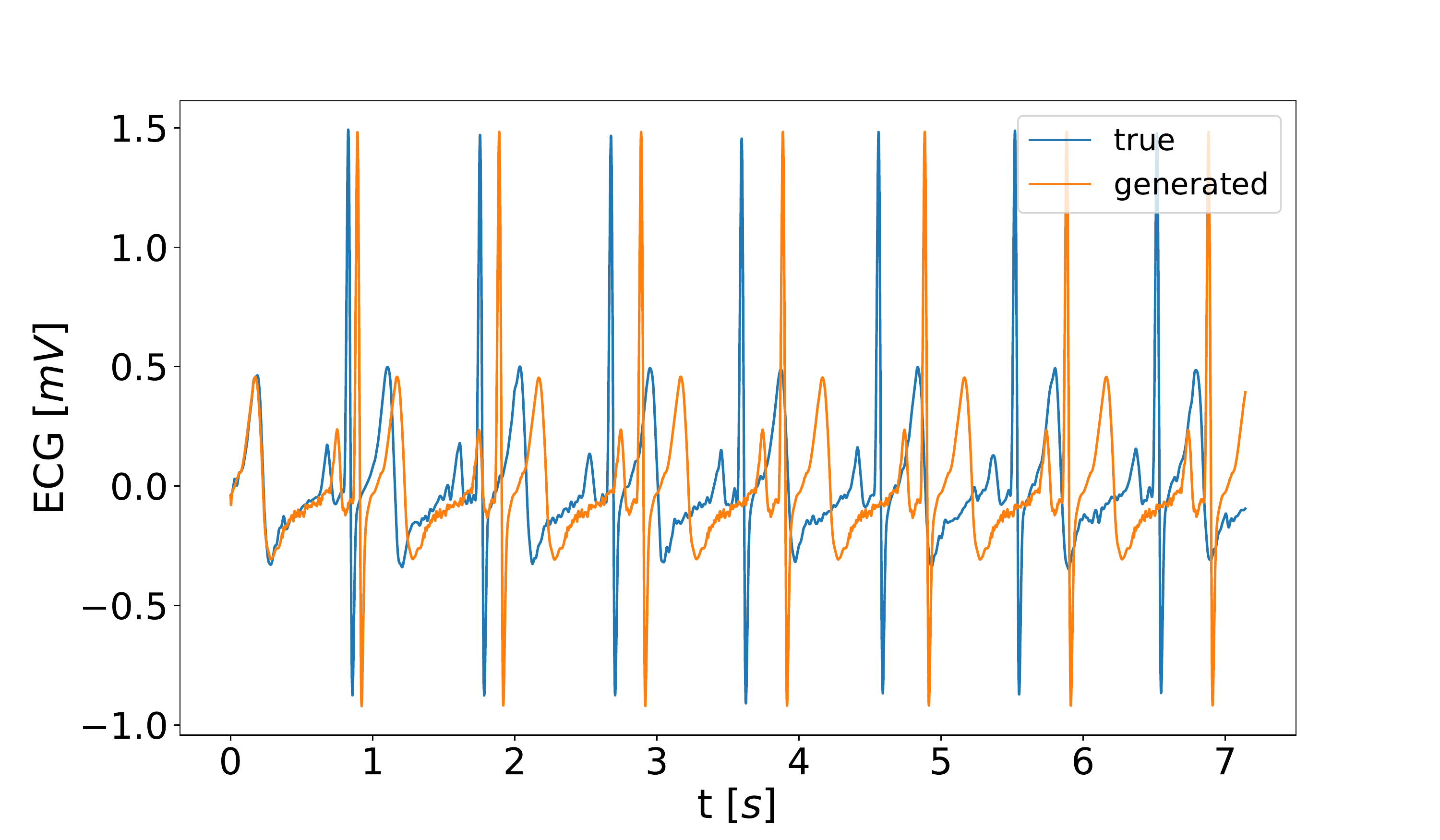}
	\caption{Original ECG recording vs. freely generated time series simulated using a \ourmethodname, trained with BPTT-TF ($M=30, B=50, \tau=10$).}
	\label{fig:ECG_trajectory}
\end{figure}

\newcommand\norm[1]{\left\lVert#1\right\rVert}
\subsection{Theoretical Analysis}\label{supp-theo-analysis}
Consider the PLRNN with linear spline basis expansion as defined by \eqref{eq:plrnn_lat}, \eqref{eq:basis_expansion}, 
reproduced here for convenience: 
\begin{align}\label{eq-1}
 \vz_t \, = \, & \mA \vz_{t-1} + \mW  \sum_{b=1}^{B} \alpha_{b} \, \max (0, \vz_{t-1}- \vh_b)  \, 
+\vh_0
%
%
+ \mC \vs_t +\bm{\epsilon}_t,
\end{align}
where $\bm{\epsilon}_t \sim N(0, \bm{\Sigma}), \, \, \, E[\bm{\epsilon}_t, \bm{\epsilon}_{t'}\tran]=0$ for $t \neq t^{\prime}$, $\alpha_{b} \in \sR$ are scalar weighting factors and $\vh_{b} \in \sR^M$ different ReLU \say{activation thresholds}, and all other parameters are as in conventional PLRNNs \citep{koppe_identifying_2019}.

Defining 
\begin{align}\label{}
\mD^{(b)}_{\Omega(t-1)}(\vz_{t-1}- \vh_b)\, := \, \max(0, \vz_{t-1}- \vh_b),
\end{align}
\eqref{eq-1} can be rewritten as
\begin{align}\label{eq-2}\nonumber
 \vz_t \, = \, & \bigg(\mA + \mW \sum_{b=1}^{B} \alpha_{b} \, \mD^{(b)}_{\Omega(t-1)}\bigg)\vz_{t-1} 
\\[1ex]
&
\, + \, \mW \sum_{b=1}^{B} \alpha_{b} \, \mD^{(b)}_{\Omega(t-1)} \, (-\vh_b) \, 
\, + \, \vh_0\, + \, \mC \vs_t +\bm{\epsilon}_t,
\end{align}
where $\mD^{(b)}_{\Omega(t-1)}= \text{diag} \big(d^{(b)}_{1,t-1}, d^{(b)}_{2,t-1}, \cdots, d^{(b)}_{M,t-1} \big)$ are diagonal binary indicator matrices with $d^{(b)}_{m,t-1}=1$ if $z_{m,t-1}> h_{m,b}$ and $0$ otherwise. 
\subsubsection{Fixed points and $n$-cycles of system \eqref{eq-2}}\label{s-zm-fp-sec}
Defining 
\begin{align}\label{defs}\nonumber
 \mD^{B}_{\Omega(t-1)}&:=  \sum_{b=1}^{B} \alpha_{b} \, \mD^{(b)}_{\Omega(t-1)} , \hspace{.2cm}
\\
\vh^{B}_{\Omega(t-1)}&:=\sum_{b=1}^{B} \alpha_{b} 
\, \mD^{(b)}_{\Omega(t-1)} (-\vh_{b}), \hspace{.2cm}
%
%
\\\nonumber 
\mW^{B}_{\Omega(t-1)} 
&:= \mA  + \mW \, \mD^{B}_{\Omega(t-1)},
\end{align}
and considering the autonomous system (i.e., without external inputs or noise terms), \eqref{eq-2} can be rewritten as 
\begin{align}\label{eq-3}
 \vz_t \, = \,  \mW^{B}_{\Omega(t-1)} \, \vz_{t-1} +  \mW \, \vh^{B}_{\Omega(t-1)} +\vh_0.
\end{align}
Fixed points and cycles of \eqref{eq-1}, and their eigenvalue spectra, can now be computed in a way analogous to standard PLRNNs. Specifically, solving the equation $F( \vz^{*1}) =  \vz^{*1}$, fixed points of the \ourmethodname\ are given by
\begin{align}\label{s-zm-fp1}
  \vz^{*1} = \Big(\mI- \mW^{B}_{\Omega(t^{*1})} \Big)^{-1} \Big[\mW \,\vh^{B}_{\Omega(t^{*1})} +\vh_0\Big] ,
\end{align}
where $\vz^{*1}=\vz_{t^{*1}}=\vz_{t^{*1}-1}$, and $\det(\mI- \mW^{B}_{\Omega(t^{*1})}) = P_{\mW^{B}_{\Omega(t^{*1})}}(1) \neq 0$, i.e. $\mW^{B}_{\Omega(t^{*1})}$ has no eigenvalue equal to $1$. If there are eigenvalues close or equal to $1$ (such that the matrix in (\ref{s-zm-fp1}) is non-invertible or badly conditioned), this means the system is 
undergoing a bifurcation or has one or more directions of marginal stability in its state space (e.g., may exhibit a line, plane, or manifold attractor if there is convergence to this object along the other directions). We emphasize that non-invertibility is thus not a `problem' for the method, but rather indicates a dynamically important scenario in itself that can be detected from the eigenspectrum (see, e.g., \citet{schmidt_identifying_2021}).
\\

For $n>1$, an $n$-cycle with periodic points $\{\vz^{*n}, F(\vz^{*n}), F^2(\vz^{*n}), \cdots, F^{n-1}(\vz^{*n})\}$ of map $F$ can be obtained by solving $F^{n}(\vz^{*n}) = \vz^{*n}$. Therefore, in order to find the periodic points, we first compute $F^n$ in the following way:\\
\begingroup
\allowdisplaybreaks
\begin{align}\nonumber
&\vz_t \, = \, F( \vz_{t-1}) \, = \,  \mW^{B}_{\Omega(t-1)} \, \vz_{t-1} +  \mW \, \vh^{B}_{\Omega(t-1)} +\vh_0,
 \\[2ex]\nonumber
&\vz_{t+1} \, = \,  F^2( \vz_{t-1}) = F( \vz_{t})
%
%
\, = \, \mW^{B}_{\Omega(t)} \,\mW^{B}_{\Omega(t-1)} \, \vz_{t-1} 
\, + \, \big(\mW^{B}_{\Omega(t)} \,\mW  \, \vh^{B}_{\Omega(t-1)} \, + \,  \mW \,  \vh^{B}_{\Omega(t)}\big)
\\[1ex]\nonumber
& \hspace{.9cm}
 \, + \, \big(\mW^{B}_{\Omega(t)}\, + \, \mI \big)\vh_0,
\\[2ex]\nonumber
&\vz_{t+2} \, = \, F^3( \vz_{t-1}) \, = \, F( \vz_{t+1})
%
%
\, = \, \mW^B_{\Omega(t+1)} \mW^B_{\Omega(t)} \mW^B_{\Omega(t-1)} \, \vz_{t-1} +  \big(\mW^{B}_{\Omega(t+1)} \mW^{B}_{\Omega(t)} \mW  \vh^{B}_{\Omega(t-1)}
\\[1ex]\nonumber
& \hspace{.9cm}
 \, + \,  \mW^{B}_{\Omega(t+1)} \mW  \vh^{B}_{\Omega(t)}
%
%
 + \mW \vh^{B}_{\Omega(t+1)} \big) \, + \, \big(\mW^B_{\Omega(t+1)} \mW^B_{\Omega(t)} \,+ \, \mW^B_{\Omega(t+1)} 
%
%
\, + \, \mI \big)\vh_0,
 \\[1ex]\nonumber
& \vdots
 \\[1ex]\nonumber
& \vz_{t+(n-1)}= F^n( \vz_{t-1}) 
%
%
\,= \, \prod_{i=2}^{n+1} \mW^B_{\Omega(t+n-i)} \, \vz_{t-1}
%
%
\,+ \, \sum_{j=2}^{n} \bigg[\prod_{i=2}^{n-j+2} \mW^B_{\Omega(t+n-i)} \,  \mW \,  \vh^{B}_{\Omega(t+j-3)} \bigg]
\\[1ex]\label{}
& \hspace{1.5cm}
\,+ \, \mW \vh^{B}_{\Omega(t+n-2)} \,+ \, \bigg(\sum_{j=2}^{n} \prod_{i=2}^{n-j+2} \mW^B_{\Omega(t+n-i)}\, + \,\mI \bigg) \vh_0,
\end{align}
\endgroup
where
\begin{align}\nonumber
 \prod_{i=2}^{n+1} \mW^B_{\Omega(t+n-i)} = \mW^B_{\Omega(t+n-2)} \mW^B_{\Omega(t+n-3)} \cdots \mW^B_{\Omega(t-1)}.   
\end{align}
Defining $t+n-1=: t^{*n}$, the periodic point $\vz^{*n}$ of the $n$-cycle of $F$ can now be obtained as the fixed point of the $n$-times iterated map $F^n$ as  
\begingroup
\allowdisplaybreaks
\begin{align}\nonumber
 &\vz^{*n} 
%
%
\, = \,  \bigg(\mI- \prod_{i=1}^{n} \mW^B_{\Omega(t^{*n}-i)} \bigg)^{-1}
%
%
\bigg(\sum_{j=2}^{n} \Big[\prod_{i=1}^{n-j+1} \mW^B_{\Omega(t^{*n}-i)} \mW   \vh^{B}_{\Omega(t^{*n}-n +j-2)} \Big]\,
\\[1ex]\label{cycles}
&  \hspace{.9cm}
+\, \mW \vh^{B}_{\Omega(t^{*n}-1)} 
 \, + \, \Big(\sum_{j=2}^{n} \prod_{i=1}^{n-j+1} \mW^B_{\Omega(t^{*n}-i)} + \mI \Big) \vh_0 \bigg),
\end{align}
\endgroup
where $\vz^{*n}=\vz_{t^{*n}}=\vz_{t^{*n}-n}$, if
$(\mI- \prod_{i=1}^{n} \mW^B_{\Omega(t^{*n}-i)})$ 
is invertible, i.e. 
$$\det \bigg(\mI-\prod_{i=1}^{n} \mW^B_{\Omega(t^{*n}-i)} \bigg)= P_{\prod_{i=1}^{n} \mW^B_{\Omega(t^{*n}-i)}}(1) \neq 0,$$ 
which implies $\mW_{\Omega^{*n}}:=\prod_{i=1}^{n} \mW^B_{\Omega(t^{*n}-i)}$ has no eigenvalue equal to $1$. As for simple fixed points, non-invertibility of the matrix in (\ref{cycles}) implies there are 
directions of marginal stability in the \ourmethodname's state space, along which we will find continuous sets of $n$-cycles, or the system is undergoing a bifurcation. 
Thus, we are approaching such a situation as one of the eigenvalues of $\mW_{\Omega^{*n}}$ moves toward $1$ and the matrix in (\ref{cycles}) may become ill-conditioned.

\begin{remark_2}\label{remark-FPs-cylces-mcPLRNN}
These results about fixed points and $n$-cycles also hold for the mean-centred \ourmethodname. This can easily be seen by defining $\mW^{B}_{\Omega(t-1)} := \mA  + \mW \, \mD^{B}_{\Omega(t-1)} \, \mM $ and noting that the elements of $\mD^{(b)}_{\Omega(t-1)}$ are now determined by the mean-centred latent states. That is $d^{(b)}_{m,t-1}=1$ if $z_{m,t-1}-\frac{1}{M}\sum_{j=1}^M z_{j,t-1}> h_{m,b}$ and $0$ otherwise.
The rest of the calculations then proceeds as above.
\end{remark_2}
\subsubsection{Sub-regions and discontinuity boundaries corresponding to system \eqref{eq-2}} \label{sec:subregions}
Consider system \eqref{eq-2} without external input and noise terms. Denoting $\vh_b=(h_{1,b}, h_{2,b}, \cdots,  h_{M,b})\tran$ in \eqref{eq-2}, for $b=1, 2, \cdots, B$, we can order the elements $h_{j,1}, h_{j,2}, \cdots, h_{j,B}$ for every $j \in \{ 1, 2, \cdots, M \}$. Without loss of generality, let 
\begin{align}\label{ineq_h}
h_{j,1} < h_{j,2} <  \cdots < h_{j,B}, \hspace{.5cm}  j=1, 2, \cdots, M.  
\end{align}
Then, for every $j$, we define the intervals $I_{j,b}$ as follows:
\begin{align}\label{ints}
I_{j,1} \, := \, (-\infty, h_{j,1}],\nonumber 
\\[1ex]
%
I_{j,b} \, := \, (h_{j,b-1}, h_{j,b}], \hspace{.3cm}  b=2, 3, \cdots, B, 
\\[1ex]\nonumber
%
I_{j,B+1} \, := \, (h_{j,B}, +\infty).
\end{align}
By definition of $\mD^{(i)}_{\Omega(t-1)}$ in \eqref{eq-2}, the phase space is separated into $(B+1)^M$ sub-regions by $M B (B+1)^{M-1}$ hyper-surfaces as discontinuity boundaries. Every sub-region can be defined by the thresholds $\vh_b$ as Cartesian product of suitable intervals in \eqref{ints} for $j \in \{ 1, 2, \cdots, M \}$. (Note that if in \eqref{ineq_h} we had $"\leq"$ instead of strict inequalities $"<"$, obviously the number of  intervals, hence sub-regions, would decrease.) 
In each sub-region the matrices $\mD^{(b)}_{\Omega(t-1)}, \, b=1,2, \cdots, B$, have a different configuration. Therefore, in \eqref{eq-3} there are $(B+1)^M$ different forms for $\mD^{B}_{\Omega(t-1)}$, and so for $\mW^{B}_{\Omega(t-1)}$ and $\vh^{B}_{\Omega(t-1)}$ as well. Hence, indexing $\mD^{B}_{\Omega(t-1)}$, $\mW^{B}_{\Omega(t-1)}$ and $\vh^{B}_{\Omega(t-1)}$ as $\mD^{B}_{(r)}$, $\mW^{B}_{(r)}$ and $\vh^{B}_{(r)}$ for $r \in \{1, 2, \cdots, (B+1)^M\}$, \eqref{eq-2} can be written as
\begin{align}\label{eq-4}
 \vz_t \, = \,  \mW^{B}_{(r)} \, \vz_{t-1} +  \mW \, \vh^{B}_{(r)} +\vh_0.
\end{align}
To visualize the sub-regions and their borders, let for example $M=2$ and $B=2$. In this case there are $9$ sub-regions divided by $12$ borders. As illustrated in Fig. \ref{figure_math_sub1}, there are different matrices $\mD^{(b)}_{\Omega(t-1)}, b=1,2,$ and $\mD^{B}_{(r)}=\mD^{2}_{(r)}, r=1,2, \cdots, 9,$ for each sub-region.
\begin{figure}[!hbt]
\centering
\includegraphics[scale=0.38]{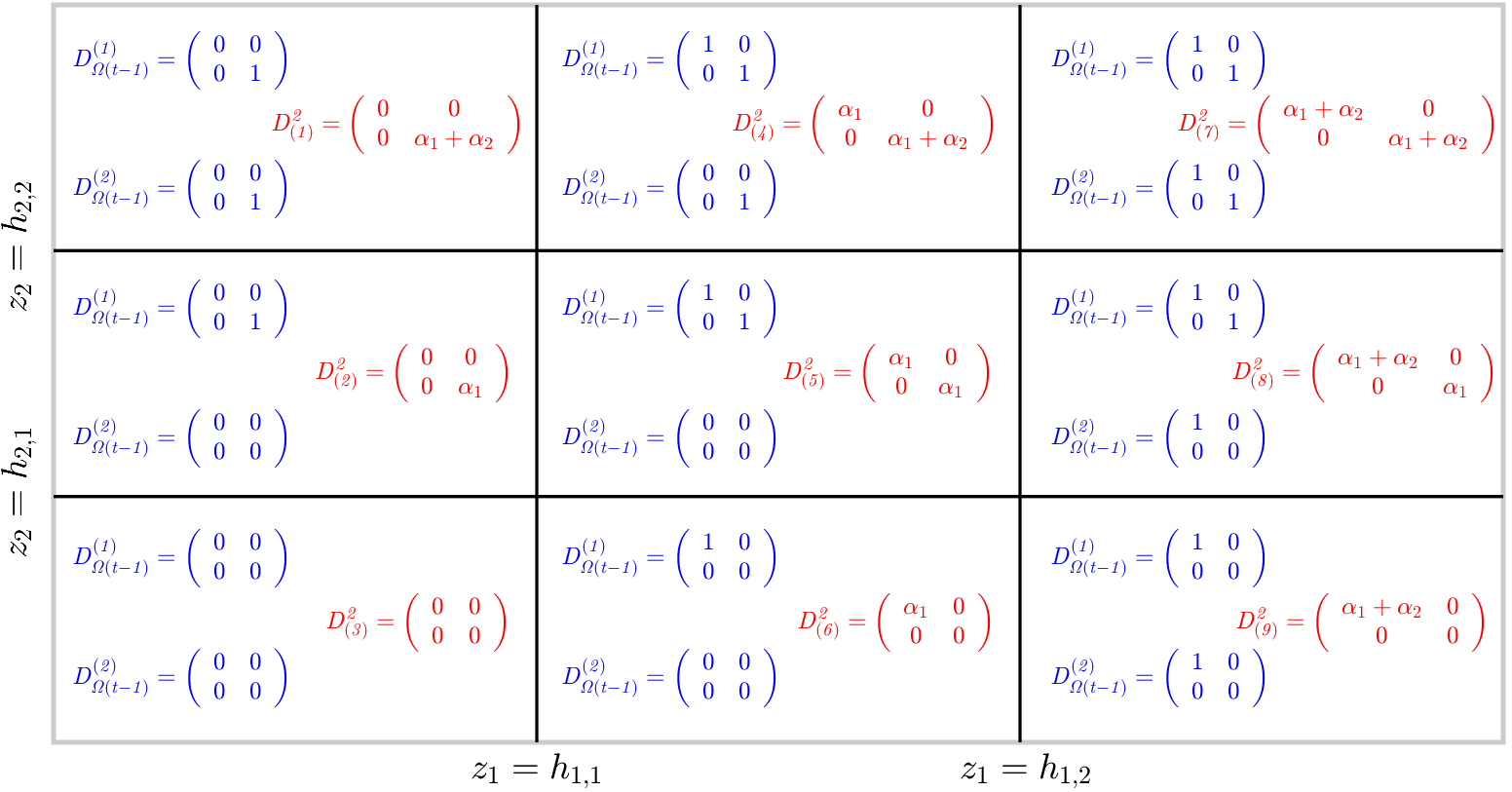}
\caption{Example of different sub-regions and related matrices $\mD^{(b)}_{\Omega(t-1)}, b=1,2,$ and $\mD^{B}_{(r)}, r=1,2, \cdots, 9$, for $M=2$ and $B=2$. Here, it is assumed that the components of $\vh_1=(h_{1,1}, h_{2,1})\tran$ and $\vh_2=(h_{1,2}, h_{2,2})\tran$ satisfy \eqref{ineq_h} with $"<"$.}
\end{figure}\label{figure_math_sub1}
\subsubsection{Bounded orbits are compatible with the Manifold Attractor Regularization}
 \begin{proposition}\label{pro-Z-3}
 The results of Theorem \ref{pro-J-1} are also true when the manifold-attractor regularization, Eq. \ref{eq:supp:MAR}, is 
 strictly enforced for the \ourmethodname, \eqref{eq-clipped}.\end{proposition} 
 \begin{proof}
Assume $\mA$, $\mW$, $\tilde{\phi}(z_{t-1})$ (see proof of Theorem
\ref{pro-J-1} in Appx. \ref{pro-J-1-proof} for the definition) and $\vh_0$ have the partitioned forms
\begingroup
\allowdisplaybreaks
\begin{align}\nonumber
\renewcommand\arraystretch{1.3}
& \mA= \begin{pmatrix}
\renewcommand{\arraystretch}{1.4}
\begin{array}{c|c}
  \mI_{reg} & \mO\tran\\[1ex]
  \hline
 \mO  & \mA_{nreg}
\end{array}
\end{pmatrix},
\hspace{1cm}
 \mW= \begin{pmatrix}
 \renewcommand{\arraystretch}{1.4}
\begin{array}{c|c}
  \mO_{reg} & \mO\tran \\[1ex]
  \hline
\mS  &  \mW_{nreg}
\end{array}
\end{pmatrix},
\\[2ex]\label{AW_part}
& \vh_0= \begin{pmatrix}
\renewcommand{\arraystretch}{1.4}
\begin{array}{c}
  \vh_0^{reg} \\[1ex]
  \hline
\vh_0^{nreg}
\end{array}
\end{pmatrix},
\hspace{2.1cm}
 \tilde{\phi}(z_{t-1})= \begin{pmatrix}
 \renewcommand{\arraystretch}{1.4}
\begin{array}{c}
  \tilde{\phi}_{reg}(z_{t-1}) \\[1ex]
  \hline
\tilde{\phi}_{nreg}(z_{t-1})
\end{array}
\end{pmatrix},
\end{align}
\endgroup
~\\
where $\, \mI_{M_{reg} \times M_{reg}}=: \mI_{reg}\in \sR^{M_{reg}\times M_{reg}}, \mO_{M_{reg} \times M_{reg}}=: \mO_{reg} \in \sR^{M_{reg}\times M_{reg}}\,$, $\, \mO, \mS \in \sR^{(M-M_{reg}) \times M_{reg}} \,$, the sub-matrices $\,\mA_{\{M_{reg}+1:M, M_{reg}+1:M \}}=: \mA_{nreg} \in \sR^{(M-M_{reg}) \times (M-M_{reg})}\,$ and $\,\mW_{\{M_{reg}+1:M, M_{reg}+1:M \}}=: \mW_{nreg} \in \sR^{(M-M_{reg}) \times (M-M_{reg})} \,$ are diagonal and off-diagonal respectively. Furthermore, $\, \vh_0^{reg}, \tilde{\phi}_{reg}(z_{t-1}) \in \sR^{M_{reg}}\,$ and $\, \vh_0^{\{M_{reg}+1:M, M_{reg}+1:M \}}\, =: \, \vh_0^{nreg}$, $\tilde{\phi}_{\{M_{reg}+1:M, M_{reg}+1:M \}}(z_{t-1})\,=: \, \tilde{\phi}_{nreg}(z_{t-1}) \in \sR^{M-M_{reg}} \,$.
\\[1ex]

In this case $\,\norm{\mA} = \sigma_{\max}(\mA)= \max \{1, \sigma_{\max}(\mA_{nreg})\} \,$ and 
\begingroup
\allowdisplaybreaks
\begin{align}\nonumber
\norm{\mA^j \,  \mW \,\tilde{\phi}(\vz_{T-1-j})} & \, = \, \norm{\begin{pmatrix}
 \renewcommand{\arraystretch}{1.4}
\begin{array}{c}
 \mO \\[1ex]
  \hline
\mA_{neg}^j \, \mS \, \tilde{\phi}_{nreg}(z_{t-1}) + \mA_{neg}^j \, \mW_{neg} \, \tilde{\phi}_{nreg}(z_{t-1})
\end{array}
\end{pmatrix}}
\\[2ex]\nonumber
&\, = \, \norm{\mA_{neg}^j \, \mS \, \tilde{\phi}_{nreg}(z_{t-1}) + \mA_{neg}^j \, \mW_{neg} \, \tilde{\phi}_{nreg}(z_{t-1})},
\\[3ex]\nonumber
\norm{\mA^j \,  \mW \,\vh_0} & \, = \, \norm{\begin{pmatrix}
 \renewcommand{\arraystretch}{1.4}
\begin{array}{c}
 \mO \\[1ex]
  \hline
\mA_{neg}^j \, \mS \, \vh_0^{nreg} + \mA_{neg}^j \, \mW_{neg} \, \vh_0^{nreg}
\end{array}
\end{pmatrix}}
\\[2ex]\label{}
&\, = \, \norm{\mA_{neg}^j \, \mS \, \vh_0^{nreg} + \mA_{neg}^j \, \mW_{neg} \, \vh_0^{nreg}}.
\end{align}
\endgroup
~\\
Thus, for $\sigma_{\max}(\mA_{nreg}) < 1$
\begingroup
\allowdisplaybreaks
\begin{align}\nonumber
\norm{\vz_{T}} &\, \leq \, 
\norm{\mA}^{T-1}\, \norm{\vz_1} \, + \, \sum_{j=0}^{T-2}\norm{\mA^j \,  \mW \,\tilde{\phi}(\vz_{T-1-j})} +\sum_{j=0}^{T-2} \norm{\mA^j \, \vh_0}
\\[2ex]\nonumber
& \, \leq \, \norm{\vz_1} \, + \,
\big(\tilde{c} + \norm{h_0} \big) \big(\norm{\mS} + \norm{\mW_{neg}}\big)\, \sum_{j=0}^{T-2} \norm{\mA_{neg}}^j 
\\[2ex]\label{}
& \, = \, \frac{\big(\tilde{c} + \norm{h_0} \big) \big(\norm{\mS} + \norm{\mW_{neg}}\big)}{1-\norm{\mA_{neg}}} \, < \, \infty.
\end{align}
\endgroup
 \end{proof}
\subsubsection{Proof of Proposition \ref{pro-1}}\label{p-pro-1}
\begin{proof}
For $\mA = (a_{ij}) \in \sR^{M \times M}$, $\mW = (w_{ij})\in \sR^{M \times M}$, $\bm{\epsilon}_{t} = (\bm{\epsilon}_{1,t}, \bm{\epsilon}_{2,t}, \cdots, \bm{\epsilon}_{M,t})\tran$, $\vs_t = (s_{1,t}, s_{2,t}, \cdots, s_{M,t})\tran$ and $\mC = (c_{ij}) \in \sR^{M \times M}$, writing \eqref{eq-2} in scalar form yields
\begingroup
\allowdisplaybreaks
\begin{align}\nonumber
z_{l, t} &\, =  \, \sum_{j=1}^{M} a_{lj} z_{j, t-1} + \sum_{j=1}^{M} w_{lj} \sum_{b=1}^{B} \alpha_{b}\, d^{(b)}_{j,t-1}  [z_{j, t-1}-h_{j,b}]
%
%
\, + \, h_{l,0} \, + \,  \sum_{j=1}^{M} c_{lj}\, s_{j, t} \, + \, \bm{\epsilon}_{l,t}
\\[1ex]\nonumber
&\, = \,
 \sum_{j=1}^{M}\bigg( a_{lj} z_{j, t-1} + w_{lj}  \sum_{b=1}^{B} \alpha_{b}\, d^{(b)}_{j,t-1} [z_{j, t-1}-h_{j,b}] \bigg) 
 %
%
 \, + \, h_{l,0}\, + \,  \sum_{j=1}^{M} c_{lj}\, s_{j, t} \, + \, \bm{\epsilon}_{l,t} 
\\[1ex]\label{eq-com}
& \, =: \, \sum_{j=1}^{M} f_{l,j}(z_{j, t-1}) \, + \, h_{l,0}\, + \, \sum_{j=1}^{M} c_{lj}\, s_{j, t} \, + \, \bm{\epsilon}_{l,t}
%
%
\,=:\, F_{l}(\vz_{t-1}),  \hspace{.4cm}
 l=1, 2, \cdots, M.
\end{align}
\endgroup
Using this, we can write \eqref{eq-2} in the vector form 
\begin{align}\label{}
\vz_{t}= \big( F_{1}(\vz_{t-1}),  F_{2}(\vz_{t-1}), \cdots,  F_{M}(\vz_{t-1}) \big)\tran.
\end{align}
We show that every $F_{l}$ is continuous and so \eqref{eq-2} is a continuous PWL map. For this purpose, by \eqref{eq-com}, it suffices to prove that every $f_{l,j}(z_{j, t-1})$ is continuous. According to the definition of the intervals $I_{j,b}$, \eqref{ints}, for any $j \in \{1, 2, \cdots, M \}$ we have
\begingroup
\allowdisplaybreaks
\begin{align}\nonumber
& z_{j, t-1} \in I_{j,1} \hspace{.2cm} \Rightarrow \hspace{.2cm} d^{(b)}_{j,t-1} \, = \, 0 \hspace{.5cm} \forall \, b=1, 2, \cdots, B,
\\[1ex]\nonumber
& z_{j, t-1} \in I_{j,s}\hspace{.2cm} \Rightarrow \hspace{.2cm}
\begin{cases}
& d^{(b)}_{j,t-1} \, = \,1, \hspace{.2cm} b=1, 2, \cdots, s-1
\\[1ex]
& d^{(b)}_{j,t-1} \, = \,0, \hspace{.2cm} b=s, s+1, \cdots, B
\\[1ex]
& \hspace{.2cm} s=2, 3, \cdots, B,
\end{cases} 
%
\\[1ex]\label{}
&  z_{j, t-1} \in I_{j,B+1} \hspace{.2cm} \Rightarrow \hspace{.2cm} d^{(b)}_{j,t-1} \, = \, 1 \hspace{.5cm} \forall \, b=1, 2, \cdots, B.  
\end{align}
\endgroup
Hence, for $l,j=1, 2, \cdots, M$, each function $f_{l,j}(z_{j, t-1})$ can be stated as 
\begingroup
\allowdisplaybreaks
\begin{align}\label{}
f_{l,j}(z_{j, t-1})  
%
%
 = 
\begin{cases}
& f^{(1)}_{l,j} = a_{lj} \, z_{j, t-1}; \hspace{6.5cm} z_{j, t-1} \in I_{j,1} 
\\[2ex]
& f^{(2)}_{l,j} = ( a_{lj}+ \alpha_{1} \,w_{lj}) \, z_{j, t-1} \,- \alpha_{1} \, w_{lj} h_{j,1};
%
%
\hspace{2.9cm}
z_{j, t-1} \in I_{j,2} 
\\[1ex]
& \vdots
\\[1ex]
& f^{(B)}_{l,j} = (a_{lj} + w_{lj} \, \sum_{b=1}^{B-1} \alpha_{b})\,z_{j, t-1} \, 
%
%
- w_{lj} \, \sum_{b=1}^{B-1} \alpha_{b}\, h_{j,b}; \hspace{.6cm} z_{j, t-1} \in I_{j,B} 
\\[2ex]
& f^{(B+1)}_{l,j} = (a_{lj} + w_{lj} \, \sum_{i=1}^{B} \alpha_{b})\,z_{j, t-1} \, 
%
%
- w_{lj} \, \sum_{b=1}^{B} \alpha_{b}\, h_{j,b}; \hspace{.45cm} z_{j, t-1} \in I_{j,B+1} 
\end{cases}
\end{align}
\endgroup
Since for every $b=1, 2, \cdots, B$,
\begin{align}\label{}
\displaystyle{\lim_{z_{j, t-1} \to h_{j,b}} f^{(b)}_{l,j}(z_{j, t-1})} 
\,= \, \displaystyle{\lim_{z_{j, t-1} \to h_{j,b}} f^{(b+1)}_{l,j}(z_{j, t-1})} 
%
%
\, = \, f^{(b)}_{l,j}(h_{j,b}),   
\end{align}
each function $f_{l,j}(z_{j, t-1})$ is continuous.
Hence, \eqref{eq-2} is a continuous PWL map in $\vz$ (but has discontinuities in its Jacobian matrix across the borders). Because of these properties, all the results established for standard PLRNNs in \cite{monfared_existence_2020,monfared_transformation_2020,schmidt_identifying_2021} apply to the \ourmethodname\ as well, only that the sub-regions and discontinuity boundaries are different. 
\end{proof}
\subsubsection{Proof of proposition \ref{pro-2}}\label{p-pro-2}
\begin{proof}
Defining $\tilde{\vz}_t$ as $B$ identical copies of $\vz_t$,
\begingroup
\allowdisplaybreaks
\begin{align}\label{}
 \tilde{\vz}_t \, = \,
\begin{pmatrix}
\tilde{z}_{1,t}\\
\tilde{z}_{2,t}\\
\vdots
\\
\tilde{z}_{M,t}\\
\tilde{z}_{M+1,t}\\
\vdots
\\
\tilde{z}_{BM,t}
\end{pmatrix}
\, := \,
\begin{pmatrix}
\vz_t\\
\vz_t\\
\vdots
\\
\vz_t
\end{pmatrix}_{BM \times 1}
\end{align}
\endgroup
and likewise
\begingroup
\allowdisplaybreaks
\begin{align}\nonumber
 & \tilde{\vh} \, = \, 
\begin{pmatrix}
\tilde{h}_{1}\\
\tilde{h}_{2}\\
\vdots
\\
\tilde{h}_{M}\\
\tilde{h}_{M+1}\\
\vdots
\\
\tilde{h}_{BM}
\end{pmatrix}
\, = \,
\begin{pmatrix}
\vh_1\\
\vh_2\\
\vdots
\\
\vh_B
\end{pmatrix}_{BM \times 1},
%
\hspace{1cm} \tilde{\vh}_0
 \, = \, 
\begin{pmatrix}
\tilde{h}_{0,1}\\
\tilde{h}_{0,2}\\
\vdots
\\
\tilde{h}_{0,M}\\
\tilde{h}_{0,M+1}\\
\vdots
\\
\tilde{h}_{0,BM}
\end{pmatrix}
\, = \, \begin{pmatrix}
\vh_0\\
\vh_0\\
\vdots
\\
\vh_0
\end{pmatrix}_{BM \times 1}
\\[1ex]\nonumber
& \tilde{\mA}_{BM \times BM} = diag \big(\underbrace{\mA_{M \times M}, \mA_{M \times M}, \cdots, \mA_{M \times M}}_\text{B \ \text{times}} \big),
%
\\[1ex]\nonumber
& \tilde{\mW}_{BM \times BM}\, = \,
%
\begin{pmatrix}
\renewcommand{\arraystretch}{2} 
\begin{array}{c|c|c|c}
\alpha_1 \mW_{M \times M} & \alpha_2 \mW_{M \times M} & \hdots & \alpha_B \mW_{M \times M} \\[1ex]
 \hline
 \alpha_1 \mW_{M \times M} & \alpha_2 \mW_{M \times M} & \hdots & \alpha_B \mW_{M \times M} \\
 \hline
\vdots & \vdots & \ddots &\vdots\\
\hline
\alpha_1 \mW_{M \times M} & \alpha_2 \mW_{M \times M} & \hdots & \alpha_B \mW_{M \times M} 
\end{array} 
\end{pmatrix},
\\[1ex]\label{}
& \tilde{\mC \vs_t} \, = \,
\begin{pmatrix}
\tilde{cs}_{1,t}\\
\tilde{cs}_{2,t}\\
\vdots
\\
\tilde{cs}_{M,t}\\
\tilde{cs}_{M+1,t}\\
\vdots
\\
\tilde{cs}_{BM,t}
\end{pmatrix}
\, = \,
\begin{pmatrix}
\mC \vs_t\\
\mC \vs_t\\
\vdots
\\
\mC \vs_t
\end{pmatrix}_{BM \times 1},
\hspace{1cm}\tilde{\bm{\epsilon}}_t \, = \,
\begin{pmatrix}
\tilde{\epsilon}_{1,t}\\
\tilde{\epsilon}_{2,t}\\
\vdots
\\
\tilde{\epsilon}_{M,t}\\
\tilde{\epsilon}_{M+1,t}\\
\vdots
\\
\tilde{\epsilon}_{BM,t}
\end{pmatrix}
\, = \,
\begin{pmatrix}
\bm{\epsilon}_t\\
\bm{\epsilon}_t\\
\vdots
\\
\bm{\epsilon}_t
\end{pmatrix}_{BM \times 1}
\end{align}
\endgroup
~\\
one can rewrite the \ourmethodname from \eqref{eq-1} as 
\begin{align}\label{eq-re-MxK}
 \tilde{\vz}_t \, = \,  \tilde{\mA} \tilde{\vz}_{t-1} + \tilde{\mW}  \, \max (0, \tilde{\vz}_{t-1}- \tilde{\vh})  \, 
+\tilde{\vh}_0 \, + \, \tilde{\mC \vs_t} + \tilde{\bm{\epsilon}}_t.
\end{align}
~\\
Now performing the substitution
\begin{align}\label{}
 \forall \, t \hspace{1cm}  \hat{\vz_t} \leftarrow \tilde{\vz}_{t}- \tilde{\vh}, 
\end{align}
\eqref{eq-re-MxK} can be rewritten as the $M \times B$-dimensional \say{conventional} PLRNN \eqref{eq-MxK-st} with
\begin{align}
\hat{\vh}_0\, = \, \big(\tilde{\mA}-\mI\big)\tilde{\vh} \, + \, \tilde{\vh}_0.  
\end{align}
%
\end{proof}
\subsubsection{Proof of Theorem
\ref{pro-J-1}}\label{pro-J-1-proof}
\begin{proof}
It can easily be shown that for every $i \in \{1,2,\cdots, M \}$
\begin{align}
\alpha_b \big[ \max(\max(0, z_{i,t-1}-h_{i,b})-\max(0, z_{i,t-1})\big] \in %
\begin{cases} 
      [-\alpha_bh_{i_b},0] \text{ if } \text{sgn}(\alpha_b)=\text{sgn}(h_{i,b}) \\[1ex]
      [0,\alpha_bh_{i,b}] \text{ else}
\end{cases}.
\end{align}
%
By defining
\begin{align}
\sum_{b=1}^B\alpha_b \big[ \max(0,\vz_{t-1}-\vh_b)-\max(0,\vz_{t-1}) \big] \, := \,\tilde{\phi}(z_{t-1})\, = \, \Big(\tilde{\phi}_1(z_{t-1}), \cdots, \tilde{\phi}_M(z_{t-1}) \Big)\tran,   
\end{align}
and 
\begin{align}
c^{\text{up}}_{i,b}= \begin{cases} 
      0 \text{ if } \text{sgn}(\alpha_b)=\text{sgn}(h_{i,b})\\[1ex]
       \alpha_bh_{i,b} \text{ else}
   \end{cases}, \hspace{.5cm}
   c^{\text{low}}_{i,b}= \begin{cases} 
     - \alpha_bh_{i,b} \text{ if } \text{sgn}(\alpha_b)=\text{sgn}(h_{i,b})\\[1ex]
      0 \text{ else}
   \end{cases},
\end{align}
we can conclude that
$$c_i^{low}\leq \tilde{\phi}_i(z_{t-1}) \leq c_i^{up},$$
\\
where $c_i^{low/up}\, = \, \sum_{b=1}^{B} c_{i,b}^{low/up}$. For $\, c_i\,=\, \max\{|c_i^{low}|, |c_i^{up}| \}\,$ we have
\begin{align*}
\tilde{\phi}_i(z_{t-1})^2 \, \leq \, c_i^2,
\end{align*}
and so letting $\,c\,=\, \max\{c_1, c_2, \cdots, c_M\}\,$ yields 
\begin{align}\label{eq-tilde-xi}
\norm{\tilde{\phi}(z_{t-1})}\, = \,\sqrt{\sum_{i=1}^{M} \big(\tilde{\phi}_i(z_{t-1})\big)^2} \, \leq \, \sqrt{\sum_{i=1}^{M} c^2}  \, := \, \tilde{c}. 
\end{align}
~\\[2ex]
Since 
\begin{align}
\vz_{t}\, = \, \mA\, \vz_{t-1} \, + \, \mW \, \tilde{\phi}(\vz_{t-1}) \, +\, \vh_0, 
\end{align}
 for $T \in \mathbb{N}$ and $t= 2, \cdots,T$, computing $z_2,z_3, \cdots,z_T $ recursively leads to
\begin{align}\nonumber
&\vz_2\, = \, \mA\, \vz_1 \, + \, \mW  \, \tilde{\phi}(\vz_1) \, + \,\vh_0
\\[1ex]\nonumber
&\vz_3\, = \, \mA^2\, \vz_1 \, + \, \mA\,  \mW  \,\tilde{\phi}(\vz_1) + \mW \,\tilde{\phi}(\vz_2)  \, + \,\big[\mA+\mI \big] \vh_0
\\\nonumber
\vdots
\\\label{}
& \vz_{T}\, = \, \mA^{T-1}\, \vz_1 \, + \, \sum_{j=0}^{T-2}\mA^j \,  \mW \,\tilde{\phi}(\vz_{T-1-j}) +\sum_{j=0}^{T-2} \mA^j \, \vh_0. 
\end{align}
Therefore, by \eqref{eq-tilde-xi}, for every $T \geq 2$, we have 
\begin{align}
\norm{\vz_{T}}\, \leq \, \norm{\mA}^{T-1}\, \norm{\vz_1} \, + \, \tilde{c}\, \norm{\mW} \sum_{j=0}^{T-2} \norm{\mA}^j  +\sum_{j=0}^{T-2} \norm{\mA}^j \, \norm{\vh_0}. 
\end{align}
If $\sigma_{\max}(\mA) < 1$, then $\displaystyle{\lim_{T \to \infty}\norm{\mA}^{T-1}} \, = \, 0$ and 
\begin{align}\label{}
 \displaystyle{\lim_{T \to \infty}\norm{ \vz_{T}}} \, \leq \,  \, \tilde{c}\, \norm{\mW} \sum_{j=0}^{\infty} \norm{\mA}^j  +\sum_{j=0}^{\infty} \norm{\mA}^j \, \norm{\vh_0} \, = \, \frac{\tilde{c}\, \norm{\mW}+ \norm{\vh_0}}{1-\norm{\mA}} \, < \, \infty.
\end{align}
\end{proof}

\end{document}


\twocolumn[
\icmltitle{Supplementary Material: Interpretable high-capacity RNN for identifying unknown nonlinear dynamical systems}

\icmlsetsymbol{equal}{*}

\begin{icmlauthorlist}
\icmlauthor{Leonard Bereska}{zi,hd}
\icmlauthor{Po-Chen Kuo}{ntu}
\icmlauthor{Manuel Brenner}{zi}
\icmlauthor{Zahra Monfared}{zi}
\icmlauthor{Daniel Durstewitz}{zi,hd}
\end{icmlauthorlist}

\icmlaffiliation{zi}{Central Institute of Mental Health, Germany}
\icmlaffiliation{hd}{Heidelberg University, Heidelberg, Germany}
\icmlaffiliation{ntu}{National Taiwan University, Taiwan}

\icmlcorrespondingauthor{Manuel Brenner}{manuel.brenner@zi.mannheim.de}
\icmlcorrespondingauthor{Daniel Durstewitz}{daniel.durstewitz@zi.mannheim.de}

\icmlkeywords{recurrent neural networks, dynamical systems, variational inference, dendritic computation, state space model, generative stochastic networks}

\vskip 0.3in
]



\printAffiliationsAndNotice{}  

\input{figures/tables/benchmarks}






%
The model can be rewritten as a continuous PWL map